\documentclass[number]{ReportTemplate}
\usepackage{utopia}
\usepackage{amssymb}
\usepackage{amsfonts}
\usepackage{amsmath}
\usepackage{mathtools}
\usepackage{amsthm}
\usepackage{bm}
\usepackage{graphicx}
\usepackage{algorithm}
\usepackage{algorithmic}
\usepackage[american]{babel}
\usepackage{subfigure}
\usepackage{hyperref}
\mathtoolsset{showonlyrefs,showmanualtags}
\usepackage{color}

\newcommand\numberthis{\addtocounter{equation}{1}\tag{\theequation}}

\renewenvironment{myproof}[1]
{\par\noindent\textbf{Proof of #1.}\ \enspace\ignorespaces\begin{allowdisplaybreaks}}
{\end{allowdisplaybreaks}\hspace{\stretch{1}}$\square$}

\begin{document}
\begin{frontmatter}

\title{Result Diversification by Multi-objective Evolutionary Algorithms with Theoretical Guarantees}
\author{Chao Qian}
\ead{qianc@lamda.nju.edu.cn}
\author{Dan-Xuan Liu}
\ead{liudx@lamda.nju.edu.cn}
\author{Zhi-Hua Zhou\corref{cor1}}
\ead{zhouzh@lamda.nju.edu.cn}
\cortext[cor1]{Corresponding author}
\address{State Key Laboratory for Novel Software Technology, Nanjing University, Nanjing 210023, China\vspace{-1em}
}

\begin{abstract}
Given a ground set of items, the result diversification problem aims to select a subset with high ``quality" and ``diversity" while satisfying some constraints. It arises in various real-world artificial intelligence applications, such as web-based search, document summarization and feature selection, and also has applications in other areas, e.g., computational geometry, databases, finance and operations research. Previous algorithms are mainly based on greedy or local search. In this paper, we propose to reformulate the result diversification problem as a bi-objective maximization problem, and solve it by a multi-objective evolutionary algorithm (EA), i.e., the GSEMO. We theoretically prove that the GSEMO can achieve the (asymptotically) optimal theoretical guarantees under both static and dynamic environments. For cardinality constraints, the GSEMO can achieve the optimal polynomial-time approximation ratio, $1/2$. For more general matroid constraints, the GSEMO can achieve an asymptotically optimal polynomial-time approximation ratio, $1/2-\epsilon/(4n)$, where $\epsilon>0$ and $n$ is the size of the ground set of items. Furthermore, when the objective function (i.e., a linear combination of quality and diversity) changes dynamically, the GSEMO can maintain this approximation ratio in polynomial running time, addressing the open question proposed by Borodin \textit{et al.}~\cite{borodin2017max}. This also theoretically shows the superiority of EAs over local search for solving dynamic optimization problems for the first time, and discloses the robustness of the mutation operator of EAs against dynamic changes. Experiments on the applications of web-based search, multi-label feature selection and document summarization show the superior performance of the GSEMO over the state-of-the-art algorithms (i.e., the greedy algorithm and local search) under both static and dynamic environments.
\end{abstract}

\end{frontmatter}

\newpage
\section{Introduction}

In many real-world artificial intelligence applications such as web-based search~\cite{borodin2017max}, document summarization~\cite{dasgupta2013summarization} and feature selection~\cite{ghadiri2019distributed}, one often needs to select a subset of items (e.g., documents, sentences and features) with high ``quality" and ``diversity" while satisfying some constraints. For example, given a query, a search engine usually wants to return a limited number of diversified documents which are relevant to the query as much as possible. Such problems are called result diversification, which can be formalized as
\begin{align}\label{eq-general}
\arg\max\nolimits_{X \subseteq V} f(X)+\lambda \cdot div(X) \quad s.t. \quad X \in \mathcal{F},
\end{align}
where $V$ is a ground set of items, $f: 2^V \rightarrow \mathbb{R}$ is a monotone submodular function characterizing the quality of a subset, the diversity function $div$ is usually defined as
\begin{align}\label{eq-max-sum}\forall X \subseteq V: div(X)=\sum_{\{u,v\}: u,v \in X} d(u,v)\end{align} w.r.t. some distance measure $d: V \times V \rightarrow \mathbb{R}$, and $\mathcal{F}$ represents the feasible solution space consisting of the subsets satisfying the constraints. The monotone submodularity of $f$ implies that additional items will not lessen the value of $f$, but the increment will be at a decreasing rate as the subset extends, which expresses the property of quality measure in many applications. Note that a distance function $d: V \times V \rightarrow \mathbb{R}$ is a symmetric nonnegative function, and satisfies $\forall v \in V: d(v,v)=0$. In this paper, if without explicit specification, $d$ is assumed to further satisfy the triangle inequality, \begin{align}\label{eq-mid-36}\forall u,v,w \in V: d(u,v)+d(v,w) \geq d(u,w),\end{align} that is, it is a metric.

The result diversification problem has applications in various areas, e.g., artificial intelligence, computational geometry, databases, finance and operations research, and has attracted a lot of attentions. It has been addressed with gradually relaxed assumptions on the quality function $f$ and the constraint $X \in \mathcal{F}$. The case with $f=0$ and $\mathcal{F}=\{X \subseteq V \mid |X| \leq k\}$, i.e.,
\begin{align}\label{eq-case-1}
\arg\max\nolimits_{X \subseteq V} div(X) \quad s.t. \quad |X| \leq k,
\end{align}
also called max-sum dispersion, was first studied. That is, the quality is not considered, and only the diversity in Eq.~(\refeq{eq-max-sum}) is maximized under a cardinality constraint. Ravi \textit{et al.}~\cite{ravi1994heuristic} proved that a standard greedy algorithm, which iteratively adds one item with the largest marginal gain on $div$, i.e., $\max_{u \in V \setminus X} div(X \cup \{u\})-div(X)=\max_{u \in V \setminus X} \sum_{v \in X} d(u,v)$, can achieve a $(1/4)$-approximation ratio, which was later shown to achieve an approximation ratio of $1/2$~\cite{birnbaum2009improved}. Hassin \textit{et al.}~\cite{hassin1997approximation} proposed a different greedy algorithm, which iteratively selects two items with the largest distance, and proved that it can also achieve a $(1/2)$-approximation ratio. Note that $1/2$ has been shown to be the optimal polynomial-time approximation ratio under the planted clique assumption~\cite{bhaskara2016linear,borodin2017max}. To handle large-scale data sets, the distributed greedy algorithm was developed~\cite{zadeh2017scalable}, which uses a two-round divide and conquer strategy: it partitions the ground set $V$ randomly into different machines and runs the greedy algorithm on each machine independently; then it combines the subset generated on each machine and runs the greedy algorithm again. Compared with the centralized greedy algorithm, the guarantee on the approximation ratio of the distributed greedy algorithm is sacrificed, which is decreased to $1/4$. In~\cite{sydow2014improved}, the distance function $d$ is relaxed, which is not necessarily a metric, but only requires to satisfy the parameterized triangle inequality, i.e., $\forall u,v,w \in V: d(u,v)+d(v,w) \geq \alpha \cdot d(u,w)$, where $0 \leq \alpha \leq 1$. The greedy algorithm was shown to achieve an approximation ratio of $\alpha/2$.

Besides the sum of distances between items in Eq.~(\refeq{eq-max-sum}), other measures of diversity have also been considered. Ravi \textit{et al.}~\cite{ravi1994heuristic} analyzed the max-min form, maximizing the minimum distance between items, i.e.,
\begin{align}\label{eq-max-min}\forall X \subseteq V: div(X)=\min_{\{u,v\}: u,v \in X} d(u,v),\end{align}
and proved that the greedy algorithm can achieve a $(1/2)$-approximation ratio. Let $G=(V,E)$ denote a complete graph, where each vertex corresponds to an item, and the weight of each edge corresponds to the distance between two items. Halld{\'o}rsson \textit{et al.}~\cite{halldorsson1999finding} analyzed the max-mst form, which maximizes
\begin{align}\label{eq-max-mst}\forall X \subseteq V: div(X)=mst(X),\end{align}
where $mst(X)$ denotes the weight of the minimum spanning tree of $X$ on the graph $G$. They proved that the greedy algorithm can achieve an approximation ratio of $1/4$. For more results with different diversity measures, we refer to~\cite{chandra2001approximation}. Note that these diversity measures are non-submodular, and without explicit specification, the diversity $div$ in this paper is assumed to be the most common one, i.e., Eq.~(\ref{eq-max-sum}).

Abbassi \textit{et al.}~\cite{abbassi2013diversity} considered a more general problem than Eq.~(\refeq{eq-case-1}) by using a matroid constraint, i.e,
\begin{align}\label{eq-case-2}
\arg\max\nolimits_{X \subseteq V} div(X) \quad s.t. \quad X \in \mathcal{F},
\end{align}
where $\mathcal{F}$ is a collection of subsets of $V$, and $(V,\mathcal{F})$ is a matroid satisfying the hereditary (i.e., $\emptyset \in \mathcal{F}$ and $\forall X \subseteq Y \in \mathcal{F}: X \in \mathcal{F}$) and augmentation (i.e., $\forall X,Y\in \mathcal{F}, |X|>|Y|: \exists v \in X \setminus Y, Y \cup \{v\} \in \mathcal{F}$) properties. Note that a cardinality constraint $|X|\leq k$ is actually a uniform matroid. The greedy algorithm now cannot achieve any constant approximation ratio~\cite{borodin2017max}. Abbassi \textit{et al.}~\cite{abbassi2013diversity} proved that a local search algorithm, which tries to iteratively improve a subset by swapping two items (i.e., deleting one item inside the subset and inserting one item outside), can achieve a $(1/2)$-approximation ratio. To be scalable to large data sets, Ceccarello \textit{et al.}~\cite{ceccarello2020general} developed the distributed and streaming algorithms to construct a small coreset, which contains a feasible solution with the diversity at least a factor $1-\epsilon$ of the optimum. The local search algorithm running on the coreset then can achieve an approximation ratio of $(1-\epsilon)/2$. By considering distances of negative type, better approximation ratios for the problem Eq.~(\refeq{eq-case-2}) have been obtained. A distance $d$ of negative type satisfies $\forall\bm{x} \in \mathbb{R}^n$ with $\sum^{n}_{i=1} x_i=0$, $\bm{x}^{\mathrm{T}}\mathbf{D}\bm{x} \leq 0$, where $n$ is the size of the ground set $V=\{v_1,v_2,\ldots,v_n\}$, and $\mathbf{D} \in \mathbb{R}^{n\times n}$ is the distance matrix with $D_{i,j}=d(v_i,v_j)$. Note that a distance of negative type is not necessarily a metric, and vice-versa. In~\cite{cevallos2016max}, it has been shown that an approximation ratio of $1-c\log r /r$ can be achieved by rounding the solution generated from a non-convex relaxation of the problem, where $c$ is an absolute constant and $r$ is the rank of the matroid $(V,\mathcal{F})$. In~\cite{cevallos2019improved}, the local search algorithm has been shown to achieve an approximation ratio of $1-4/(r+2)$.

Instead of considering a more general constraint, Gollapudi and Sharma~\cite{gollapudi2009axiomatic} generalized the problem Eq.~(\refeq{eq-case-1}) by adding a modular function on the objective, i.e.,
\begin{align}\label{eq-case-3}
\arg\max\nolimits_{X \subseteq V} \sum_{v \in X} w(v) + \lambda \cdot div(X) \quad s.t. \quad |X| \leq k.
\end{align}
That is, the objective function is now a linear combination of a modular function $f(X)=\sum_{v \in X} w(v)$ (characterizing the quality) and the diversity in Eq.~(\ref{eq-max-sum}). By designing a new distance function, they reduced this problem (i.e., Eq.~(\refeq{eq-case-3})) to Eq.~(\refeq{eq-case-1}), and then applied the greedy algorithm in~\cite{hassin1997approximation}, resulting in a $(1/2)$-approximation ratio.

Borodin \textit{et al.}~\cite{borodin2017max} further relaxed the objective function by using a general monotone submodular function $f$ instead of a modular function to characterize the quality of a subset of items. That is,
\begin{align}\label{eq-case-4}
\arg\max\nolimits_{X \subseteq V} f(X)+\lambda \cdot div(X) \quad s.t. \quad |X| \leq k.
\end{align}
They proposed a greedy algorithm, which iteratively adds one item with the largest marginal gain on $f(X)/2+\lambda \cdot div(X)$ instead of the original objective function $f(X)+\lambda \cdot div(X)$. That is, in each iteration, the item maximizing $(f(X \cup \{u\})-f(X))/2 + \lambda \cdot (div(X \cup \{u\})-div(X))=(f(X \cup \{u\})-f(X))/2+\lambda \cdot \sum_{v \in X} d(u,v)$ is selected. This greedy algorithm achieves an approximation ratio of $1/2$. A distributed version was developed in~\cite{ghadiri2019distributed} to handle huge data sets, whose approximation ratio is decreased to $1/31$. Zadeh and Ghadiri~\cite{zadeh2015max} relaxed the distance function $d$ to satisfy \begin{align}\label{eq-mid-35}\forall u,v,w \in V: \alpha \cdot (d(u,v)+d(v,w)) \geq d(u,w),\end{align} where $\alpha \geq 1$, instead of the triangle inequality with $\alpha=1$. They showed that the greedy algorithm achieves an approximation ratio of $1/(2\alpha)$. In~\cite{dasgupta2013summarization}, the diversity measures in Eqs.~(\refeq{eq-max-min}) and~(\refeq{eq-max-mst}) were considered. For the problem with Eq.~(\refeq{eq-max-min}) (i.e., the minimum distance between items) as the diversity, an approximation ratio of $1/4$ can be achieved by running two standard greedy algorithms to maximize $f$ and $div$, respectively, and returning the better one between the two generated subsets. For the problem with Eq.~(\refeq{eq-max-mst}) (i.e., the weight of minimum spanning tree) as the diversity, the greedy algorithm, which iteratively adds one item maximizing $f(X \cup \{u\})-f(X) + \lambda \cdot \min_{v \in X} d(u,v)$, achieves an approximation ratio of $1/(3\log k)$.

Finally, the general problem with a monotone submodular function $f$ and a matroid constraint, i.e.,
\begin{align}\label{eq-case-5}
\arg\max\nolimits_{X \subseteq V} f(X)+\lambda \cdot div(X) \quad s.t. \quad X \in \mathcal{F},
\end{align}
where $(V,\mathcal{F})$ is a matroid, was considered. Borodin \textit{et al.}~\cite{borodin2017max} showed that the local search algorithm, which starts from a basis (i.e., a subset with the maximum size in $\mathcal{F}$) of the matroid and iteratively improves the basis by swapping two items, can achieve a $(1/2)$-approximation ratio. When the triangle inequality satisfied by the distance function $d$ is relaxed to $\forall u,v,w \in V: \alpha \cdot (d(u,v)+d(v,w)) \geq d(u,w)$, where $\alpha \geq 1$, the local search algorithm achieves an approximation ratio of $1/(2\alpha^2)$ accordingly~\cite{zadeh2015max}. By considering that distances are negative type instead of metrics, Cevallos \textit{et al.}~\cite{cevallos2019improved} proved that the local search algorithm can obtain a better approximation ratio of $1-\max\{4/(r+2),c/e\}-\epsilon$, where $r$ is the rank of the matroid, $c$ is the curvature of $f$, and $\epsilon >0$.

Borodin \textit{et al.}~\cite{borodin2017max} also considered the result diversification problem Eq.~(\refeq{eq-case-5}) under dynamic environments. That is, the quality function $f$ or the distance $d$ changes over time, resulting in the objective function $f(X)+\lambda \cdot div(X)$ changing dynamically. Once seeing a change, we want to maintain the solution quality by modifying the current solution without completely recomputing it. In particular, starting from a solution with a good approximation ratio for the old objective, an algorithm has to be able to regain a solution with the same approximation ratio for the new objective efficiently. However, it is not yet clear whether the local search algorithm can maintain a $(1/2)$-approximation ratio for a dynamic version of the problem Eq.~(\refeq{eq-case-5}) in polynomial running time. On the positive side, by considering the specific problem Eq.~(\refeq{eq-case-3}) (i.e., $f$ is a modular function and the constraint is a cardinality constraint) and restricting the magnitude of dynamic change, the local search algorithm can maintain a $(1/3)$-approximation ratio by a single greedy swap operation~\cite{borodin2017max}. For the general problem Eq.~(\refeq{eq-case-5}) under dynamic environments, it is still an open question whether it is possible to maintain a $(1/2)$-approximation ratio in polynomial running time.

Inspired by the encouraging performance of multi-objective evolutionary algorithms (MOEAs) for solving diverse submodular optimization problems~\cite{friedrich2015maximizing,friedrich2018heavy,qian2020multi,qian2017constrained,qian2019maximizing}, we propose to solve the result diversification problem by MOEAs. Mimicking natural phenomena, MOEAs are a popular tool for multi-objective optimization, i.e., optimizing multiple objectives simultaneously~\cite{chen2021interactive,deb:2001}. The general idea of using MOEAs for result diversification is to first transform the original problem as a bi-objective maximization problem $\arg\max\nolimits_{X \subseteq V} (f_1(X),f_2(X))$, then employ a simple MOEA (i.e., the GSEMO~\cite{laumanns2004running}) to solve it, and finally return the best feasible solution from the generated population. Our main theoretical results can be summarized as follows.\vspace{-0.8em}
\begin{itemize}
  \item For the result diversification problem with a cardinality constraint, i.e., Eq.~(\refeq{eq-case-4}), we set $f_1(X)=(1+|X|/k)\cdot f(X)/2+\lambda \cdot div(X)$ and $f_2(X)=-|X|$. That is, the GSEMO is to maximize the objective function $f_1$ and minimize the subset size $|X|$ simultaneously. We prove that after running at most $enk^3/2$ expected number of iterations, the GSEMO can achieve an approximation ratio of $1/2$ (i.e., {\bf Theorem~\ref{theo-cardinality-sum}}), which is optimal~\cite{borodin2017max}. If the metric distance $d$ is relaxed to satisfy Eq.~(\refeq{eq-mid-35}) instead of Eq.~(\refeq{eq-mid-36}), the approximation ratio of the GSEMO becomes $1/(2\alpha)$ (i.e., {\bf Corollary~\ref{coro:cardinality}}). When $f_1(X)$ is set to the original objective function $f(X)+\lambda \cdot div(X)$, the GSEMO achieves an approximation ratio of $1/2-\epsilon/(4n)$ in polynomial time (i.e., {\bf Theorem~\ref{theo-cardinality-sum-v}}). We also consider the problem where the diversity measure $div$ is Eq.~(\refeq{eq-max-min}) or~(\refeq{eq-max-mst}) instead of the commonly used Eq.~(\refeq{eq-max-sum}), and show that the GSEMO can still achieve good approximation performance.\\
      (1) \;When the diversity measure is Eq.~(\refeq{eq-max-min}), we reformulate the original problem as two bi-objective maximization problems: $f_1(X) = f(X) \wedge f_2(X) = -|X|$ and $f_1(X) = div(X) \wedge f_2(X) = |X|$, which are optimized by the GSEMO, respectively. The better one between the two generated solutions is output as the final solution. We prove that after running at most $enk(k+1)$ and $enk^2$ expected number of iterations, respectively, for the two bi-objective problems, the GSEMO can achieve an approximation ratio of $1/4$ (i.e., {\bf Theorem~\ref{theo-cardinality-min}}), which reaches the best-known one~\cite{dasgupta2013summarization}. But if we use the straightforward bi-objective reformulation, i.e., $f_1(X)=f(X)+\lambda \cdot div(X)$ and $f_2(X)=-|X|$, the GSEMO will fail to guarantee an approximation ratio of $1/4$ in polynomial time (i.e., {\bf Theorem~\ref{theo-cardinality-min-v}}).\\
      (2) \;When the diversity measure is Eq.~(\refeq{eq-max-mst}), we set $f_1(X)=f(X)+\lambda \cdot \sum^{|X|}_{i=2} \min_{v \in \{u_1,\ldots,u_{i-1}\}} d(u_i,v)$ and $f_2(X)=-|X|$, where $(u_1,\ldots,u_{|X|})$ is a permutation of the items in $X$. We prove that after running at most $enk(k+1)$ expected number of iterations, the GSEMO can achieve an approximation ratio of $(1-1/e)(1/(2\log k))\approx 0.316/\log k$ (i.e., {\bf Theorem~\ref{theo-cardinality-mst}}), which is slightly worse than the best-known one, $1/(3\log k)$~\cite{dasgupta2013summarization}.
 \item For the result diversification problem with a matroid constraint, i.e., Eq.~(\refeq{eq-case-5}), we set $f_1(X)=f(X)+\lambda \cdot div(X)$ and $f_2(X)=|X|$. That is, the GSEMO is to maximize the original objective function $f(X)+\lambda \cdot div(X)$ and the subset size $|X|$ simultaneously. We prove that after running at most $O(rn^3(n+(r\log r)/\epsilon))$  expected number of iterations, the GSEMO can achieve an approximation ratio of $1/2-\epsilon/(4n)$ (i.e., {\bf Theorem~\ref{theo-matroid-sum}}), which is asymptotically optimal~\cite{borodin2017max}. If the metric distance $d$ is relaxed to satisfy Eq.~(\refeq{eq-mid-35}), the approximation ratio of the GSEMO becomes $(1/2-\epsilon/(4n))/\alpha^2$ (i.e., {\bf Corollary~\ref{coro:matroid}}).
  \item For the general problem Eq.~(\refeq{eq-case-5}) under dynamic environments, i.e., where the objective function changes over time, we prove that the GSEMO can maintain an asymptotically optimal approximation ratio of $1/2-\epsilon/(4n)$ by running $O(rn^3(n+(r\log r)/\epsilon))$ expected number of iterations (i.e., {\bf Theorem~\ref{theo-dynamic}}), which addresses the open question proposed in~\cite{borodin2017max}. Compared with the local search algorithm~\cite{borodin2017max}, the advantage of the GSEMO adapting to dynamic changes is mainly due to the employed bit-wise mutation operator, which has a good exploration ability. To the best of our knowledge, this is the first time to theoretically show the superiority of evolutionary algorithms (EAs) over local search for solving dynamic optimization problems.\vspace{-0.8em}
\end{itemize}

We perform experiments on the applications of web-based search~\cite{borodin2017max}, multi-label feature selection~\cite{ghadiri2019distributed} and document summarization~\cite{dasgupta2013summarization}. The results show that the GSEMO can always achieve the best objective value, and is significantly better than the greedy algorithm and local search in most cases. Compared with the running time derived in theoretical analyses, the GSEMO requires much less running time to achieve a better performance in practice, and can be even faster than the greedy algorithm and local search sometimes. Under dynamic environments, the GSEMO also exhibits the superiority clearly.

We start the rest of the paper by formally introducing the result diversification problem and the GSEMO algorithm. We then present in three subsequent sections the theoretical analyses for the cardinality constraint, the matroid constraint and dynamic environments, following by the section of empirical study. The final section concludes this paper.

\section{Result Diversification}\label{sec-problem}

Let $\mathbb{R}$ and $\mathbb{R}^{+}$ denote the set of reals and non-negative reals, respectively. Let $V=\{v_1,v_2,\ldots,v_n\}$ denote a ground set of items. A set function $f:2^V \rightarrow \mathbb{R}$ is defined on subsets of $V$, and maps any subset to a real value. The goal of result diversification is to select a subset of items from $V$ with high quality and diversity while satisfying some constraints.

The quality can be usually characterized by a monotone submodular set function. A set function $f:2^V \rightarrow \mathbb{R}$ is monotone (non-decreasing) if $\forall X \subseteq Y \subseteq V: f(X) \leq f(Y)$. Without loss of generality, we assume that monotone functions are normalized, i.e., $f(\emptyset)=0$.
\begin{definition}[Submodularity~\citep{nemhauser1978analysis}]
A set function $f:2^V \rightarrow \mathbb{R}$ is submodular if
\begin{align}\label{def-submodular}
\forall X,Y\subseteq V: f(X)+f(Y) \geq f(X\cup Y) +f(X \cap Y);
\end{align}
or equivalently
\begin{align}\label{def-submodular-1}
\forall X \subseteq Y \subseteq V, v \notin Y: f(X \cup \{v\})-f(X) \geq f(Y \cup \{v\}) - f(Y);
\end{align}
or equivalently
\begin{align}\label{def-submodular-2}
\forall X \subseteq Y \subseteq V: f(Y)-f(X) \leq \sum_{v \in Y \setminus X} \big(f(X \cup \{v\})-f(X)\big).
\end{align}\vspace{-2em}
\end{definition}
Eq.~(\refeq{def-submodular-1}) intuitively represents the diminishing returns property, i.e., the benefit of adding an item to a set will not increase as the set extends. The quality measure in real-world applications (e.g., the relevance of the selected documents to a query in web-based search~\cite{borodin2017max}, the coverage of the selected sentences in document summarization~\cite{dasgupta2013summarization}, and the relevance of the selected features to the labels in multi-label feature selection~\cite{ghadiri2019distributed}) often satisfies the monotone submodular property. A set function $f:2^V \rightarrow \mathbb{R}$ is modular if Eq.~(\refeq{def-submodular}), Eq.~(\refeq{def-submodular-1}) or Eq.~(\refeq{def-submodular-2}) holds with equality. For a modular function $f$, it holds that $\forall X \subseteq V: f(X)=\sum_{v \in X} f(\{v\})$.

The diversity relies on a distance function $d: V \times V \rightarrow \mathbb{R}^+$ between items. We consider that distances are metrics in this paper. A distance function $d$ is a metric if
\begin{align}\label{eq-triangle}
&\forall u, v \in V: d(u,v)\ge 0;\nonumber\\
&\forall u, v \in V: d(u,v)=0 \iff u=v;\nonumber\\
&\forall u,v \in V: d(u,v)=d(v,u);\nonumber\\
&\forall u,v,w \in V: d(u,v)+d(v,w) \geq d(u,w).
\end{align}
The commonly used diversity is the sum-diversity in Definition~\ref{def-sum-div}, which measures the sum of the distances between items. Other measures of diversity include the min-diversity in Definition~\ref{def-min-div} and the mst-diversity in Definition~\ref{def-mst-div}, which are the minimum distance between items and the weight of the minimum spanning tree, respectively. It can be verified that the sum-diversity is monotone non-decreasing; the min-diversity is monotone non-increasing; the mst-diversity is not necessarily monotone non-decreasing or non-increasing. Note that all these three diversity measures are non-submodular.
\begin{definition}[Sum-Diversity]\label{def-sum-div}
Given a distance function $d: V \times V \rightarrow \mathbb{R}^+$, the sum-diversity of a subset $X$ of items is
\begin{align}
div(X)=\sum_{\{u,v\}: u,v \in X} d(u,v).
\end{align}
\end{definition}
\begin{definition}[Min-Diversity]\label{def-min-div}
Given a distance function $d: V \times V \rightarrow \mathbb{R}^+$, the min-diversity of a subset $X$ of items is
\begin{align}
div(X)=\min_{\{u,v\}: u,v \in X} d(u,v).
\end{align}
\end{definition}
\begin{definition}[MST-Diversity]\label{def-mst-div}
Given a distance function $d: V \times V \rightarrow \mathbb{R}^+$, let $G=(V,E)$ denote a complete graph, where each vertex corresponds to an item, and the weight of each edge corresponds to the distance between two items. The mst-diversity $div(X)$ of a subset $X$ of items is the weight of the minimum spanning tree of $X$ on the graph $G$.
\end{definition}

The result diversification problem with a cardinality constraint is shown in Definition~\ref{def-prob-cardinality}. It is to select a subset of size at most $k$ maximizing the linear combination of a quality function $f$ and a diversity measure $div$. The parameter $\lambda$ tradeoffs the two terms of quality and diversity. As $div$ is non-submodular, the objective function $f+\lambda \cdot div$ is non-submodular in general.
\begin{definition}[Result Diversification with a Cardinality Constraint]\label{def-prob-cardinality}
Given a monotone submodular function $f: 2^V \rightarrow \mathbb{R}^+$, a distance function $d: V \times V \rightarrow \mathbb{R}^+$, a diversity measure $div: 2^V \rightarrow \mathbb{R}^+$, a parameter $\lambda \in \mathbb{R}^+$ and a budget $k$, to find a subset $X\subseteq V$ such that
\begin{align}\label{eq-problem-cardinality}
\arg\max\nolimits_{X \subseteq V} f(X)+\lambda \cdot div(X) \quad s.t. \quad |X| \leq k.
\end{align}
\end{definition}

When the diversity $div$ is the sum-diversity (i.e., the sum of the distances between items) in Definition~\ref{def-sum-div}, Borodin \textit{et al.}~\cite{borodin2017max} proposed the greedy algorithm in Algorithm~\ref{alg:greedy-sum}. The greedy algorithm starts from the empty set, and iteratively selects one item with the largest marginal gain, until $k$ items are selected. Note that this greedy algorithm is non-oblivious, because the marginal gain is based on $f/2+\lambda \cdot div$ instead of the original objective function $f+ \lambda \cdot div$. Thus, in each iteration, the algorithm selects the item maximizing $(f(X \cup \{u\})-f(X))/2 + \lambda \cdot (div(X \cup \{u\})-div(X))=(f(X \cup \{u\})-f(X))/2+\lambda \cdot \sum_{v \in X} d(u,v)$. This greedy algorithm can achieve an approximation ratio of $1/2$~\cite{borodin2017max}. That is, the returned subset $X_k$ satisfies $f(X_k)+\lambda \cdot div(X_k) \geq \mathrm{OPT}/2$, where $\mathrm{OPT}$ denotes the optimal function value. Note that $1/2$ has already been the optimal polynomial-time approximation ratio, because it has been shown that $1/2$ is optimal even for the special case of maximizing only the sum-diversity with a cardinality constraint, i.e., Eq.~(\refeq{eq-case-1}). By relaxing the triangle inequality (i.e., Eq.~(\refeq{eq-triangle})) satisfied by the distance $d$ to
\begin{align}\label{eq-triangle-relaxed}
\forall u,v,w \in V: \alpha \cdot (d(u,v)+d(v,w)) \geq d(u,w),
\end{align}
where $\alpha \geq 1$, the approximation ratio of the greedy algorithm changes to $1/(2\alpha)$~\cite{zadeh2015max}.

\begin{algorithm}
[h!]\caption{Greedy Algorithm for Sum-Diversity~\cite{borodin2017max}}
\label{alg:greedy-sum}
\textbf{Input}: monotone submodular $f: 2^V \rightarrow \mathbb{R}^+$, distance $d: V \times V \rightarrow \mathbb{R}^+$, $\lambda \in \mathbb{R}^+$, and budget $k$\\
\textbf{Output}: a subset of $V$ with $k$ items\\
\textbf{Process}:
\begin{algorithmic}[1]
\STATE Let $i=0$ and $X_0=\emptyset$;
\STATE \textbf{repeat}
\STATE \quad Let $u^*=\arg\max_{u \in V \setminus X_i} (f(X_i \cup \{u\})-f(X_i))/2+\lambda \cdot \sum_{v \in X_i} d(u,v)$;
\STATE \quad Let $X_{i+1}=X_{i} \cup \{u^*\}$, and $i=i+1$
\STATE \textbf{until} $i=k$
\STATE \textbf{return} $X_k$
\end{algorithmic}
\end{algorithm}

When the diversity $div$ is the min-diversity (i.e., the minimum distance between items) in Definition~\ref{def-min-div}, Dasgupta \textit{et al.}~\cite{dasgupta2013summarization} proposed the greedy algorithm in Algorithm~\ref{alg:greedy-min}. Note that the min-diversity is monotone non-increasing, and thus the constraint in Definition~\ref{def-prob-cardinality} is strictly $|X|=k$. This greedy algorithm maximizes $f$ and $div$, respectively. It can be observed that the item with the largest marginal gain on $f$ is selected in line~3 of Algorithm~\ref{alg:greedy-min}, and the item maximizing the distance to the current set (which will lead to the largest $div$) is selected in line~4. In line~7, the better one between the two generated subsets is finally returned, which has an approximation ratio of $1/4$~\cite{dasgupta2013summarization}.

\begin{algorithm}
[h!]\caption{Greedy Algorithm for Min-Diversity~\cite{dasgupta2013summarization}}
\label{alg:greedy-min}
\textbf{Input}: monotone submodular $f: 2^V \rightarrow \mathbb{R}^+$, distance $d: V \times V \rightarrow \mathbb{R}^+$, $\lambda \in \mathbb{R}^+$, and budget $k$\\
\textbf{Output}: a subset of $V$ with $k$ items\\
\textbf{Process}:
\begin{algorithmic}[1]
\STATE Let $i=0$, $X_0=\emptyset$ and $Y_0=\emptyset$;
\STATE \textbf{repeat}
\STATE \quad Let $u_1^*=\arg\max_{u \in V \setminus X_i} f(X_i \cup \{u\})-f(X_i)$;
\STATE \quad Let $u_2^*=\arg\max_{u \in V \setminus Y_i} \min_{v \in Y_i} d(u,v)$;
\STATE \quad Let $X_{i+1}=X_{i} \cup \{u_1^*\}$, $Y_{i+1}=Y_{i} \cup \{u_2^*\}$, and $i=i+1$
\STATE \textbf{until} $i=k$
\STATE \textbf{return} $\arg\max_{X \in \{X_k,Y_k\}} f(X)+\lambda \cdot \min_{\{u,v\}: u,v \in X} d(u,v)$
\end{algorithmic}
\end{algorithm}

Dasgupta \textit{et al.}~\cite{dasgupta2013summarization} also considered the problem where the diversity $div$ is the mst-diversity (i.e., the weight of the minimum spanning tree) in Definition~\ref{def-mst-div}. The cardinality constraint is also strictly $|X|=k$, as the mst-diversity does not have the monotone non-decreasing property in general. They proposed the greedy algorithm in Algorithm~\ref{alg:greedy-mst}. The marginal gain on $div$ by adding a single item $u$ to $X$ is the weight difference between the two minimum spanning trees of $X\cup \{u\}$ and $X$, which is approximated by the distance of $u$ to $X$, i.e., $\min_{v \in X} d(u,v)$. Thus, in line~3 of Algorithm~\ref{alg:greedy-mst}, the item maximizing $f(X \cup \{u\})-f(X)+\lambda \cdot \min_{v \in X} d(u,v)$ is selected. This greedy algorithm can achieve an approximation ratio of $1/(3\log k)$~\cite{dasgupta2013summarization}.

\begin{algorithm}
[h!]\caption{Greedy Algorithm for MST-Diversity~\cite{dasgupta2013summarization}}
\label{alg:greedy-mst}
\textbf{Input}: monotone submodular $f: 2^V \rightarrow \mathbb{R}^+$, distance $d: V \times V \rightarrow \mathbb{R}^+$, $\lambda \in \mathbb{R}^+$, and budget $k$\\
\textbf{Output}: a subset of $V$ with $k$ items\\
\textbf{Process}:
\begin{algorithmic}[1]
\STATE Let $i=0$ and $X_0=\emptyset$;
\STATE \textbf{repeat}
\STATE \quad Let $u^*=\arg\max_{u \in V \setminus X_i} f(X_i \cup \{u\})-f(X_i)+\lambda \cdot \min_{v \in X_i} d(u,v)$;
\STATE \quad Let $X_{i+1}=X_{i} \cup \{u^*\}$, and $i=i+1$
\STATE \textbf{until} $i=k$
\STATE \textbf{return} $X_k$
\end{algorithmic}
\end{algorithm}

The result diversification problem with a matroid constraint is shown in Definition~\ref{def-prob-matroid}. A matroid is a pair $(V,\mathcal{F})$, where $V$ is a finite set and $\mathcal{F} \subseteq 2^{V}$, satisfying
\begin{align}
(1) & \quad \emptyset \in \mathcal{F}; \\
(2) &\quad \forall X \subseteq Y \in \mathcal{F}: X \in \mathcal{F};\\
(3) & \quad \forall X,Y\in \mathcal{F}, |X|>|Y|: \exists v \in X\setminus Y, Y \cup \{v\} \in \mathcal{F}.
\end{align}
The elements of $\mathcal{F}$ (i.e., the subsets of $V$ belonging to $\mathcal{F}$) are called \emph{independent}. For any $X \subseteq V$, a maximal independent subset of $X$ is called a basis of $X$; the rank of $X$ is the maximal cardinality of a basis of $X$, i.e., $r(X)=\max\{|Y| \mid Y \subseteq X, Y \in \mathcal{F}\}$. The rank of $V$ is also called the rank of the matroid. Note that for a matroid, all bases of a subset $X\subseteq V$ have the same cardinality.
\begin{definition}[Result Diversification with a Matroid Constraint]\label{def-prob-matroid}
Given a monotone submodular function $f: 2^V \rightarrow \mathbb{R}^+$, a distance function $d: V \times V \rightarrow \mathbb{R}^+$, a diversity measure $div: 2^V \rightarrow \mathbb{R}^+$, a parameter $\lambda \in \mathbb{R}^+$ and a matroid $(V,\mathcal{F})$, to find a subset $X\subseteq V$ such that
\begin{align}\label{eq-problem-matroid}
\arg\max\nolimits_{X \subseteq V} f(X)+\lambda \cdot div(X) \quad s.t. \quad X \in \mathcal{F}.
\end{align}
\end{definition}

It is easy to verify that a cardinality constraint is actually a uniform matroid $(V,\mathcal{F})$ with $\mathcal{F}=\{X \subseteq V \mid |X| \leq k\}$. A general matroid constraint can characterize more requirements on the selected subset of items. For example, given a partition of $V$ into $m$ disjoint subsets $V_1,V_2,\ldots,V_m$, a partition matroid $(V,\mathcal{F})$ with $\mathcal{F}=\{X \subseteq V \mid \forall 1 \leq i \leq m: |X \cap V_i| \leq k_i\}$ can be used to ensure that the retrieved items (e.g., documents) come from a variety of different sources (e.g., topics).

The problem with a matroid constraint has only been studied with the sum-diversity in Definition~\ref{def-sum-div}. The greedy algorithm in Algorithm~\ref{alg:greedy-sum} can still be applied by requiring the newly added item to keep the independent constraint. However, Borodin \textit{et al.}~\cite{borodin2017max} showed that the greedy algorithm now fails to achieve any constant approximation ratio even for the special case of $f=0$. They proved that the local search algorithm can achieve a $(1/2)$-approximation ratio. Because the sum-diversity is monotone, the objective function $f+\lambda \cdot div$ is also monotone, and thus an optimal solution must be a basis of the ground set $V$. As presented in Algorithm~\ref{alg:local-search}, the local search algorithm starts from a basis of $V$ containing the best two items, and iteratively improves the basis by swapping an item in the basis with an item outside the basis. Because $1/2$ is the optimal polynomial-time approximation ratio for the problem with a cardinality constraint, it is also optimal for the more general problem with a matroid constraint. To make local search run in polynomial time, one can require at least an $\epsilon$-improvement at each iteration rather than just any improvement, which will result in a small sacrifice on the approximation ratio. But when $f$ or $d$ changes dynamically, it is not yet clear whether the local search algorithm can maintain a $(1/2)$-approximation ratio in polynomial running time~\cite{borodin2017max}.

\begin{algorithm}
[h!]\caption{Local Search Algorithm for Sum-Diversity~\cite{borodin2017max}}
\label{alg:local-search}
\textbf{Input}: monotone submodular $f: 2^V \rightarrow \mathbb{R}^+$, distance $d: V \times V \rightarrow \mathbb{R}^+$, $\lambda \in \mathbb{R}^+$, and matroid $(V,\mathcal{F})$\\
\textbf{Output}: a basis of $V$\\
\textbf{Process}:
\begin{algorithmic}[1]
\STATE Let $\{u^*,v^*\}=\arg \max_{\{u,v\} \in \mathcal{F}} f(\{u,v\}) + \lambda \cdot d(u,v)$;
\STATE Let $X$ be a basis of $V$ containing $u^*$ and $v^*$;
\STATE \textbf{while} there is $u \in V \setminus X$ and $v \in X$ such that $X \cup \{u\} \setminus \{v\} \in \mathcal{F}$ and $g(X \cup \{u\} \setminus \{v\}) > g(X)$
\STATE \quad $X=X \cup \{u\} \setminus \{v\}$
\STATE \textbf{end while}
\STATE \textbf{return} $X$
\end{algorithmic}
where $g(X)=f(X) + \lambda \cdot \sum_{\{u,v\}: u,v \in X} d(u,v)$ denotes the objective function.
\end{algorithm}

\section{Multi-objective Evolutionary Algorithms}

To solve the result diversification problem in Definitions~\ref{def-prob-cardinality} and~\ref{def-prob-matroid}, we will use a general framework based on MOEAs. That is, the original problem in Definitions~\ref{def-prob-cardinality} and~\ref{def-prob-matroid} is reformulated as a bi-objective maximization problem $\arg\max\nolimits_{X \subseteq V} (f_1(X),f_2(X))$, which is solved by a simple MOEA, the GSEMO~\cite{laumanns2004running}, and the best feasible solution w.r.t. the original problem, in the population found by the GSEMO, is output as the final solution. Such a framework has shown good approximation performance for solving diverse submodular optimization problems~\cite{friedrich2015maximizing,friedrich2018heavy,qian2020multi,qian2017constrained,qian2019maximizing}.

As presented in Algorithm~\ref{algo:GSEMO}, the GSEMO is used for maximizing multiple pseudo-Boolean objective functions simultaneously. Because a subset $X$ of $V$ can be represented by a Boolean vector $\bm{x} \in \{0,1\}^n$, where the $i$-th bit $x_i=1$ iff $v_i \in X$, a pseudo-Boolean function $f: \{0,1\}^n \rightarrow \mathbb{R}$ naturally characterizes a set function $f: 2^{V} \rightarrow \mathbb{R}$. In this paper, we will not distinguish $\bm{x}\in \{0,1\}^n$ and its corresponding subset $\{v_i \in V \mid x_i=1\}$ for notational convenience.

In multi-objective maximization $\max\, (f_1,f_2,\ldots,f_m)$, solutions may be incomparable due to the conflicting of objectives. The domination relationship in Definition~\ref{def_Domination} is often used for comparison.
\begin{definition}[Domination]\label{def_Domination}
For two solutions $\bm x$ and $\bm{x}'$,
\begin{enumerate}
  \item $\bm{x}$ \emph{weakly dominates} $\bm{x}'$ (i.e., $\bm{x}$ is \emph{better} than $\bm{x}'$, denoted by $\bm{x} \succeq \bm{x}'$) if \;$\forall i: f_i(\bm{x}) \geq f_i(\bm{x}')$;
  \item ${\bm{x}}$ \emph{dominates} $\bm{x}'$ (i.e., $\bm{x}$ is \emph{strictly better} than $\bm{x}'$, denoted by $\bm{x} \succ \bm{x}'$) if ${\bm{x}} \succeq \bm{x}' \wedge \exists i: f_i(\bm{x}) > f_i(\bm{x}')$;
  \item $\bm{x}$ and $\bm{x}'$ are \emph{incomparable} if neither $\bm{x} \succeq \bm{x}'$ nor $\bm{x}' \succeq \bm{x}$.
\end{enumerate}
\end{definition}

The GSEMO starts from the all-0s vector $\bm{0}$ (i.e., the empty set) in line~1, and iteratively improves the quality of solutions in the population $P$ (lines~2--8). In each iteration, a parent solution $\bm{x}$ is selected from $P$ uniformly at random (line~3), and used to generate an offspring solution $\bm{x}'$ by bit-wise mutation (line~4), which flips each bit of $\bm{x}$ independently with probability $1/n$. The newly generated offspring solution $\bm{x}'$ is then used to update the population $P$ (lines~5--7). If $\bm{x}'$ is not dominated by any parent solution in $P$ (line~5), it will be added into $P$, and those parent solutions weakly dominated by $\bm{x}'$ will be deleted (line~6). By this updating procedure, the solutions contained in the population $P$ are always incomparable. Note that the GSEMO here is a little different from the original version in~\cite{laumanns2004running}, which starts from an initial solution selected from $\{0,1\}^n$ uniformly at random. In most of our theoretical analyses, the empty set $\bm{0}$ is required. Though the GSEMO using a random initial solution can probably generate the solution $\bm{0}$, it will cost extra running time. Thus, we use $\bm{0}$ as the initial solution directly for efficiency.

\begin{algorithm}[h!]\caption{GSEMO Algorithm}\label{algo:GSEMO}
\textbf{Input}: $m$ pseudo-Boolean functions $f_1,f_2,\ldots,f_m$, where $f_i: \{0,1\}^n \rightarrow \mathbb{R}$\\
\textbf{Output}: a feasible subset of $V$\\
\textbf{Process}:
    \begin{algorithmic}[1]
    \STATE Let $P \gets \{\bm 0\}$;
    \STATE \textbf{repeat}
    \STATE \quad Choose $\bm x$ from $P$ uniformly at random;
    \STATE \quad Create $\bm{x}'$ by flipping each bit of $\bm x$ with probability $1/n$;
    \STATE \quad \textbf{if} \, {$\nexists \bm z \in P$ such that $\bm z \succ \bm {x}'$} \,\textbf{then}
    \STATE  \qquad $P \gets (P \setminus \{\bm z \in P \mid \bm {x}' \succeq \bm z\}) \cup \{\bm {x}'\}$
    \STATE \quad \textbf{end if}
    \STATE \textbf{until} some criterion is met
    \STATE \textbf{return} the best feasible solution in $P$
    \end{algorithmic}
\end{algorithm}

When the GSEMO terminates, the best feasible solution w.r.t. the original result diversification problem will be selected from the resulting population $P$ as the final solution. For example, for result diversification with a cardinality constraint in Definition~\ref{def-prob-cardinality}, it will return $\arg\max_{\bm{x} \in P, |\bm{x}|\leq k} f(\bm{x})+\lambda \cdot div(\bm{x})$ in line~9 of Algorithm~\ref{algo:GSEMO}; for result diversification with a matroid constraint in Definition~\ref{def-prob-matroid}, it will return $\arg\max_{\bm{x} \in P \cap \mathcal{F}} f(\bm{x})+\lambda \cdot div(\bm{x})$.

In the next two sections, we will show how to reformulate the result diversification problem with a cardinality constraint or a matroid constraint as a bi-objective maximization problem
\begin{align}\label{eq-bi-problem}
&\arg\max\nolimits_{\bm{x} \in \{0,1\}^n} \quad  (f_1(\bm{x}),f_2(\bm{x})),
\end{align}
and analyze the expected number of iterations of the GSEMO required to reach some approximation guarantee for the first time. Note that bi-objective reformulation here is an intermediate process, and our focus is the quality of the best feasible solution w.r.t. the original single-objective problem, in the population found by the GSEMO, rather than the quality of the population w.r.t. the reformulated bi-objective problem.

\section{MOEAs for Result Diversification with a Cardinality Constraint}\label{sec-cardinality}

In this section, we analyze the approximation guarantee of the GSEMO for result diversification with a cardinality constraint in Definition~\ref{def-prob-cardinality}. The most commonly used diversity measure, i.e., the sum-diversity in Definition~\ref{def-sum-div}, is first considered. To employ the GSEMO, the problem, i.e., Eq.~(\refeq{eq-problem-cardinality}), in Definition~\ref{def-prob-cardinality} is transformed into a bi-objective maximization problem
\begin{align}\label{eq-bi-problem-1}
&\arg\max\nolimits_{\bm{x} \in \{0,1\}^n} \quad  (f_1(\bm{x}),f_2(\bm{x})),\\
&\text{where}\quad\begin{cases}
f_1(\bm{x}) = \frac{1}{2}\left(1+\frac{|\bm{x}|}{k}\right)f(\bm{x})+\lambda \cdot div(\bm{x}),\\
f_2(\bm x) = -|\bm{x}|.
\end{cases}
\end{align}
Note that $|\bm{x}|=\sum^n_{i=1}x_i$ denotes the size of $\bm{x}$. Thus, the GSEMO is to maximize the objective function $f_1$ (which gradually increases the importance of $f$ as $\bm{x}$ extends) and minimize the subset size $|\bm{x}|$ simultaneously. The infeasible solutions, i.e., the subsets with size larger than $k$, are excluded during the optimization process of the GSEMO.

Let $\mathrm{OPT}$ denote the optimal function value of Eq.~(\ref{eq-problem-cardinality}). Theorem~\ref{theo-cardinality-sum} shows that the GSEMO can achieve a $(1/2)$-approximation ratio after running at most $enk^3/2$ expected number of iterations. This reaches the optimal polynomial-time approximation ratio~\cite{borodin2017max}.

\begin{theorem}\label{theo-cardinality-sum}
For the result diversification problem with a cardinality constraint in Definition~\ref{def-prob-cardinality}, where the diversity measure $div$ is the sum-diversity in Definition~\ref{def-sum-div}, the expected number of iterations of the GSEMO using Eq.~(\refeq{eq-bi-problem-1}), until finding a solution $\bm{x}$ with $|\bm{x}| \leq k$ and $f(\bm{x})+\lambda \cdot div(\bm{x}) \geq (1/2) \cdot \mathrm{OPT}$, is at most $enk^3/2$.
\end{theorem}

Let $\bm{x}^*$ denote an optimal solution of Eq.~(\refeq{eq-problem-cardinality}), i.e., $f(\bm{x}^*) + \lambda \cdot div(\bm{x}^*)=\mathrm{OPT}$. Without loss of generality, we assume that $|\bm{x}^*|=k$ since the objective function $f+\lambda \cdot div$ is monotone. The proof relies on Lemma~\ref{lemma:cardinality-sum}, which shows that for any $\bm{x} \in \{0,1\}^n$ with $|\bm{x}|<k$, there always exists one item whose inclusion can improve $f(\bm{x})/2+\lambda \cdot div(\bm{x})$ by at least some quantity relating to $f(\bm{x}^*)$ and $div(\bm{x}^*)$. For the sum-diversity $div(X)=\sum_{\{u,v\}: u,v \in X} d(u,v)$, we extend the notation to measure the diversity between two disjoint sets by $div(X,Y)=\sum_{\{u,v\}: u \in X, v \in Y} d(u,v)$. Lemma~\ref{lemma:sum-div} gives the relation between $div(X)$ and $div(X,Y)$, and will be used in the proof of Lemma~\ref{lemma:cardinality-sum}.

\begin{lemma}[Lemma~2 in~\cite{ravi1994heuristic}]\label{lemma:sum-div}
Given a metric distance function $d$, and two disjoint sets $X$ and $Y$, it holds that \begin{align}\label{eq-mid-18}(|X|-1)\cdot div(X,Y) \geq |Y|\cdot div(X).\end{align}
\end{lemma}

\begin{lemma}\label{lemma:cardinality-sum}
For any $\bm{x} \in \{0,1\}^n$ with $|\bm{x}| < k$, there exists one item $v \notin \bm{x}$ such that
\begin{align}\label{eq-mid-19}
& \frac{f(\bm{x} \cup \{v\})}{2}+\lambda \cdot div(\bm{x} \cup \{v\})-\frac{f(\bm{x})}{2}-\lambda \cdot div(\bm{x}) \geq \frac{f(\bm{x}^*)-f(\bm{x})}{2k} +\frac{\lambda |\bm{x}|}{k(k-1)}\cdot div(\bm{x}^*).
\end{align}
\end{lemma}
\begin{proof}
By the submodularity (i.e., Eq.~(\ref{def-submodular-2})) of $f$, we have
\begin{align}\label{eq-mid-9}
& \sum_{v \in \bm{x}^*\setminus \bm{x}} \left(\frac{f(\bm{x} \cup \{v\})}{2}+\lambda \cdot div(\bm{x} \cup \{v\})-\frac{f(\bm{x})}{2}-\lambda \cdot div(\bm{x})\right)\\
& \geq \frac{f(\bm{x}\cup \bm{x}^*)-f(\bm{x})}{2} + \lambda \cdot div(\bm{x}^*\setminus \bm{x}, \bm{x}).
\end{align}

Now we derive a lower bound on $div(\bm{x}^*\setminus \bm{x}, \bm{x})$. By applying Lemma~\ref{lemma:sum-div} repeatedly, we have
\begin{align}
&(|\bm{x}^*\setminus \bm{x}|-1)\cdot div(\bm{x}^*\setminus \bm{x}, \bm{x} \setminus \bm{x}^*) \geq |\bm{x} \setminus \bm{x}^*| \cdot div(\bm{x}^*\setminus \bm{x}),\label{eq-mid-1}\\
&(|\bm{x}^*\setminus \bm{x}|-1)\cdot div(\bm{x}^*\setminus \bm{x}, \bm{x}^* \cap \bm{x}) \geq |\bm{x}^* \cap \bm{x}| \cdot div(\bm{x}^*\setminus \bm{x}),\numberthis\label{eq-mid-2}\\
&(|\bm{x}^*\cap\bm{x}|-1)\cdot div(\bm{x}^* \cap \bm{x}, \bm{x}^*\setminus \bm{x}) \geq |\bm{x}^* \setminus \bm{x}| \cdot div(\bm{x}^*\cap \bm{x}).\numberthis\label{eq-mid-3}
\end{align}
Note that $div(\bm{x}^*\setminus \bm{x}, \bm{x}^* \cap \bm{x})=div(\bm{x}^* \cap \bm{x}, \bm{x}^*\setminus \bm{x})$ by the symmetry. According to the definition of sum-diversity, we also have
\begin{align}
&div(\bm{x}^* \cap \bm{x}, \bm{x}^*\setminus \bm{x})+div(\bm{x}^*\cap \bm{x})+div(\bm{x}^*\setminus \bm{x})=div(\bm{x}^*).\label{eq-mid-4}
\end{align}
When $|\bm{x}^*\setminus \bm{x}|>1$, we multiply Eqs.~(\refeq{eq-mid-1}) to~(\refeq{eq-mid-4}) by the factors of $\frac{1}{|\bm{x}^*\setminus \bm{x}|-1}$, $\frac{|\bm{x}^*\setminus \bm{x}|-|\bm{x}\setminus \bm{x}^*|}{k(|\bm{x}^*\setminus \bm{x}|-1)}$, $\frac{|\bm{x}|}{k(k-1)}$ and $\frac{|\bm{x}|\cdot |\bm{x}^*\setminus \bm{x}|}{k(k-1)}$, respectively. That is,
\begin{align}
\!\!\!\!\!\!\!\!\!div(\bm{x}^*\setminus \bm{x}, \bm{x} \setminus \bm{x}^*) &\geq \frac{|\bm{x} \setminus \bm{x}^*|}{|\bm{x}^*\setminus \bm{x}|-1} \cdot div(\bm{x}^*\setminus \bm{x}),\label{eq-mid-1-v}\\
\!\!\!\!\!\!\!\!\!\frac{|\bm{x}^*\setminus \bm{x}|-|\bm{x}\setminus \bm{x}^*|}{k}\cdot div(\bm{x}^*\setminus \bm{x}, \bm{x}^* \cap \bm{x}) &\geq  \frac{|\bm{x}^* \cap \bm{x}|\cdot (|\bm{x}^*\setminus \bm{x}|-|\bm{x}\setminus \bm{x}^*|)}{k(|\bm{x}^*\setminus \bm{x}|-1)} \cdot div(\bm{x}^*\setminus \bm{x}),\numberthis\label{eq-mid-2-v}\\
\!\!\!\!\!\!\!\!\!\frac{|\bm{x}|\cdot (|\bm{x}^*\cap\bm{x}|-1)}{k(k-1)}\cdot div(\bm{x}^* \cap \bm{x}, \bm{x}^*\setminus \bm{x}) &\geq  \frac{|\bm{x}|\cdot |\bm{x}^*\setminus \bm{x}|}{k(k-1)} \cdot div(\bm{x}^*\cap \bm{x}),\numberthis\label{eq-mid-3-v}\\
\!\!\!\!\!\!\!\!\!\frac{|\bm{x}|\cdot |\bm{x}^*\setminus \bm{x}|}{k(k-1)}\cdot div(\bm{x}^* \cap \bm{x}, \bm{x}^*\setminus \bm{x})&=\frac{|\bm{x}|\cdot |\bm{x}^*\setminus \bm{x}|}{k(k-1)} \cdot \left(div(\bm{x}^*)-div(\bm{x}^*\cap \bm{x})-div(\bm{x}^*\setminus \bm{x})\right)\!.\label{eq-mid-4-v}
\end{align}
By summing up the left-hand side of Eqs.~(\refeq{eq-mid-1-v}) to~(\refeq{eq-mid-4-v}) as well as their right-hand side, we get
\begin{align}
&\left( \frac{|\bm{x}^*\setminus \bm{x}|-|\bm{x}\setminus \bm{x}^*|}{k}+ \frac{|\bm{x}|\cdot (|\bm{x}^*\cap\bm{x}|-1)}{k(k-1)}+\frac{|\bm{x}|\cdot |\bm{x}^*\setminus \bm{x}|}{k(k-1)}\right) \cdot div(\bm{x}^* \cap \bm{x}, \bm{x}^*\setminus \bm{x}) \\
&+div(\bm{x}^*\setminus \bm{x}, \bm{x} \setminus \bm{x}^*) \geq \frac{|\bm{x}|\cdot |\bm{x}^*\setminus \bm{x}|}{k(k-1)}\cdot div(\bm{x}^*)\\
&+ \left(\frac{|\bm{x} \setminus \bm{x}^*|}{|\bm{x}^*\setminus \bm{x}|-1} + \frac{|\bm{x}^* \cap \bm{x}|\cdot (|\bm{x}^*\setminus \bm{x}|-|\bm{x}\setminus \bm{x}^*|)}{k(|\bm{x}^*\setminus \bm{x}|-1)}-\frac{|\bm{x}|\cdot |\bm{x}^*\setminus \bm{x}|}{k(k-1)}\right) \cdot div(\bm{x}^*\setminus \bm{x}).\label{eq-mid-5}
\end{align}
Note that $|\bm{x}^*|=k$. Because
\begin{align}
&\frac{|\bm{x}^*\setminus \bm{x}|-|\bm{x}\setminus \bm{x}^*|}{k}+ \frac{|\bm{x}|\cdot (|\bm{x}^*\cap\bm{x}|-1)}{k(k-1)}+\frac{|\bm{x}|\cdot |\bm{x}^*\setminus \bm{x}|}{k(k-1)}\\
&=\frac{|\bm{x}^*\setminus \bm{x}|-|\bm{x}\setminus \bm{x}^*|}{k}+ \frac{|\bm{x}|\cdot (|\bm{x}^*|-1)}{k(k-1)}\\
&=\frac{|\bm{x}^*\setminus \bm{x}|-|\bm{x}\setminus \bm{x}^*|}{k}+ \frac{|\bm{x}|}{k}=\frac{|\bm{x}^*|}{k}=1,
\end{align}
and
\begin{align}
&\frac{|\bm{x} \setminus \bm{x}^*|}{|\bm{x}^*\setminus \bm{x}|-1} + \frac{|\bm{x}^* \cap \bm{x}|\cdot (|\bm{x}^*\setminus \bm{x}|-|\bm{x}\setminus \bm{x}^*|)}{k(|\bm{x}^*\setminus \bm{x}|-1)}-\frac{|\bm{x}|\cdot |\bm{x}^*\setminus \bm{x}|}{k(k-1)}\\
&=\frac{|\bm{x} \setminus \bm{x}^*|\cdot (|\bm{x}^* \cap \bm{x}|+|\bm{x}^*\setminus \bm{x}|)+ |\bm{x}^* \cap \bm{x}|\cdot (|\bm{x}^*\setminus \bm{x}|-|\bm{x}\setminus \bm{x}^*|)}{k(|\bm{x}^*\setminus \bm{x}|-1)}-\frac{|\bm{x}|\cdot |\bm{x}^*\setminus \bm{x}|}{k(k-1)}\\
&=\frac{|\bm{x}|\cdot |\bm{x}^*\setminus \bm{x}|}{k(|\bm{x}^*\setminus \bm{x}|-1)}-\frac{|\bm{x}|\cdot |\bm{x}^*\setminus \bm{x}|}{k(k-1)}\geq 0,
\end{align}
Eq.~(\refeq{eq-mid-5}) changes to
\begin{align}
div(\bm{x}^* \cap \bm{x}, \bm{x}^*\setminus \bm{x})+div(\bm{x}^*\setminus \bm{x}, \bm{x} \setminus \bm{x}^*) \geq \frac{|\bm{x}|\cdot |\bm{x}^*\setminus \bm{x}|}{k(k-1)}\cdot div(\bm{x}^*).
\end{align}
Because $div(\bm{x}^*\setminus \bm{x}, \bm{x})=div(\bm{x}^*\setminus \bm{x}, \bm{x}^* \cap \bm{x})+div(\bm{x}^*\setminus \bm{x}, \bm{x} \setminus \bm{x}^*)$, we have
\begin{align}\label{eq-mid-6}
div(\bm{x}^*\setminus \bm{x}, \bm{x})\geq \frac{|\bm{x}|\cdot |\bm{x}^*\setminus \bm{x}|}{k(k-1)}\cdot div(\bm{x}^*).
\end{align}
When $|\bm{x}^*\setminus \bm{x}|=1$, it must hold that $\bm{x} \subseteq \bm{x}^*$ and $|\bm{x}|=k-1$. Thus, Eqs.~(\ref{eq-mid-3}) and~(\ref{eq-mid-4}) change to
\begin{align}
&(k-2)\cdot div(\bm{x}^* \cap \bm{x}, \bm{x}^*\setminus \bm{x}) \geq  div(\bm{x}^*\cap \bm{x}),\\
&div(\bm{x}^* \cap \bm{x}, \bm{x}^*\setminus \bm{x})+div(\bm{x}^*\cap \bm{x})=div(\bm{x}^*).
\end{align}
By summing them up, we have
\begin{align}
&(k-1)\cdot div(\bm{x}^* \cap \bm{x}, \bm{x}^*\setminus \bm{x}) \geq div(\bm{x}^*).
\end{align}
Because $\bm{x} \subseteq \bm{x}^*$, it holds that $\bm{x}^* \cap \bm{x}=\bm{x}$. Thus,
\begin{align}\label{eq-mid-7}
div(\bm{x}^*\setminus \bm{x}, \bm{x})\geq \frac{1}{k-1}\cdot div(\bm{x}^*) \geq \frac{1}{k}\cdot div(\bm{x}^*)=\frac{|\bm{x}|\cdot |\bm{x}^*\setminus \bm{x}|}{k(k-1)}\cdot div(\bm{x}^*),
\end{align}
where the equality holds by $|\bm{x}^*\setminus \bm{x}|=1$ and $|\bm{x}|=k-1$.

According to the above analysis under the two cases of $|\bm{x}^*\setminus \bm{x}|>1$ and $|\bm{x}^*\setminus \bm{x}|=1$, we have derived Eqs.~(\refeq{eq-mid-6}) and~(\refeq{eq-mid-7}). That is, $div(\bm{x}^*\setminus \bm{x}, \bm{x})$ is lower bounded by
\begin{align}\label{eq-mid-8}
div(\bm{x}^*\setminus \bm{x}, \bm{x})\geq \frac{|\bm{x}|\cdot |\bm{x}^*\setminus \bm{x}|}{k(k-1)}\cdot div(\bm{x}^*).
\end{align}
Applying this inequality to Eq.~(\refeq{eq-mid-9}) leads to
\begin{align}
& \sum_{v \in \bm{x}^*\setminus \bm{x}} \left(\frac{f(\bm{x} \cup \{v\})}{2}+\lambda \cdot div(\bm{x} \cup \{v\})-\frac{f(\bm{x})}{2}-\lambda \cdot div(\bm{x})\right)\\
& \geq \frac{f(\bm{x}\cup \bm{x}^*)-f(\bm{x})}{2} + \lambda \cdot \frac{|\bm{x}|\cdot |\bm{x}^*\setminus \bm{x}|}{k(k-1)}\cdot div(\bm{x}^*).
\end{align}
Let $v^*=\arg \max_{v \in \bm{x}^*\setminus \bm{x}} \frac{f(\bm{x} \cup \{v\})}{2}+\lambda \cdot div(\bm{x} \cup \{v\})-\frac{f(\bm{x})}{2}-\lambda \cdot div(\bm{x})$. Thus, we have
\begin{align}
&\frac{f(\bm{x} \cup \{v^*\})}{2}+\lambda \cdot div(\bm{x} \cup \{v^*\})-\frac{f(\bm{x})}{2}-\lambda \cdot div(\bm{x})\\
 &\geq \frac{f(\bm{x}\cup \bm{x}^*)-f(\bm{x})}{2|\bm{x}^*\setminus \bm{x}|} + \frac{\lambda|\bm{x}|}{k(k-1)}\cdot div(\bm{x}^*)\\
&\geq \frac{f(\bm{x}\cup \bm{x}^*)-f(\bm{x})}{2k} + \frac{\lambda |\bm{x}|}{k(k-1)}\cdot div(\bm{x}^*)\\
&\geq \frac{f(\bm{x}^*)-f(\bm{x})}{2k} + \frac{\lambda |\bm{x}|}{k(k-1)}\cdot div(\bm{x}^*),
\end{align}
where the second inequality holds by $|\bm{x}^*\setminus \bm{x}| \leq k$ and $f(\bm{x}\cup \bm{x}^*)\geq f(\bm{x})$ due to the monotonicity of $f$, and the last inequality holds by $f(\bm{x}\cup \bm{x}^*)\geq f(\bm{x}^*)$. The lemma holds.
\end{proof}

The proof of Lemma~\ref{lemma:cardinality-sum} is inspired by the analysis of the greedy algorithm, i.e., the proof of Theorem~4.1 in~\cite{borodin2017max}. In the original proof of~\cite{borodin2017max}, the conclusion Eq.~(\refeq{eq-mid-19}) of Lemma~\ref{lemma:cardinality-sum} has been derived by adding a single item greedily to the solution $\bm{x}$ maintained by the greedy algorithm and assuming $|\bm{x}^*\setminus \bm{x}|>1$. Here we show that Eq.~(\refeq{eq-mid-19}) holds for any $\bm{x}$ with $|\bm{x}|<k$. The main difference from their original proof is that we consider the case $|\bm{x}^*\setminus \bm{x}|=1$ additionally when analyzing the lower bound on $div(\bm{x}^*\setminus \bm{x}, \bm{x})$. For the sake of completeness, we give a full proof here.

\begin{myproof}{Theorem~\ref{theo-cardinality-sum}}
The initial solution $\bm{0}$ will always be kept in the population $P$, because $\bm{0}$ has the largest $f_2$ value (i.e., $f_2(\bm{0})=0$), and no solution can weakly dominate it. To analyze the expected number of iterations until reaching the desired approximation guarantee, we consider a quantity $J_{\max}$, which is defined as
\begin{align}
&J_{\max}=\max\left\{j \in \{0,1,\ldots,k\} \mid \exists \bm{x} \in P: |\bm{x}| \leq j  \wedge  f_1(\bm{x}) \geq \frac{j}{2k}f(\bm{x}^*)+\frac{\lambda\cdot div(\bm{x}^*)}{k(k-1)}\sum^{j-1}_{m=1} m \right\}.
\end{align}
It can be seen that $J_{\max}=k$ implies that there exists one solution $\bm{x}$ in $P$ satisfying that $|\bm{x}| \leq k$ and
\begin{align}\label{eq-mid-10}
f_1(\bm{x}) \geq \frac{k}{2k}f(\bm{x}^*)+\frac{\lambda\cdot div(\bm{x}^*)}{k(k-1)}\sum^{k-1}_{m=1} m=\frac{1}{2}f(\bm{x}^*)+\frac{\lambda}{2}\cdot div(\bm{x}^*)=\frac{\mathrm{OPT}}{2}.
\end{align}
According to the definition of $f_1$ in Eq.~(\refeq{eq-bi-problem-1}) and $|\bm{x}| \leq k$, we have
\begin{align}\label{eq-mid-11}
f_1(\bm{x}) = \frac{1}{2}\left(1+\frac{|\bm{x}|}{k}\right)f(\bm{x})+\lambda \cdot div(\bm{x}) \leq f(\bm{x})+\lambda \cdot div(\bm{x}).
\end{align}
Combining Eqs.~(\refeq{eq-mid-10}) and~(\refeq{eq-mid-11}) leads to
\begin{align}\label{eq-mid-20}
f(\bm{x})+\lambda \cdot div(\bm{x}) \geq \mathrm{OPT}/2,
\end{align}
that is, the desired approximation guarantee is reached. Next, we only need to analyze the expected number of iterations until $J_{\max}=k$.

As the population $P$ contains the solution $\bm{0}$, which satisfies that $|\bm{0}|=0$ and $f_1(\bm{0})=0$, $J_{\max}$ is at least 0. Assume that currently $J_{\max}=i <k$, implying that $P$ contains solutions satisfying that $|\bm{x}|\leq i$ and
\begin{align}\label{eq-mid-12}
f_1(\bm{x}) \geq \frac{i}{2k}f(\bm{x}^*)+\frac{\lambda\cdot div(\bm{x}^*)}{k(k-1)}\sum^{i-1}_{m=1} m.
\end{align}
Let $\hat{\bm{x}}$ be the one with the largest $f_1$ value among these solutions, which is actually the solution with size at most $i$ and the largest $f_1$ value in $P$. First, $J_{\max}$ will not decrease. If $\hat{\bm{x}}$ is deleted from $P$ in line~6 of Algorithm~\ref{algo:GSEMO}, the newly included solution $\bm{x}'$ must weakly dominate $\hat{\bm{x}}$, implying that $|\bm{x}'| \leq |\hat{\bm{x}}|$ and $f_1(\bm{x}')\geq f_1(\hat{\bm{x}})$.

Second, we analyze the expected number of iterations required to increase $J_{\max}$. We consider such an event in one iteration of Algorithm~\ref{algo:GSEMO}: $\hat{\bm{x}}$ is selected for mutation in line~3, and only one specific 0-bit corresponding to the item $v$ in Lemma~\ref{lemma:cardinality-sum} is flipped in line~4. This event is called ``a successful event", occurring with probability $(1/|P|)\cdot (1/n)(1-1/n)^{n-1} \geq 1/(en|P|)$ due to uniform selection and bit-wise mutation. According to the procedure of updating the population $P$ in the GSEMO, the solutions maintained in $P$ must be incomparable. Because two solutions with the same value on one objective are comparable, $P$ contains at most one solution for each value of one objective. As the infeasible solutions, i.e., the solutions with size larger than $k$, are excluded, $f_2(\bm{x})=-|\bm{x}|$ can take values $0,-1,\ldots,-k$, implying $|P| \leq k+1$. Thus, the probability of a successful event is at least $1/(en(k+1))$. According to Lemma~\ref{lemma:cardinality-sum}, the offspring solution $\bm{x}'$ generated by a successful event satisfies
\begin{align}\label{eq-mid-13}
& \frac{f(\bm{x}')}{2}+\lambda \cdot div(\bm{x}')-\frac{f(\hat{\bm{x}})}{2}-\lambda \cdot div(\hat{\bm{x}}) \geq \frac{f(\bm{x}^*)-f(\hat{\bm{x}})}{2k} +\frac{\lambda |\hat{\bm{x}}|}{k(k-1)}\cdot div(\bm{x}^*).
\end{align}
By the definition of $f_1$ in Eq.~(\refeq{eq-bi-problem-1}), we have
\begin{align}\label{eq-mid-14}
f_1(\bm{x}')-f_1(\hat{\bm{x}})&= \frac{1}{2}\left(1+\frac{|\bm{x}'|}{k}\right)f(\bm{x}')+\lambda \cdot div(\bm{x}')- \frac{1}{2}\left(1+\frac{|\hat{\bm{x}}|}{k}\right)f(\hat{\bm{x}})-\lambda \cdot div(\hat{\bm{x}}) \\
&\geq \frac{f(\bm{x}')}{2}+\lambda \cdot div(\bm{x}')-\frac{f(\hat{\bm{x}})}{2}-\lambda \cdot div(\hat{\bm{x}}) +\frac{f(\hat{\bm{x}})}{2k}\\
&\geq \frac{f(\bm{x}^*)}{2k} +\frac{\lambda |\hat{\bm{x}}|}{k(k-1)}\cdot div(\bm{x}^*),
\end{align}
where the first inequality holds by $|\bm{x}'|=|\hat{\bm{x}}|+1$ and $f(\bm{x}') \geq f(\hat{\bm{x}})$ due to the monotonicity of $f$, and the second inequality holds by Eq.~(\refeq{eq-mid-13}). By applying the lower bound of $f_1(\hat{\bm{x}})$ in Eq.~(\refeq{eq-mid-12}) to Eq.~(\refeq{eq-mid-14}), we have
\begin{align}\label{eq-mid-15}
f_1(\bm{x}') \geq \frac{i+1}{2k}f(\bm{x}^*)+\frac{\lambda\cdot div(\bm{x}^*)}{k(k-1)}\left(|\hat{\bm{x}}|+\sum^{i-1}_{m=1} m \right).
\end{align}
We next consider two cases according to the value of $|\hat{\bm{x}}|$, satisfying $|\hat{\bm{x}}| \leq i$.

(1) $|\hat{\bm{x}}| = i$. Eq.~(\refeq{eq-mid-15}) becomes
\begin{align}\label{eq-mid-17}
f_1(\bm{x}') \geq \frac{i+1}{2k}f(\bm{x}^*)+\frac{\lambda\cdot div(\bm{x}^*)}{k(k-1)}\sum^{i}_{m=1} m.
\end{align}
Note that $|\bm{x}'|=|\hat{\bm{x}}|+1 = i+1$. Then, $\bm{x}'$ will be added into $P$; otherwise, $\bm{x}'$ must be dominated by one solution in $P$ (line~5 of Algorithm~\ref{algo:GSEMO}), and this implies that $J_{\max}$ has already been larger than $i$, contradicting the assumption $J_{\max}=i$. After including $\bm{x}'$, $J_{\max} \geq i+1$, i.e., $J_{\max}$ increases.

(2) $|\hat{\bm{x}}| < i$. This implies that $i \geq 1$. Thus, to make Eq.~(\refeq{eq-mid-12}) hold, $|\hat{\bm{x}}|$ must be at least 1. According to Eq.~(\refeq{eq-mid-14}), we have
\begin{align}\label{eq-mid-16}
f_1(\bm{x}')-f_1(\hat{\bm{x}}) \geq \frac{f(\bm{x}^*)}{2k} +\frac{\lambda}{k(k-1)}\cdot div(\bm{x}^*).
\end{align}
Note that $\bm{x}'$ (where $|\bm{x}'|=|\hat{\bm{x}}|+1 \leq i$) will be added into $P$; otherwise, $\bm{x}'$ must be dominated by one solution in $P$, contradicting the definition of $\hat{\bm{x}}$, which is the solution with size at most $i$ and the largest $f_1$ value in $P$. If $f_1(\bm{x}') \geq \frac{i+1}{2k}f(\bm{x}^*)+\frac{\lambda\cdot div(\bm{x}^*)}{k(k-1)}\sum^{i}_{m=1} m$, $J_{\max}$ increases. Otherwise, the solution $\hat{\bm{x}}$ now becomes $\bm{x}'$, and $f_1(\hat{\bm{x}})$ increases by at least $\frac{f(\bm{x}^*)}{2k} +\frac{\lambda}{k(k-1)}\cdot div(\bm{x}^*)$ according to Eq.~(\refeq{eq-mid-16}).

Based on the above analysis, a successful event will either increase $J_{\max}$ directly or increase $f_1(\hat{\bm{x}})$ by at least $\frac{f(\bm{x}^*)}{2k} +\frac{\lambda}{k(k-1)}\cdot div(\bm{x}^*)$. It is easy to see that $f_1(\hat{\bm{x}})$ will not decrease due to the domination-based comparison. It is also known from Eq.~(\refeq{eq-mid-17}) that $f_1(\hat{\bm{x}})$ needs to increase at most $\frac{f(\bm{x}^*)}{2k} +i\cdot \frac{\lambda}{k(k-1)}\cdot div(\bm{x}^*)$ for increasing $J_{\max}$. Thus, the number of successful events required to increase $J_{\max}$ is at most $i$. A successful event occurs with probability at least $1/(en(k+1))$ in one iteration, implying that the expected number of iterations for one successful event is at most $en(k+1)$. Thus, the expected number of iterations to make $J_{\max} \geq i+1$ (i.e., increase $J_{\max}$) is at most $i \cdot en(k+1)$.

To make $J_{\max}=k$, it is sufficient to increase $J_{\max}$ from $0$ to $k$ step-by-step, implying that the expected number of iterations until $J_{\max}=k$ is at most $en(k+1)\cdot \sum^{k-1}_{i=0} i \leq enk^3/2$. The theorem holds.
\end{myproof}

From the above analysis, we can also find the reason of multiplying $f$ by a factor of $(1+|\bm{x}|/k)/2$ in the setting of $f_1$. The term $|\bm{x}|/k$ gradually increases the importance of $f$ as $\bm{x}$ extends, and thus can compensate for the term $-f(\bm{x})/2k$ in the right-hand side of Eq.~(\refeq{eq-mid-19}). Without the factor of $1/2$, i.e., $f_1(\bm{x})=(1+|\bm{x}|/k)f(\bm{x})+\lambda \cdot div(\bm{x})$, we can modify Eq.~(\refeq{eq-mid-19}) and the inductive inequality in the definition of $J_{\max}$ as $f(\bm{x} \cup \{v\})+\lambda \cdot div(\bm{x} \cup \{v\})-f(\bm{x})-\lambda \cdot div(\bm{x}) \geq \frac{f(\bm{x}^*)-f(\bm{x})}{k} +\frac{\lambda |\bm{x}|}{k(k-1)}\cdot div(\bm{x}^*)$ and $f_1(\bm{x}) \geq \frac{j}{k}f(\bm{x}^*)+\frac{\lambda\cdot div(\bm{x}^*)}{k(k-1)}\sum^{j-1}_{m=1} m$ accordingly, and can still apply the proof process of Theorem~\ref{theo-cardinality-sum}, but the conclusion will be $2f(\bm{x})+\lambda \cdot div(\bm{x}) \geq f(\bm{x}^*) + (\lambda/2)\cdot div(\bm{x}^*)$, which cannot guarantee an approximation ratio of $1/2$ on the original objective function $f(\bm{x})+\lambda \cdot div(\bm{x})$.

We have required the distance function $d$ to be a metric, i.e., satisfy the triangle inequality Eq.~(\refeq{eq-triangle}). By relaxing it to
\begin{align}\label{eq-triangle-relax}\forall u,v,w \in V: \alpha \cdot (d(u,v)+d(v,w)) \geq d(u,w),
\end{align}
where $\alpha \geq 1$, we can get an approximation ratio $1/(2\alpha)$ of the GSEMO, as shown in Corollary~\ref{coro:cardinality}.

\begin{corollary}\label{coro:cardinality}
For the result diversification problem with a cardinality constraint in Definition~\ref{def-prob-cardinality}, where the diversity measure $div$ is the sum-diversity in Definition~\ref{def-sum-div} and the distance function $d$ satisfies Eq.~(\refeq{eq-triangle-relax}), the expected number of iterations of the GSEMO using Eq.~(\refeq{eq-bi-problem-1}), until finding a solution $\bm{x}$ with $|\bm{x}| \leq k$ and $f(\bm{x})+\lambda \cdot div(\bm{x}) \geq (1/(2\alpha)) \cdot \mathrm{OPT}$, is at most $enk^3/2$.
\end{corollary}

This corollary can be proved by following the proof of Theorem~\ref{theo-cardinality-sum}. For concise illustration, we will mainly show the places where different arguments are needed. The relaxed triangle inequality Eq.~(\refeq{eq-triangle-relax}) makes Eq.~(\refeq{eq-mid-18}) in Lemma~\ref{lemma:sum-div} change to
\begin{align}\alpha \cdot (|X|-1)\cdot div(X,Y) \geq |Y|\cdot div(X),\end{align}
which results in a corresponding change of Eq.~(\refeq{eq-mid-19}) in Lemma~\ref{lemma:cardinality-sum}. That is, Eq.~(\refeq{eq-mid-19}) changes to
\begin{align}
& \frac{f(\bm{x} \cup \{v\})}{2}+\lambda \cdot div(\bm{x} \cup \{v\})-\frac{f(\bm{x})}{2}-\lambda \cdot div(\bm{x}) \geq \frac{f(\bm{x}^*)-f(\bm{x})}{2k} +\frac{\lambda |\bm{x}|}{\alpha k(k-1)}\cdot div(\bm{x}^*).
\end{align}
The proof procedure of Theorem~\ref{theo-cardinality-sum} then can be performed by changing the inductive inequality on $f_1$ in the definition of the quantity $J_{\max}$ to
\begin{align}
f_1(\bm{x}) \geq \frac{j}{2k}f(\bm{x}^*)+\frac{\lambda\cdot div(\bm{x}^*)}{\alpha k(k-1)}\sum^{j-1}_{m=1} m.
\end{align}
Eq.~(\refeq{eq-mid-10}) will change accordingly to
\begin{align}
f_1(\bm{x}) \geq \frac{k}{2k}f(\bm{x}^*)+\frac{\lambda\cdot div(\bm{x}^*)}{\alpha k(k-1)}\sum^{k-1}_{m=1} m=\frac{1}{2}f(\bm{x}^*)+\frac{\lambda}{2\alpha}\cdot div(\bm{x}^*)\geq \frac{\mathrm{OPT}}{2\alpha},
\end{align}
where the last inequality holds by $\alpha \geq 1$. Thus, the approximation ratio in Eq.~(\refeq{eq-mid-20}) changes from $1/2$ to $1/(2\alpha)$, implying that the corollary holds. Note that $1/(2\alpha)$ reaches the best-known approximation ratio~\cite{zadeh2015max}.

To prove the optimal polynomial-time approximation ratio of $1/2$ in Theorem~\ref{theo-cardinality-sum}, the first objective $f_1$ to be maximized by the GSEMO has been set to $(1+|\bm{x}|/k)f(\bm{x})/2+\lambda \cdot div(\bm{x})$ as in Eq.~(\refeq{eq-bi-problem-1}). A natural question is whether the GSEMO using the original objective function (i.e., $f_1(\bm{x})=f(\bm{x})+\lambda \cdot div(\bm{x})$) can still achieve this approximation ratio, by adopting a different analysis from that in the proof of Theorem~\ref{theo-cardinality-sum}. The bi-objective reformulation of the problem in Definition~\ref{def-prob-cardinality} now is
\begin{align}\label{eq-bi-problem-1-v}
&\arg\max\nolimits_{\bm{x} \in \{0,1\}^n} \quad  (f_1(\bm{x}),f_2(\bm{x})),\\
&\text{where}\quad\begin{cases}
f_1(\bm{x}) = f(\bm{x})+\lambda \cdot div(\bm{x}),\\
f_2(\bm x) = -|\bm{x}|.
\end{cases}
\end{align}
We prove in Theorem~\ref{theo-cardinality-sum-v} that the GSEMO using Eq.~(\refeq{eq-bi-problem-1-v}) can achieve an approximation ratio arbitrarily close to $1/2$ in polynomial time.
\begin{theorem}\label{theo-cardinality-sum-v}
For the result diversification problem with a cardinality constraint in Definition~\ref{def-prob-cardinality}, where the diversity measure $div$ is the sum-diversity in Definition~\ref{def-sum-div}, the expected number of iterations of the GSEMO using Eq.~(\refeq{eq-bi-problem-1-v}), until finding a solution $\bm{x}$ with $|\bm{x}| \leq k$ and $f(\bm{x})+\lambda \cdot div(\bm{x}) \geq (1/2-\epsilon/(4n)) \cdot \mathrm{OPT}$, is at most $O(kn^3(n+(k\log k)/\epsilon))$, where $\epsilon >0$.
\end{theorem}

The main proof idea is to utilize the behavior of both greedy (i.e., adding a single item maximizing $f+\lambda \cdot div$) and local search (i.e., swapping two items) operations, which can be accomplished by the bit-wise mutation operator of the GSEMO. In Section~\ref{sec-matroid}, we will study the more general problem, i.e., result diversification with a matroid constraint in Definition~\ref{def-prob-matroid}, and show that the GSEMO maximizing $f_1(\bm{x})=f(\bm{x})+\lambda \cdot div(\bm{x})$ and $f_2(\bm{x})=|\bm{x}|$ simultaneously can achieve an approximation ratio of $1/2-\epsilon/(4n)$ by following the behavior of local search. As a cardinality constraint is a specific matroid constraint, the general proof of following local search can be directly applied. For concise illustration, we thus leave the proof of Theorem~\ref{theo-cardinality-sum-v} at the end of Section~\ref{sec-matroid}. But also note that the settings of $f_2$ (which are $-|\bm{x}|$ and $|\bm{x}|$ for the cardinality and matroid constraints, respectively) are different, making that following local search only is insufficient while following the greedy behavior is also needed.

Next, we consider the result diversification problem in Definition~\ref{def-prob-cardinality} where the diversity measure $div$ is min-diversity and mst-diversity, respectively.

\subsection{Min-Diversity}

Because the min-diversity in Definition~\ref{def-min-div} is monotone non-increasing, the cardinality constraint in Eq.~(\refeq{eq-problem-cardinality}) is strictly $|X|=k$ as in~\cite{dasgupta2013summarization}. That is, the problem in Definition~\ref{def-prob-cardinality} changes to
\begin{align}\label{eq-problem-cardinality-strict}
\arg\max\nolimits_{X \subseteq V} f(X)+\lambda \cdot div(X) \quad s.t. \quad |X| = k.
\end{align}
To solve this problem using the GSEMO, we reformulate it as two bi-objective maximization problems,
\begin{align}\label{eq-bi-problem-2-1}
&\arg\max\nolimits_{\bm{x} \in \{0,1\}^n} \quad  (f_1(\bm{x}),f_2(\bm{x})),\\
&\text{where}\quad\begin{cases}
f_1(\bm{x}) = f(\bm{x}),\\
f_2(\bm x) = -|\bm{x}|,
\end{cases}
\end{align}
and
\begin{align}\label{eq-bi-problem-2-2}
&\arg\max\nolimits_{\bm{x} \in \{0,1\}^n} \quad  (f_1(\bm{x}),f_2(\bm{x})),\\
&\text{where}\quad\begin{cases}
f_1(\bm{x}) = div(\bm{x}),\\
f_2(\bm x) = |\bm{x}|.
\end{cases}
\end{align}
For the former, the GSEMO is to maximize the quality function $f$ and minimize the subset size $|\bm{x}|$ simultaneously, while for the latter, it is to maximize both the diversity function $div$ and the subset size $|\bm{x}|$ simultaneously. Note that all the solutions with size larger than $k$ are excluded during the running of the GSEMO. For the latter problem Eq.~(\refeq{eq-bi-problem-2-2}), $div$ is not defined for the empty set $\bm{0}$ and the subsets of size 1 (i.e., containing only a single item), and we define their values of $f_1$ to be $+\infty$.

Let $T_1$ and $T_2$ denote the number of iterations employed by the GSEMO to solve Eqs.~(\refeq{eq-bi-problem-2-1}) and~(\refeq{eq-bi-problem-2-2}), respectively. For solving Eq.~(\refeq{eq-bi-problem-2-1}), after the GSEMO terminates, the solution (denoted as $\bm{x}_1$) with the largest size in the population will be returned. If $|\bm{x}_1|<k$, $k-|\bm{x}_1|$ items will be arbitrarily selected from the remaining items and added into $\bm{x}_1$, which will not worsen the $f$ value due to its monotonicity. For solving Eq.~(\refeq{eq-bi-problem-2-2}), the GSEMO will return the solution (denoted as $\bm{x}_2$) with size $k$ in the population. After obtaining $\bm{x}_1$ and $\bm{x}_2$, the better one between them w.r.t. the original problem Eq.~(\refeq{eq-problem-cardinality-strict}), i.e., $\arg\max_{\bm{x} \in \{\bm{x}_1, \bm{x}_2\}} f(\bm{x})+\lambda \cdot div(\bm{x})$, will be output as the final solution. We use $\mathbb{E}[\cdot]$ to denote the expectation of a random variable. Theorem~\ref{theo-cardinality-min} shows that a $(1/4)$-approximation ratio can be achieved by the GSEMO using $\mathbb{E}[T_1] \leq enk(k+1)$ and $\mathbb{E}[T_2] \leq enk^2$. Note that this reaches the best-known polynomial-time approximation ratio~\cite{dasgupta2013summarization}.

Lemma~\ref{lemma:cardinality-min} shows that for any two subsets $\bm{x}$ and $\bm{y}$ with $1 \leq |\bm{x}| < |\bm{y}|$, there exists an item $u$ whose distance to $\bm{x}$ (i.e., $\min_{v \in \bm{x}} d(u,v)$) is at least half of the diversity of $\bm{y}$ (i.e., $div(\bm{y})/2$). It will be used in the proof of Theorem~\ref{theo-cardinality-min}.

\begin{lemma}[Lemma~1(iii) in~\cite{dasgupta2013summarization}]\label{lemma:cardinality-min}
Let $\bm{y} \in \{0,1\}^n$ denote any solution with size larger than 1. For any $\bm{x} \in \{0,1\}^n$ with $1 \leq |\bm{x}| < |\bm{y}|$, there exists one item $u \notin \bm{x}$ such that
\begin{align}
& \min_{v \in \bm{x}} d(u,v)  \geq div(\bm{y})/2.
\end{align}
\end{lemma}

\begin{theorem}\label{theo-cardinality-min}
For the result diversification problem with a cardinality constraint in Definition~\ref{def-prob-cardinality}, where the diversity measure $div$ is the min-diversity in Definition~\ref{def-min-div}, the expected number of iterations of the GSEMO using Eqs.~(\refeq{eq-bi-problem-2-1}) and~(\refeq{eq-bi-problem-2-2}), until finding a solution $\bm{x}$ with $|\bm{x}| = k$ and $f(\bm{x})+\lambda \cdot div(\bm{x}) \geq (1/4) \cdot \mathrm{OPT}$, satisfies that $\mathbb{E}[T_1] \leq enk(k+1)$ and $\mathbb{E}[T_2] \leq enk^2$.
\end{theorem}
\begin{proof}
Let $\bm{x}^*_1$ denote an optimal solution of $\max_{\bm{x}: |\bm{x}|=k} f(\bm{x})$. Because $f$ is monotone submodular, it has been proved in Theorem~2 of~\cite{friedrich2015maximizing} that after the GSEMO optimizing Eq.~(\ref{eq-bi-problem-2-1}) for at most $enk(k+1)$ expected number of iterations (i.e., $\mathbb{E}[T_1] \leq enk(k+1)$), the found solution $\bm{x}_1$ satisfies
\begin{align}\label{eq-mid-21}
f(\bm{x}_1) \geq (1-1/e)\cdot f(\bm{x}^*_1).
\end{align}
Note that in their setting, the GSEMO starts from an initial solution uniformly selected from $\{0,1\}^n$ at random. Their proof first derives an upper bound $O(n^2\log n)$ on the expected number of iterations of the GSEMO for finding the all 0s solution $\bm{0}$, and then shows that starting from $\bm{0}$, the remaining number of iterations for achieving an approximation ratio of $1-1/e$ is $O(n^2k)$ in expectation. While in our setting, the initial solution of the GSEMO is just $\bm{0}$, and thus the term $O(n^2\log n)$ can be deleted. Furthermore, the population size is upper bounded by $n+1$ in their proof, while here it is upper bounded by $k+1$ because the solutions with size larger than $k$ are excluded in the optimization. We have modified the upper bound on $\mathbb{E}[T_1]$ accordingly due to these differences.

Let $\bm{x}^*_2$ denote an optimal solution of $\max_{\bm{x}: |\bm{x}|=k} div(\bm{x})$. Next, we will show that after the GSEMO optimizing Eq.~(\ref{eq-bi-problem-2-2}) for at most $enk^2$ expected number of iterations (i.e., $\mathbb{E}[T_2] \leq enk^2$), the found solution $\bm{x}_2$ satisfies
\begin{align}\label{eq-mid-22}
div(\bm{x}_2) \geq (1/2)\cdot div(\bm{x}^*_2).
\end{align}
We consider a quantity $J_{\max}$, which is defined as
\begin{align}
J_{\max}=\max\{j \in \{2,3,\ldots,k\} \mid \exists \bm{x} \in P: |\bm{x}| = j \wedge div(\bm{x}) \geq (1/2)\cdot div(\bm{x}^*_2)\}.
\end{align} That is, $J_{\max}$ denotes the maximum value of $j \in \{2,3,\ldots,k\}$ such that in the population $P$, there exists a solution $\bm{x}$ with $|\bm{x}| = j$ and $div(\bm{x}) \geq (1/2)\cdot div(\bm{x}^*_2)$. We only need to analyze the expected number of iterations until $J_{\max}=k$, because it implies that there exists one solution $\bm{x}$ in $P$ satisfying that $|\bm{x}| = k$ and $div(\bm{x}) \geq (1/2)\cdot div(\bm{x}^*_2)$, i.e., Eq.~(\refeq{eq-mid-22}) holds.

It is known from Algorithm~\ref{algo:GSEMO} that the GSEMO starts from the empty set $\bm{0}$ (having $f_1(\bm{0})=+\infty$ and $f_2(\bm{0})=0$), which is dominated by only a solution $\bm{x}$ with size 1 (having $f_1(\bm{x})=+\infty$ and $f_2(\bm{x})=1$). By selecting $\bm{0}$ for mutation in line~3 of Algorithm~\ref{algo:GSEMO} and flipping only one 0-bit in line~4, which occur with probability $(1/|P|) \cdot (n/n)(1-1/n)^{n-1} \geq 1/(e(k+1))$, a solution with size $1$ is generated. After that, the population $P$ will always contain a solution (denoted as $\bm{x}$) with size 1, because no other solution can dominate it. By setting $\bm{y}$ in Lemma~\ref{lemma:cardinality-min} to $\bm{x}^*_2$, we know that flipping one specific 0-bit of $\bm{x}$ (i.e., adding a specific item), which occurs with probability $(1/|P|) \cdot (1/n)(1-1/n)^{n-1} \geq 1/(en(k+1))$, can generate a new solution $\bm{x}'$, which satisfies that $|\bm{x}'|=2$ and $div(\bm{x}') \geq (1/2)\cdot div(\bm{x}^*_2)$. If $\bm{x}'$ is added into the population $P$, $J_{\max} \geq 2$. Otherwise, $\bm{x}'$ is dominated by one solution in $P$, implying that $J_{\max}$ has already been at least 2. Thus, the expected number of iterations of the GSEMO until $J_{\max} \geq 2$ is at most $e(k+1)+en(k+1)=e(n+1)(k+1)$.

After that, assume currently $J_{\max}=i <k$. Let $\bm{x}$ be the corresponding solution with the value $i$, i.e., $|\bm{x}|= i$ and $div(\bm{x}) \geq (1/2)\cdot div(\bm{x}^*_2)$. It is easy to see that $J_{\max}$ cannot decrease because deleting $\bm{x}$ from $P$ (line~6 of Algorithm~\ref{algo:GSEMO}) implies that $\bm{x}$ is weakly dominated by a newly generated solution $\bm{x}'$, which must satisfy that $|\bm{x}'| \geq |\bm{x}|=i$ and $div(\bm{x}')\geq div(\bm{x})\geq (1/2)\cdot div(\bm{x}^*_2)$. Again by setting $\bm{y}$ in Lemma~\ref{lemma:cardinality-min} to $\bm{x}^*_2$, we know that flipping one specific 0-bit (the corresponding item is denoted as $u$) of $\bm{x}$ can generate a new solution $\bm{x}'=\bm{x} \cup \{u\}$, satisfying $\min_{v \in \bm{x}} d(u,v)  \geq div(\bm{x}^*_2)/2$. According to the definition of min-diversity in Definition~\ref{def-min-div}, we have \begin{align}div(\bm{x}')=\min\{div(\bm{x}),\min\nolimits_{v \in \bm{x}} d(u,v)\} \geq (1/2)\cdot div(\bm{x}^*_2),\end{align}
where the inequality holds by $div(\bm{x}) \geq (1/2)\cdot div(\bm{x}^*_2)$ and $\min_{v \in \bm{x}} d(u,v)  \geq div(\bm{x}^*_2)/2$. Since $|\bm{x}'|=|\bm{x}|+1 = i+1$, $\bm{x}'$ will be included into $P$; otherwise, $\bm{x}'$ must be dominated by one solution in $P$ (line~5 of Algorithm~\ref{algo:GSEMO}), and this implies that $J_{\max}$ has already been larger than $i$, contradicting the assumption $J_{\max}=i$. After including $\bm{x}'$, $J_{\max} = i+1$. Thus, $J_{\max}$ can increase by 1 in one iteration with probability at least $(1/|P|) \cdot (1/n)(1-1/n)^{n-1} \geq 1/(en(k+1))$, where $1/|P|$ is the probability of selecting $\bm{x}$ in line~3 of Algorithm~\ref{algo:GSEMO}, $(1/n)(1-1/n)^{n-1}$ is the probability of flipping a specific bit of $\bm{x}$ while keeping other bits unchanged in line~4, and the inequality holds because the population size $|P| \leq k+1$. This implies that it needs at most $en(k+1)$ expected number of iterations to increase $J_{\max}$. Thus, after at most $(k-2) \cdot en(k+1)$ iterations in expectation, $J_{\max}$ must have reached $k$. Combining with $e(n+1)(k+1)$, the upper bound on the expected number of iterations required to make $J_{\max} \geq 2$, we get $\mathbb{E}[T_2] \leq e(n+1)(k+1)+ en(k-2)(k+1)\leq enk^2$.

Finally, we will show that the output solution, i.e., the better one between $\bm{x}_1$ and $\bm{x}_2$ w.r.t. the original objective $f(\bm{x})+\lambda \cdot div(\bm{x})$, has a $(1/4)$-approximation ratio. Let $\bm{x}^*$ denote an optimal solution of the original problem $\max_{\bm{x}: |\bm{x}|=k} f(\bm{x})+\lambda \cdot div(\bm{x})$. That is, $f(\bm{x}^*)+\lambda \cdot div(\bm{x}^*)=\mathrm{OPT}$. Because the output solution $\hat{\bm{x}}=\arg\max_{\bm{x} \in \{\bm{x}_1, \bm{x}_2\}} f(\bm{x})+\lambda \cdot div(\bm{x})$, we have
\begin{align}
f(\hat{\bm{x}})+\lambda \cdot div(\hat{\bm{x}}) &\geq \frac{f(\bm{x}_1)+\lambda \cdot div(\bm{x}_1)}{2}+\frac{f(\bm{x}_2)+\lambda \cdot div(\bm{x}_2)}{2}\\
& \geq \frac{f(\bm{x}_1)}{2}+\frac{\lambda \cdot div(\bm{x}_2)}{2}\\
&\geq \left(1-\frac{1}{e}\right)\frac{f(\bm{x}_1^*)}{2}+\frac{\lambda \cdot div(\bm{x}^*_2)}{4}\\
&\geq \left(1-\frac{1}{e}\right)\frac{f(\bm{x}^*)}{2}+\frac{\lambda \cdot div(\bm{x}^*)}{4}\\
&\geq \frac{f(\bm{x}^*)+ \lambda \cdot div(\bm{x}^*)}{4}=\frac{\mathrm{OPT}}{4},
\end{align}
where the second inequality holds by the non-negativity of $f$ and $div$, the third inequality holds by Eqs.~(\refeq{eq-mid-21}) and~(\refeq{eq-mid-22}), and the fourth inequality holds because $f(\bm{x}_1^*)=\max_{\bm{x}: |\bm{x}|=k} f(\bm{x})$, $div(\bm{x}_2^*)=\max_{\bm{x}: |\bm{x}|=k} div(\bm{x})$ and $|\bm{x}^*|=k$.
\end{proof}

The above analysis has shown that the GSEMO can achieve the best-known polynomial-time approximation ratio of $1/4$ by optimizing the two bi-objective problems in Eqs.~(\refeq{eq-bi-problem-2-1}) and~(\refeq{eq-bi-problem-2-2}). Next, we will show that if optimizing the bi-objective problem in Eq.~(\refeq{eq-bi-problem-1-v}), i.e., maximizing the original objective function $f+\lambda \cdot div$ and minimizing the subset size simultaneously, the GSEMO fails to guarantee an approximation ratio of $1/4$ in polynomial time. The idea is to construct a specific example, where the GSEMO has some probability to get trapped in a local optimum, which requires to be flipped by many bits simultaneously in mutation for making improvement, and thus leads to exponential running time.

\begin{theorem}\label{theo-cardinality-min-v}
For the result diversification problem with a cardinality constraint in Definition~\ref{def-prob-cardinality}, where the diversity measure $div$ is the min-diversity in Definition~\ref{def-min-div}, there exists one example where the expected number of iterations of the GSEMO using Eq.~(\refeq{eq-bi-problem-1-v}) for achieving an approximation ratio of at least $1/4$ is at least exponential.
\end{theorem}
\begin{proof}
We construct an example as follows. The quality function $f$ is set as a modular function, i.e., $\forall \bm{x} \in \{0,1\}^n: f(\bm{x})=\sum_{v \in \bm{x}} f(\{v\})$. The ground set $V=V_1 \cup V_2$, where $V_1 \cap V_2 =\emptyset$ and $|V_1|=|V_2|=n/2$. For each $v \in V_1$, we set $f(\{v\})=1$, and for each $v \in V_2$, $f(\{v\})=0$. The tradeoff parameter $\lambda$ between the quality and diversity is set to $1$. For a specific item $u^* \in V_1$, its distance from any other item is $n/9$. That is, $\forall v \in V\setminus\{u^*\}: d(u^*,v)=n/9$. Otherwise, the distance between two items is always 1. It is easy to verify that such a distance is a metric. The budget $k$ is set to $n/2$.

By the definition $f(\bm{x})+\lambda \cdot div(\bm{x})$ of the objective function where $div(\bm{x})=\min_{\{u,v\}: u,v \in \bm{x}} d(u,v)$, we have for any $\bm{x}$ with $|\bm{x}| \geq 2$, if $u^* \in \bm{x}$ and $|\bm{x}|=2$, $f(\bm{x})+\lambda \cdot div(\bm{x})=|\bm{x}\cap V_1|+n/9$; otherwise, $f(\bm{x})+\lambda \cdot div(\bm{x})=|\bm{x}\cap V_1|+1$. It can be seen that the optimal function value is $n/2+1$.

The GSEMO starts from the initial solution $\bm{0}$. We consider such an event in the first iteration: only two 0 bits are flipped, where one 0 bit corresponds to the specific item $u^* \in V_1$ and the other corresponds to any other item in $V_1$. This event occurs with probability $(1/n) \cdot ((n/2-1)/n) \cdot (1-1/n)^{n-2} \geq (n/2-1)/(en^2)$, where $1/n$ is the probability of flipping the specific 0 bit corresponding to $u^*$, $(n/2-1)/n$ is the probability of flipping any one of the $(n/2-1)$ 0 bits corresponding to $V_1\setminus\{u^*\}$, and $(1-1/n)^{n-2}$ is the probability of keeping the remaining $n-2$ bits unchanged. The size of the generated solution is $2$ and its objective value is $2+n/9$. By the objective calculation before, we know that this solution is better than any other solution with size smaller than $2+n/9$. Thus, to generate a solution with the objective value larger than $2+n/9$, it is necessary to flip at least $n/9$ bits simultaneously when mutating a solution in line~5 of the GSEMO. For convenience of analysis, assume that $n/9$ is an integer. As the probability of flipping at least $n/9$ bits simultaneously in mutation is at most $\binom{n}{n/9}/ n^{n/9} \leq 1/(n/9)!$, the expected number of iterations of the GSEMO until finding a solution with the objective value larger than $2+n/9$ is at least $((n/2-1)/(en^2)) \cdot (n/9)!$. Because the optimal function value is $n/2+1$, the objective value $2+n/9$ implies an approximation ratio of $(2+n/9)/(n/2+1)<1/4$, where $n$ is sufficiently large. Thus, the theorem holds.
\end{proof}

\subsection{MST-Diversity}

As shown in Definition~\ref{def-mst-div}, the mst-diversity $div$ of a solution $\bm{x}$ is the weight, denoted as $mst(\bm{x})$, of the minimum spanning tree of $\bm{x}$ on the graph $G=(V,E)$, where each vertex is an item, and the weight of each edge is the distance between two items. That is, $div(\bm{x})=mst(\bm{x})$. For the result diversification problem where $div$ is the mst-diversity, the cardinality constraint in Eq.~(\refeq{eq-problem-cardinality}) is strictly $|X|=k$ as in~\cite{dasgupta2013summarization}, because the mst-diversity is not necessarily monotone non-decreasing. To compute $div(\bm{x})$, i.e., $mst(\bm{x})$, we use
\begin{align}\label{eq-mid-23}
\sum^{|\bm{x}|}_{i=2} \min_{v \in \{u_1,\ldots,u_{i-1}\}} d(u_i,v)
\end{align} for approximation, where $(u_1,\ldots,u_{|\bm{x}|})$ is a permutation of the items in $\bm{x}$. It can be seen that different permutations will lead to different approximate values. Lemma~\ref{lemma:cardinality-mst-2} shows that this approximation is larger than the true value by at most a factor of $\log |\bm{x}|$.
\begin{lemma}[Lemma~2 in~\cite{dasgupta2013summarization}]\label{lemma:cardinality-mst-2}
Let $(u_1,\ldots,u_m)$ be a sequence of items. It holds
\begin{align}
\sum^{m}_{i=2} \min_{v \in \{u_1,\ldots,u_{i-1}\}} d(u_i,v) \leq (\log m) \cdot mst(\{u_1,\ldots,u_m\}),
\end{align}
where $mst(\{u_1,\ldots,u_m\})$ denotes the weight of the minimum spanning tree of $\{u_1,\ldots,u_m\}$.
\end{lemma}

Let $\mathcal{P}(\bm{x})$ denote the set of all permutations of the items in $\bm{x}$. To employ the GSEMO, we reformulate the original problem
\begin{align}\label{eq-problem-cardinality-strict-1}\arg\max\nolimits_{X \subseteq V} f(X)+\lambda \cdot div(X) \quad s.t. \quad |X| = k
\end{align} as a bi-objective maximization problem,
\begin{align}\label{eq-bi-problem-3}
&\arg\max\nolimits_{\bm{x} \in \{0,1\}^n, (u_1,\ldots,u_{|\bm{x}|}) \in \mathcal{P}(\bm{x})} \quad  (f_1(\bm{x},(u_1,\ldots,u_{|\bm{x}|})),f_2(\bm{x},(u_1,\ldots,u_{|\bm{x}|}))),\\
&\text{where}\quad\begin{cases}
f_1(\bm{x},(u_1,\ldots,u_{|\bm{x}|})) = f(\bm{x})+\lambda \cdot \sum^{|\bm{x}|}_{i=2} \min_{v \in \{u_1,\ldots,u_{i-1}\}} d(u_i,v),\\
f_2(\bm{x},(u_1,\ldots,u_{|\bm{x}|})) = -|\bm{x}|.
\end{cases}
\end{align}
Note that each solution $\bm{x} \in \{0,1\}^n$ now has a corresponding permutation $(u_1,\ldots,u_{|\bm{x}|})$, and the same solution with different permutations will have different $f_1$ values.

The GSEMO in Algorithm~\ref{algo:GSEMO} starts from the empty subset $\bm{0}$. When generating a new offspring solution $\bm{x}'$ by mutating a parent solution $\bm{x}$ in line~4 of Algorithm~\ref{algo:GSEMO}, the corresponding permutation of $\bm{x}'$ is decided as follows: the newly added items in $\bm{x}' \setminus \bm{x}$ are put after the old items in $\bm{x}' \cap \bm{x}$; the ordering of the old items in $\bm{x}' \cap \bm{x}$ is kept as same as that in the permutation of $\bm{x}$; the new items in $\bm{x}' \setminus \bm{x}$ are arbitrarily ordered. Note that all the solutions with size larger than $k$ are excluded during the running of the GSEMO. After the GSEMO terminates, the solution $\bm{x}$ with the largest size in the population is output. If $|\bm{x}|<k$, $k-|\bm{x}|$ items will be arbitrarily selected from the remaining items and added into $\bm{x}$, in order to satisfy the constraint $|\bm{x}|=k$. Theorem~\ref{theo-cardinality-mst} shows that the GSEMO using at most $enk(k+1)$ expected number of iterations can achieve an approximation ratio of $(1-1/e)(1/(2\log k))\approx 0.316/\log k$, slightly worse than the best-known polynomial-time approximation ratio $1/(3\log k)$~\cite{dasgupta2013summarization}.

\begin{theorem}\label{theo-cardinality-mst}
For the result diversification problem with a cardinality constraint in Definition~\ref{def-prob-cardinality}, where the diversity measure $div$ is the mst-diversity in Definition~\ref{def-mst-div}, the expected number of iterations of the GSEMO using Eq.~(\refeq{eq-bi-problem-3}), until finding a solution $\bm{x}$ with $|\bm{x}| = k$ and $f(\bm{x})+\lambda \cdot div(\bm{x}) \geq (1-1/e)(1/(2\log k)) \cdot \mathrm{OPT}$, is at most $enk(k+1)$.
\end{theorem}

The proof of Theorem~\ref{theo-cardinality-mst} relies on Lemma~\ref{lemma:cardinality-mst-3}, which shows that for any $\bm{x} \in \{0,1\}^n$ with $|\bm{x}|<k$, there always exists another item whose inclusion can bring an improvement on $f_1$ roughly proportional to the current distance to half of the optimum. The proof of Lemma~\ref{lemma:cardinality-mst-3} is inspired by that of Theorem~3 in~\cite{dasgupta2013summarization}. Lemma~\ref{lemma:cardinality-mst-1} gives a lower bound on the sum of the distances of items in $\bm{x}^*\setminus \bm{x}$ to $\bm{x}$, which will be used in the proof of Lemma~\ref{lemma:cardinality-mst-3}.

\begin{lemma}[Lemma~1(ii) in~\cite{dasgupta2013summarization}]\label{lemma:cardinality-mst-1}
For any $\bm{x} \in \{0,1\}^n$ with $|\bm{x}| < k$, it holds
\begin{align}
\sum_{u \in \bm{x}^* \setminus \bm{x}} \min_{v \in \bm{x}} d(u,v) \geq \frac{mst(\bm{x}^*)}{2}-mst(\bm{x}),
\end{align}
where $\bm{x}^*$ denotes an optimal solution of Eq.~(\refeq{eq-problem-cardinality-strict-1}), i.e., $f(\bm{x}^*) + \lambda \cdot mst(\bm{x}^*)=\mathrm{OPT}$.
\end{lemma}

\begin{lemma}\label{lemma:cardinality-mst-3}
For any $\bm{x} \in \{0,1\}^n$ with $|\bm{x}| < k$, and an arbitrary permutation $(u_1,\ldots,u_{|\bm{x}|})$ of the items in $\bm{x}$, there exists one item $u \notin \bm{x}$ such that
\begin{align}
& f_1(\bm{x} \cup \{u\},(u_1,\ldots,u_{|\bm{x}|},u))- f_1(\bm{x},(u_1,\ldots,u_{|\bm{x}|})) \geq \frac{1}{k}\cdot \left(\frac{\mathrm{OPT}}{2}-f_1(\bm{x},(u_1,\ldots,u_{|\bm{x}|}))\right),
\end{align}
where $f_1$ is defined in Eq.~(\refeq{eq-bi-problem-3}).
\end{lemma}
\begin{proof}
Let $\bm{x}^*$ denote an optimal solution. By the definition of $f_1$ in Eq.~(\refeq{eq-bi-problem-3}), we have
\begin{align}\label{eq-mid-24}
&\sum_{u \in \bm{x}^*\setminus \bm{x}} f_1(\bm{x} \cup \{u\},(u_1,\ldots,u_{|\bm{x}|},u))- f_1(\bm{x},(u_1,\ldots,u_{|\bm{x}|}))\\
&=\sum_{u \in \bm{x}^*\setminus \bm{x}} f(\bm{x} \cup \{u\})-f(\bm{x})+\lambda \cdot \min_{v \in \bm{x}} d(u,v)\\
&\geq f(\bm{x} \cup \bm{x}^*)-f(\bm{x})+ \lambda \cdot \left(\frac{mst(\bm{x}^*)}{2}-mst(\bm{x})\right)\\
& \geq f(\bm{x}^*)+\frac{\lambda}{2}\cdot mst(\bm{x}^*) - f(\bm{x}) - \lambda \cdot mst(\bm{x})\\
& \geq \frac{\mathrm{OPT}}{2} - f(\bm{x}) - \lambda \cdot mst(\bm{x}),
\end{align}
where the first inequality holds by the submodularity (i.e., Eq.~(\refeq{def-submodular-2})) of $f$ and Lemma~\ref{lemma:cardinality-mst-1}, and the second inequality holds by the monotonicity of $f$. By connecting $u_i$ to the nearest item in $\{u_1,\ldots,u_{i-1}\}$ for each $i \in \{2,\ldots,|\bm{x}|\}$, we can construct a spanning tree of $\bm{x}$, whose weight is
\begin{align}
\sum^{|\bm{x}|}_{i=2} \min_{v \in \{u_1,\ldots,u_{i-1}\}} d(u_i,v).
\end{align}
As $mst(\bm{x})$ is the weight of the minimum spanning tree of $\bm{x}$, it holds
\begin{align}
\sum^{|\bm{x}|}_{i=2} \min_{v \in \{u_1,\ldots,u_{i-1}\}} d(u_i,v) \geq mst(\bm{x}).
\end{align}
By applying the above inequality to Eq.~(\refeq{eq-mid-24}), we have
\begin{align}
&\sum_{u \in \bm{x}^*\setminus \bm{x}} f_1(\bm{x} \cup \{u\},(u_1,\ldots,u_{|\bm{x}|},u))- f_1(\bm{x},(u_1,\ldots,u_{|\bm{x}|}))\\
&\geq \frac{\mathrm{OPT}}{2} - f(\bm{x}) - \lambda \cdot \sum^{|\bm{x}|}_{i=2} \min_{v \in \{u_1,\ldots,u_{i-1}\}} d(u_i,v)\\
&=\frac{\mathrm{OPT}}{2} - f_1(\bm{x},(u_1,\ldots,u_{|\bm{x}|})),
\end{align}
where the equality can be derived from the definition of $f_1$ in Eq.~(\refeq{eq-bi-problem-3}). Let $u^* = \arg\max_{u \in \bm{x}^*\setminus \bm{x}} f_1(\bm{x} \cup \{u\},(u_1,\ldots,u_{|\bm{x}|},u))- f_1(\bm{x},(u_1,\ldots,u_{|\bm{x}|}))$. Because $|\bm{x}^*|=k$, we have
\begin{align}
&f_1(\bm{x} \cup \{u^*\},(u_1,\ldots,u_{|\bm{x}|},u^*))- f_1(\bm{x},(u_1,\ldots,u_{|\bm{x}|}))\geq \frac{1}{k}\left(\frac{\mathrm{OPT}}{2} - f_1(\bm{x},(u_1,\ldots,u_{|\bm{x}|}))\right),
\end{align}
implying that the lemma holds.
\end{proof}

\begin{myproof}{Theorem~\ref{theo-cardinality-mst}}
We consider a quantity $J_{\max}$, which is defined as
$$
J_{\max}=\max\left\{j \in \{0,1,\ldots,k\} \mid \exists \bm{x} \in P: |\bm{x}| \leq j \wedge f_1(\bm{x},(u_1,\ldots,u_{|\bm{x}|})) \geq \left(1-\left(1-\frac{1}{k}\right)^j\right) \cdot \frac{\mathrm{OPT}}{2}\right\}.
$$
Note that each solution $\bm{x}$ in the population $P$ has a corresponding permutation $(u_1,\ldots,u_{|\bm{x}|})$ of its items. $J_{\max}=k$ implies that there exists one solution $\bm{x}$ in $P$ satisfying that $|\bm{x}| \leq k$ and
\begin{align}f_1(\bm{x},(u_1,\ldots,u_{|\bm{x}|})) \geq \left(1-\left(1-\frac{1}{k}\right)^k\right) \cdot \frac{\mathrm{OPT}}{2} \geq \left(1-\frac{1}{e}\right)\frac{1}{2} \cdot \mathrm{OPT}.\end{align}
Because the output solution $\hat{\bm{x}}$ is the solution with the largest size in $P$, and the solutions in $P$ are incomparable, $\hat{\bm{x}}$ must have the largest $f_1$ value. Let $(v_1,\ldots,v_{|\hat{\bm{x}}|})$ denote the corresponding permutation of $\hat{\bm{x}}$. Then, we have
\begin{align}\label{eq-mid-25}f_1(\hat{\bm{x}},(v_1,\ldots,v_{|\hat{\bm{x}}|})) \geq \left(1-\frac{1}{e}\right)\frac{1}{2} \cdot \mathrm{OPT}.\end{align}
Note that the solutions with size larger than $k$ are excluded during the optimization. If $|\hat{\bm{x}}|<k$, $\hat{\bm{x}}$ will be repaired by adding arbitrarily $k-|\hat{\bm{x}}|$ unselected items, and the resulting solution is denoted as $\hat{\bm{x}}'$. Let $(v_1,\ldots,v_{k})$ denote a permutation of the items in $\hat{\bm{x}}'$, generated by appending the $k-|\hat{\bm{x}}|$ supplementary items to the end of the sequence $(v_1,\ldots,v_{|\hat{\bm{x}}|})$. According to the definition of $f_1$ in Eq.~(\refeq{eq-bi-problem-3}), we have
\begin{align}\label{eq-mid-26}
f_1(\hat{\bm{x}}',(v_1,\ldots,v_{k})) &= f(\hat{\bm{x}}')+\lambda \cdot \sum^{k}_{i=2} \min_{v \in \{v_1,\ldots,v_{i-1}\}} d(v_i,v)\\
&\geq f(\hat{\bm{x}})+\lambda \cdot \sum^{|\hat{\bm{x}}|}_{i=2} \min_{v \in \{v_1,\ldots,v_{i-1}\}} d(v_i,v)\\
&=f_1(\hat{\bm{x}},(v_1,\ldots,v_{|\hat{\bm{x}}|}))\\
&\geq \left(1-\frac{1}{e}\right)\frac{1}{2} \cdot \mathrm{OPT},
\end{align}
where the first inequality holds by the monotonicity of $f$ and $|\hat{\bm{x}}| < k$, and the last inequality holds by Eq.~(\refeq{eq-mid-25}). If $|\hat{\bm{x}}|=k$, $\hat{\bm{x}}$ keeps unchanged, i.e., $\hat{\bm{x}}'=\hat{\bm{x}}$. Thus, Eq.~(\refeq{eq-mid-25}) directly implies
\begin{align}\label{eq-mid-27}
f_1(\hat{\bm{x}}',(v_1,\ldots,v_{k})) \geq \left(1-\frac{1}{e}\right)\frac{1}{2} \cdot \mathrm{OPT}.
\end{align}
By Lemma~\ref{lemma:cardinality-mst-2}, we have
\begin{align}
f(\hat{\bm{x}}')+\lambda \cdot mst(\hat{\bm{x}}') &\geq f(\hat{\bm{x}}')+\lambda \cdot \frac{1}{\log k} \sum^{k}_{i=2} \min_{v \in \{v_1,\ldots,v_{i-1}\}} d(v_i,v)\\
&\geq f_1(\hat{\bm{x}}',(v_1,\ldots,v_{k}))/\log k\\
&\geq \left(1-\frac{1}{e}\right)\frac{1}{2\log k} \cdot \mathrm{OPT},
\end{align}
where the second inequality holds by the definition of $f_1$, and the last inequality holds by Eqs.~(\refeq{eq-mid-26}) and~(\refeq{eq-mid-27}). That is, the desired approximation guarantee is reached. Note that $div(\hat{\bm{x}}')=mst(\hat{\bm{x}}')$ here. Thus, we only need to analyze the expected number of iterations until $J_{\max}=k$.

The value of $J_{\max}$ is initially 0, since the population $P$ starts from the solution $\bm{0}$, whose values on both objectives are 0. Assume that currently $J_{\max}=i <k$. Let $\bm{x}$ be a corresponding solution with the value $i$, i.e., $|\bm{x}|\leq i$ and \begin{align}\label{eq-mid-28}f_1(\bm{x},(u_1,\ldots,u_{|\bm{x}|})) \geq \left(1-\left(1-\frac{1}{k}\right)^i\right) \cdot \frac{\mathrm{OPT}}{2}.\end{align} It is easy to see that $J_{\max}$ cannot decrease because deleting $\bm{x}$ from the population $P$ (line~6 of Algorithm~\ref{algo:GSEMO}) implies that $\bm{x}$ is weakly dominated by a newly generated solution $\bm{x}'$, which must satisfy that $|\bm{x}'| \leq |\bm{x}|$ and $f_1(\bm{x}',(u'_1,\ldots,u'_{|\bm{x}'|})) \geq f_1(\bm{x},(u_1,\ldots,u_{|\bm{x}|}))$. By Lemma~\ref{lemma:cardinality-mst-3}, we know that adding a specific item $u$ into $\bm{x}$ can generate a new solution $\bm{x}'=\bm{x} \cup \{u\}$ such that
\begin{align}\label{eq-mid-29}
& f_1(\bm{x}',(u_1,\ldots,u_{|\bm{x}|},u))- f_1(\bm{x},(u_1,\ldots,u_{|\bm{x}|})) \geq \frac{1}{k}\cdot \left(\frac{\mathrm{OPT}}{2}-f_1(\bm{x},(u_1,\ldots,u_{|\bm{x}|}))\right).
\end{align}
This can be accomplished by selecting $\bm{x}$ in line~3 of Algorithm~\ref{algo:GSEMO} (occurring with probability $1/|P| \geq 1/(k+1)$) and flipping only a specific 0-bit (corresponding to the item $u$) of $\bm{x}$ in line~4 (occurring with probability $(1/n)(1-1/n)^{n-1} \geq 1/(en)$). Note that the corresponding permutation of $\bm{x}'$ is constructed by keeping the ordering of the old items in $\bm{x}' \cap \bm{x}=\bm{x}$ and putting the newly added item $u$ after these old items, and thus is just $(u_1,\ldots,u_{|\bm{x}|},u)$. Combining Eqs.~(\refeq{eq-mid-28}) and~(\refeq{eq-mid-29}) leads to
\begin{align}
f_1(\bm{x}',(u_1,\ldots,u_{|\bm{x}|},u)) \geq \left(1-\left(1-\frac{1}{k}\right)^{i+1}\right) \cdot \frac{\mathrm{OPT}}{2}.
\end{align}
Since $|\bm{x}'|=|\bm{x}|+1 \leq i+1$, $\bm{x}'$ will be included into $P$; otherwise, $\bm{x}'$ must be dominated by one solution in $P$ (line~5 of Algorithm~\ref{algo:GSEMO}), which implies that $J_{\max}$ has already been larger than $i$, contradicting the assumption $J_{\max}=i$. After including $\bm{x}'$, $J_{\max} \geq i+1$. Thus, $J_{\max}$ can increase by at least 1 in one iteration with probability at least $(1/(k+1)) \cdot (1/(en)) = 1/(en(k+1))$, implying that it needs at most $en(k+1)$ expected number of iterations to increase $J_{\max}$. To make $J_{\max}=k$, it is sufficient to increase $J_{\max}$ for $k$ times, and thus the total required number of iterations is at most $k \cdot en(k+1)$ in expectation, implying the theorem holds.
\end{myproof}

\section{MOEAs for Result Diversification with a Matroid Constraint}\label{sec-matroid}

Next, we consider the result diversification problem with a matroid constraint in Definition~\ref{def-prob-matroid}. The diversity measure $div$ is set as the sum-diversity in Definition~\ref{def-sum-div}, which is the only previously studied measure for this problem. To employ the GSEMO, the original problem
\begin{align}
\arg\max\nolimits_{X \subseteq V} f(X)+\lambda \cdot div(X) \quad s.t. \quad X \in \mathcal{F},
\end{align}
i.e., Eq.~(\refeq{eq-problem-matroid}) in Definition~\ref{def-prob-matroid}, is transformed into a bi-objective maximization problem
\begin{align}\label{eq-bi-problem-4}
&\arg\max\nolimits_{\bm{x} \in \{0,1\}^n} \quad  (f_1(\bm{x}),f_2(\bm{x})),\\
&\text{where}\quad\begin{cases}
f_1(\bm{x}) = f(\bm{x})+\lambda \cdot div(\bm{x}),\\
f_2(\bm x) = |\bm{x}|.
\end{cases}
\end{align}
That is, the GSEMO is to maximize the original objective function $f(\bm{x})+\lambda \cdot div(\bm{x})$ and the subset size $|\bm{x}|$ simultaneously. Note that the infeasible solutions, i.e., the subsets not belonging to $\mathcal{F}$, are excluded during the optimization process of the GSEMO.

Let $r$ denote the rank of the matroid $(V,\mathcal{F})$. Theorem~\ref{theo-matroid-sum} shows that the GSEMO can achieve a $(1/2-\epsilon/(4n))$-approximation ratio after running at most $O(rn^3(n+(r\log r)/\epsilon))$ expected number of iterations. Note that this approximation ratio has been asymptotically optimal, because $1/2$ is the optimal polynomial-time approximation ratio even for the special case of cardinality constraints~\cite{borodin2017max}.

\begin{theorem}\label{theo-matroid-sum}
For the result diversification problem with a matroid constraint in Definition~\ref{def-prob-matroid}, where the diversity measure $div$ is the sum-diversity in Definition~\ref{def-sum-div}, the expected number of iterations of the GSEMO using Eq.~(\refeq{eq-bi-problem-4}), until finding a solution $\bm{x}$ with $\bm{x} \in \mathcal{F}$ and $f(\bm{x})+\lambda \cdot div(\bm{x}) \geq (1/2-\epsilon/(4n)) \cdot \mathrm{OPT}$, is at most $O(rn^3(n+(r\log r)/\epsilon))$, where $\epsilon >0$.
\end{theorem}

Lemma~\ref{lemma:matroid-property} shows that for any two bases $\bm{x}, \bm{y}$ of $V$, there exists a pairing between items in $\bm{x} \setminus \bm{y}$ and $\bm{y} \setminus \bm{x}$ such that deleting an item in $\bm{x} \setminus \bm{y}$ from $\bm{x}$ and adding its paired counterpart in $\bm{y} \setminus \bm{x}$ will still lead to a basis.

\begin{lemma}[Corollary~3 in~\cite{brualdi1969comments}]\label{lemma:matroid-property}
Let $(V,\mathcal{F})$ denote a matroid. For any two bases $\bm{x}, \bm{y}$ of $V$, there is a bijective mapping $\phi: \bm{x} \setminus \bm{y} \rightarrow  \bm{y} \setminus \bm{x}$ such that $\bm{x} \setminus \{u\} \cup \{\phi(u)\} \in \mathcal{F}$ for any $u \in \bm{x} \setminus \bm{y}$.
\end{lemma}

Let $\bm{x}^*$ denote an optimal solution of Eq.~(\refeq{eq-problem-matroid}), i.e., $f(\bm{x}^*)+\lambda \cdot div(\bm{x}^*)=\mathrm{OPT}$. Because the objective function $f+\lambda \cdot div$ is monotone non-decreasing, $\bm{x}^*$ must be a basis of $V$. Note that the size of a basis of $V$ is just the rank $r$ of the matroid, that is, a basis of $V$ contains $r$ items. Let $\bm{x}$ be any basis of $V$. We use $\{u_1,\ldots,u_{|\bm{x} \setminus \bm{x}^*|}\}$ to denote $\bm{x} \setminus \bm{x}^*$. Based on Lemma~\ref{lemma:matroid-property}, we assume without loss of generality that $\bm{x}^* \setminus \bm{x}=\{v_1,\ldots,v_{|\bm{x}^* \setminus \bm{x}|}\}$, where $\forall i: \phi(u_i)=v_i$. For those bases generated by exchanging the two items $u_i$ and $v_i$, the following two lemmas give lower bounds on the sum of their $f$ and $div$ values, respectively.

\begin{lemma}[Lemma~5.5 in~\cite{borodin2017max}]\label{lemma:matroid-mid-1}
For any basis $\bm{x}$ of $V$, it holds
\begin{align}
\sum^{|\bm{x} \setminus \bm{x}^*|}_{i=1} f(\bm{x} \setminus \{u_i\} \cup \{v_i\}) \geq (|\bm{x} \setminus \bm{x}^*|-2) \cdot f(\bm{x})+f(\bm{x}^*).
\end{align}
\end{lemma}

\begin{lemma}[Lemma~5.7 in~\cite{borodin2017max}]\label{lemma:matroid-mid-2}
For any basis $\bm{x}$ of $V$, it holds
\begin{align}\label{eq-mid-30}
\sum^{|\bm{x} \setminus \bm{x}^*|}_{i=1} div(\bm{x} \setminus \{u_i\} \cup \{v_i\}) \geq (|\bm{x} \setminus \bm{x}^*|-2) \cdot div(\bm{x})+div(\bm{x}^*).
\end{align}
\end{lemma}

The proof of Theorem~\ref{theo-matroid-sum} relies on Lemma~\ref{lemma:matroid-mid-3}, which shows that it is always possible to improve a basis of $V$ by swapping two items until a good approximation has been achieved. This lemma is inspired from Theorem~5.1 in~\cite{borodin2017max}, and the proof will use Lemmas~\ref{lemma:matroid-mid-1} and~\ref{lemma:matroid-mid-2}.

\begin{lemma}\label{lemma:matroid-mid-3}
Let $\bm{x}$ be a basis of $V$ such that no basis $\bm{x}'$ with the objective value $f(\bm{x}')+\lambda \cdot div(\bm{x}')>(1+\epsilon/(rn))\cdot (f(\bm{x})+\lambda \cdot div(\bm{x}))$ can be achieved by deleting one item inside $\bm{x}$ and inserting one item outside into $\bm{x}$, where $\epsilon >0$. Then it holds \begin{align}\label{eq-mid-33}f(\bm{x}) +\lambda \cdot div(\bm{x}) \geq \left(\frac{1}{2}-\frac{\epsilon}{4n}\right)\cdot \mathrm{OPT}.\end{align}
\end{lemma}
\begin{proof}
Based on Lemma~\ref{lemma:matroid-property}, let $\bm{x} \setminus \bm{x}^*=\{u_1,\ldots,u_{|\bm{x} \setminus \bm{x}^*|}\}$, and $\bm{x}^* \setminus \bm{x}=\{v_1,\ldots,v_{|\bm{x}^* \setminus \bm{x}|}\}$, where $\forall i: \phi(u_i)=v_i$. Note that $|\bm{x} \setminus \bm{x}^*|=|\bm{x}^* \setminus \bm{x}|$. Considering those bases of $V$, generated by deleting $u_i$ from $\bm{x}$ and adding $v_i$ into $\bm{x}$, we have for any $i$,
\begin{align}
\left(1+\frac{\epsilon}{rn}\right)\cdot (f(\bm{x}) +\lambda \cdot div(\bm{x})) \geq  f(\bm{x} \setminus \{u_i\} \cup \{v_i\}) +\lambda \cdot div(\bm{x} \setminus \{u_i\} \cup \{v_i\}).
\end{align}
Summing up the above inequalities over $i$ leads to
\begin{align}\label{eq-mid-31}
&|\bm{x} \setminus \bm{x}^*|\cdot \left(1+\frac{\epsilon}{rn}\right)\cdot (f(\bm{x}) +\lambda \cdot div(\bm{x}))\\
& \geq  \sum^{|\bm{x} \setminus \bm{x}^*|}_{i=1} f(\bm{x} \setminus \{u_i\} \cup \{v_i\}) +\lambda \cdot div(\bm{x} \setminus \{u_i\} \cup \{v_i\})\\
&\geq(|\bm{x} \setminus \bm{x}^*|-2) \cdot f(\bm{x})+f(\bm{x}^*) + \lambda \cdot ((|\bm{x} \setminus \bm{x}^*|-2) \cdot div(\bm{x})+div(\bm{x}^*))\\
&= (|\bm{x} \setminus \bm{x}^*|-2) \cdot (f(\bm{x})+\lambda \cdot div(\bm{x}))+\mathrm{OPT},
\end{align}
where the second inequality holds by Lemmas~\ref{lemma:matroid-mid-1} and~\ref{lemma:matroid-mid-2}. By moving $(|\bm{x} \setminus \bm{x}^*|-2) \cdot (f(\bm{x})+\lambda \cdot div(\bm{x}))$ to the left side, we have
\begin{align}
&\left(2+|\bm{x} \setminus \bm{x}^*|\cdot \frac{\epsilon}{rn}\right)\cdot (f(\bm{x}) +\lambda \cdot div(\bm{x}))\geq \mathrm{OPT}.
\end{align}
Because $|\bm{x} \setminus \bm{x}^*| \leq r$, we have
\begin{align}
&\left(2+\frac{\epsilon}{n}\right)\cdot (f(\bm{x}) +\lambda \cdot div(\bm{x}))\geq \mathrm{OPT},
\end{align}
implying
\begin{align}\label{eq-mid-32}
&f(\bm{x}) +\lambda \cdot div(\bm{x}) \geq \frac{1}{2+\frac{\epsilon}{n}} \cdot \mathrm{OPT} = \left(\frac{1}{2}-\frac{\frac{\epsilon}{2n}}{2+\frac{\epsilon}{n}}\right)\cdot \mathrm{OPT} \geq \left(\frac{1}{2}-\frac{\epsilon}{4n}\right) \cdot \mathrm{OPT}.
\end{align}
Thus, the lemma holds.
\end{proof}

\begin{myproof}{Theorem~\ref{theo-matroid-sum}}
We divide the optimization process into three phases: (1) starts from the initial solution $\bm{0}$ (i.e., $\emptyset$) and finishes after finding a basis of $V$; (2) starts after phase~(1) and finishes after finding a basis of $V$ with the objective value at least $(2/(r(r-1)))\cdot \mathrm{OPT}$; (3) starts after phase~(2) and finishes after finding a basis with the desired approximation ratio $1/2-\epsilon/(4n)$. We analyze the expected number of iterations of each phase, respectively, and then sum them up to get an upper bound on the total expected number of iterations of the GSEMO.

For phase (1), we consider the maximum number of 1-bits of the solutions in the population $P$, denoted by $J_{\max}$. That is, $J_{\max}=\max\{|\bm{x}| \mid \bm{x} \in P\}$. Note that the infeasible solutions, i.e., the solutions not belonging to $\mathcal{F}$, are excluded during the optimization. Thus, $J_{\max}$ is at most the rank $r$ of the matroid, and $J_{\max}=r$ implies that a basis of $V$ has been found. $J_{\max}$ is initially 0. Assume that currently $J_{\max}=i < r$, and let $\bm{x}$ be the corresponding solution, i.e., $|\bm{x}|=i$. $J_{\max}$ will not decrease because $\bm{x}$ cannot be weakly dominated by a solution with less 1-bits. It can be known from the definition of matroid that there are at least $r-i$ items, adding one of which into $\bm{x}$ can generate a solution $\bm{x}' \in \mathcal{F}$ with $|\bm{x}'|=|\bm{x}|+1=i+1$. Thus, by selecting $\bm{x}$ in line~3 of Algorithm~\ref{algo:GSEMO} and flipping only one of those corresponding 0-bits of $\bm{x}$ (i.e., adding one of those items into $\bm{x}$) in line~4, which occur with probability at least $(1/|P|) \cdot ((r-i)/n)(1-1/n)^{n-1} \geq (r-i)/(en|P|)$, a new solution $\bm{x}' \in \mathcal{F}$ with $|\bm{x}'|=i+1$ can be generated in one iteration of the GSEMO. Note that the size of the solution in $P$ is at most $r$, and thus the second objective $f_2$ in Eq.~(\refeq{eq-bi-problem-4}) can take values $0,1,\ldots,r$, implying $|P| \leq r+1$. In fact, $|P| \leq r$, because the all-0s solution is dominated by any other feasible solution, and will not exist in $P$ once a feasible solution with size larger than 0 has been generated. Thus, the probability of generating a new solution $\bm{x}' \in \mathcal{F}$ with $|\bm{x}'|=i+1$ in one iteration is at least $(r-i)/(ern)$. Because the newly generated solution $\bm{x}' \in \mathcal{F}$ now has the largest number of 1-bits and no solution in $P$ can dominate it, it will be included into $P$, making $J_{\max}=i+1$. This implies that the probability of increasing $J_{\max}$ by 1 in one iteration of the GSEMO is at least $(r-i)/(ern)$. We can then get that the expected number of iterations of phase (1) (i.e., to make $J_{\max}$ reach $r$) is at most
\begin{align}
\sum^{r-1}_{i=0} \frac{ern}{r-i} \leq ern(1+\log r).
\end{align}
Note that the population $P$ will always contain a basis of $V$ once generated, since it has the largest $f_2$ value $r$ and can be weakly dominated by only other bases of $V$.

Let $\hat{\bm{x}}$ denote the basis of $V$ in the population $P$. For phase (2), i.e., finding a basis of $V$ with the objective value at least $(2/(r(r-1)))\cdot \mathrm{OPT}$, we will show that it is sufficient to select $\hat{\bm{x}}$ in line~3 of Algorithm~\ref{algo:GSEMO} and flip only one specific 0-bit and 1-bit, or two specific 0-bits and 1-bits in line~4. Let $\{u^*,v^*\}=\arg \max_{\{u,v\}: u,v \in \bm{x}^*} f(\{u,v\}) + \lambda \cdot d(u,v)$, where $\bm{x}^*$ denotes an optimal solution. We have
\begin{align}
f(\bm{x}^*)+\lambda \cdot div(\bm{x}^*)& \leq \sum_{\{u,v\}: u,v \in \bm{x}^*} f(\{u,v\})+\lambda \cdot d(u,v)\\
&\leq \binom{|\bm{x}^*|}{2}\cdot (f(\{u^*,v^*\})+\lambda \cdot d(u^*,v^*))\\
&=\frac{r(r-1)}{2}\cdot (f(\{u^*,v^*\})+\lambda \cdot d(u^*,v^*)),
\end{align}
where the first inequality holds by the submodularity (i.e., Eq.~(\refeq{def-submodular})) of $f$, the second inequality holds by the definition of $\{u^*,v^*\}$, and the equality holds by $|\bm{x}^*|=r$. Thus,
\begin{align}
f(\{u^*,v^*\})+\lambda \cdot d(u^*,v^*) \geq \frac{2}{r(r-1)} \cdot (f(\bm{x}^*)+\lambda \cdot div(\bm{x}^*)) = \frac{2}{r(r-1)} \cdot \mathrm{OPT}.
\end{align}
According to Lemma~\ref{lemma:matroid-property}, we know that there is a bijective mapping $\phi: \hat{\bm{x}} \setminus \bm{x}^* \rightarrow  \bm{x}^* \setminus \hat{\bm{x}}$ such that $\hat{\bm{x}} \setminus \{u\} \cup \{\phi(u)\} \in \mathcal{F}$ for any $u \in \hat{\bm{x}} \setminus \bm{x}^*$. If $u^*$ or (and) $v^*$ is not in $\hat{\bm{x}}$, we can swap the item $\phi^{-1}(u^*)$ in $\hat{\bm{x}}$ with $u^*$ or (and) swap $\phi^{-1}(v^*)$ with $v^*$ to generate a new basis containing both $u^*$ and $v^*$. As $f$ is monotone, the objective value of this new basis is at least $f(\{u^*,v^*\})+\lambda \cdot d(u^*,v^*) \geq (2/(r(r-1))) \cdot \mathrm{OPT}$. After generating such a new basis, it will be used to update the population $P$, which makes $P$ always contain a basis with the objective value at least $(2/(r(r-1))) \cdot \mathrm{OPT}$, achieving the goal of this phase. Note that swapping one item inside $\hat{\bm{x}}$ with one item outside corresponds to flipping one specific 0-bit and 1-bit of $\hat{\bm{x}}$, and thus our claim holds. Because the probability of selecting $\hat{\bm{x}}$ in line~3 of the GSEMO is $1/|P| \geq 1/r$ and the probability of flipping one specific 0-bit and 1-bit in line~4 is lower bounded by that of flipping two specific 0-bits and 1-bits, i.e., $(1/n)^4(1-1/n)^{n-4}\geq 1/(en^4)$, the expected number of iterations of phase (2) is at most $ern^4$.

We call a basis $\bm{x}$ of $V$ a $(1+\alpha)$-approximate local optimum if \begin{align}f(\bm{x} \setminus \{u\} \cup \{v\}) +\lambda \cdot div(\bm{x} \setminus \{u\} \cup \{v\}) \leq (1+\alpha)\cdot (f(\bm{x})+\lambda \cdot div(\bm{x}))\end{align} for any $u \in \bm{x}$, $v \in V \setminus \bm{x}$ and $\bm{x} \setminus \{u\} \cup \{v\} \in \mathcal{F}$. By Lemma~\ref{lemma:matroid-mid-3}, we know that a $(1+\epsilon/(rn))$-approximate local optimum $\bm{x}$ satisfies \begin{align}f(\bm{x}) +\lambda \cdot div(\bm{x}) \geq \left(\frac{1}{2}-\frac{\epsilon}{4n}\right)\cdot \mathrm{OPT},\end{align}
achieving the desired approximation ratio. For phase~(3), we thus only need to analyze the expected number of iterations until generating a $(1+\epsilon/(rn))$-approximate local optimum. We consider the $f_1$ value of the basis $\hat{\bm{x}}$ of $V$ in the population $P$, where $f_1=f+\lambda \cdot div$ as in Eq.~(\refeq{eq-bi-problem-4}). After phase (2), \begin{align}f_1(\hat{\bm{x}}) \geq \frac{2}{r(r-1)} \cdot \mathrm{OPT}.\end{align} It is obvious that $f_1(\hat{\bm{x}})$ will not decrease, because $\hat{\bm{x}}$ can be dominated by only other bases of $V$ with larger $f_1$ values. As long as $\hat{\bm{x}}$ is not a $(1+\epsilon/(rn))$-approximate local optimum, we know that a new solution $\bm{x}'$ with \begin{align}f_1(\bm{x}')>\left(1+\frac{\epsilon}{rn}\right)\cdot f_1(\hat{\bm{x}})\end{align} can be generated through selecting $\bm{x}$ in line~3 of Algorithm~\ref{algo:GSEMO} and flipping only one specific 0-bit and 1-bit (i.e., deleting one item inside $\hat{\bm{x}}$ and inserting one item outside into $\hat{\bm{x}}$) in line~4, the probability of which is $(1/|P|)\cdot (1/n^2)(1-1/n)^{n-2}\geq 1/(ern^2)$. Since $\bm{x}'$ now has the largest $f_1$ value and no other solution in $P$ can dominate it, it will be included into $P$, and $\hat{\bm{x}}$ now becomes $\bm{x}'$. Thus, $f_1(\hat{\bm{x}})$ can increase by at least a factor of $(1+\epsilon/(rn))$ with probability at least $1/(ern^2)$ in each iteration. Such an increase on $f_1(\hat{\bm{x}})$ is called a successful step. Thus, a successful step needs at most $ern^2$ expected number of iterations. It is also easy to see that until generating a $(1+\epsilon/(rn))$-approximate local optimum, the number of successful steps is at most \begin{align}\label{eq-mid-34}\log_{1+\frac{\epsilon}{rn}} \frac{\mathrm{OPT}}{\frac{2}{r(r-1)} \cdot\mathrm{OPT}} = O\left(\frac{rn}{\epsilon}\log r\right).\end{align} Thus, the expected number of iterations of phase (3) is at most
\begin{align}
ern^2 \cdot O\left(\frac{rn}{\epsilon}\log r\right) =  O\left(\frac{r^2n^3}{\epsilon}\log r\right).
\end{align}

By combining the expected number of iterations in the above three phases, we can conclude that the GSEMO requires at most
\begin{align}
ern(1+\log r)+ern^4+O\left(\frac{r^2n^3}{\epsilon}\log r\right)=O\left(rn^3\left(n+\frac{r\log r}{\epsilon}\right)\right)
\end{align}
iterations in expectation to achieve an approximation ratio of $1/2-\epsilon/(4n)$.\vspace{0.5em}
\end{myproof}

We also consider the more general case where the metric distance function $d$ is relaxed to satisfy Eq.~(\refeq{eq-triangle-relax}). Due to this relaxation, Eq.~(\refeq{eq-mid-30}) in Lemma~\ref{lemma:matroid-mid-2} will change to
\begin{align}
\sum^{|\bm{x} \setminus \bm{x}^*|}_{i=1} div(\bm{x} \setminus \{u_i\} \cup \{v_i\}) \geq (|\bm{x} \setminus \bm{x}^*|-2) \cdot div(\bm{x})+div(\bm{x}^*)/\alpha^2
\end{align}
accordingly, which leads to a change on Eq.~(\refeq{eq-mid-31}) in the proof of Lemma~\ref{lemma:matroid-mid-3} as
\begin{align}
&|\bm{x} \setminus \bm{x}^*|\cdot \left(1+\frac{\epsilon}{rn}\right)\cdot (f(\bm{x}) +\lambda \cdot div(\bm{x}))\\
&\geq(|\bm{x} \setminus \bm{x}^*|-2) \cdot f(\bm{x})+f(\bm{x}^*) + \lambda \cdot ((|\bm{x} \setminus \bm{x}^*|-2) \cdot div(\bm{x})+div(\bm{x}^*)/\alpha^2)\\
&\geq (|\bm{x} \setminus \bm{x}^*|-2) \cdot (f(\bm{x})+\lambda \cdot div(\bm{x}))+\mathrm{OPT}/\alpha^2.
\end{align}
Note that $\alpha \geq 1$. Thus, Eq.~(\refeq{eq-mid-33}) in Lemma~\ref{lemma:matroid-mid-3} changes to
\begin{align}f(\bm{x}) +\lambda \cdot div(\bm{x}) \geq \left(\frac{1}{2}-\frac{\epsilon}{4n}\right)\cdot \frac{1}{\alpha^2}\cdot \mathrm{OPT}.\end{align}
Following the proof of Theorem~\ref{theo-matroid-sum}, we have

\begin{corollary}\label{coro:matroid}
For the result diversification problem with a matroid constraint in Definition~\ref{def-prob-matroid}, where the diversity measure $div$ is the sum-diversity in Definition~\ref{def-sum-div} and the distance function $d$ satisfies Eq.~(\refeq{eq-triangle-relax}), the expected number of iterations of the GSEMO using Eq.~(\refeq{eq-bi-problem-4}), until finding a solution $\bm{x}$ with $\bm{x} \in \mathcal{F}$ and $f(\bm{x})+\lambda \cdot div(\bm{x}) \geq (1/2-\epsilon/(4n)) \cdot \mathrm{OPT}/\alpha^2$, is at most $O(rn^3(n+(r\log r)/\epsilon))$, where $\epsilon >0$.
\end{corollary}

In Section~\ref{sec-cardinality}, we have proved in Theorem~\ref{theo-cardinality-sum} that for the result diversification problem with a cardinality constraint and the sum-diversity, the GSEMO using Eq.~(\refeq{eq-bi-problem-1}) (i.e., maximizing $(1+|\bm{x}|/k)f(\bm{x})/2+\lambda \cdot div(\bm{x})$ and minimizing $|\bm{x}|$ simultaneously) can achieve an approximation ratio of $1/2$ in polynomial time. Because a cardinality constraint $|X| \leq k$ is a specific matroid, i.e., a uniform matroid $(V,\mathcal{F})$ with $\mathcal{F}=\{X \subseteq V \mid |X| \leq k\}$, we can also apply Theorem~\ref{theo-matroid-sum} directly to get that a polynomial-time approximation ratio of $1/2-\epsilon/(4n)$ can be achieved by the GSEMO using Eq.~(\refeq{eq-bi-problem-4}), i.e., maximizing the original objective function $f(\bm{x})+\lambda \cdot div(\bm{x})$ and the subset size $|\bm{x}|$ simultaneously.

In fact, for the most straightforward bi-objective reformulation Eq.~(\refeq{eq-bi-problem-1-v}), i.e., maximizing $f(\bm{x})+\lambda \cdot div(\bm{x})$ and minimizing $|\bm{x}|$ simultaneously, the GSEMO can still achieve an approximation ratio of $1/2-\epsilon/(4n)$ in polynomial time, as shown in Theorem~\ref{theo-cardinality-sum-v}. The proof is accomplished by following the behavior of both local search (which swaps two items) and greedy (which adds a single item greedily) operations, instead of only local search. By maximizing $f(\bm{x})+\lambda \cdot div(\bm{x})$ and $|\bm{x}|$ simultaneously, the GSEMO can always keep a basis in the population and then perform local search on the basis to achieve the desired approximation ratio, as shown in the proof of Theorem~\ref{theo-matroid-sum}. When maximizing $|\bm{x}|$ is changed to minimizing $|\bm{x}|$, a basis (i.e., a solution with size $k$) may, however, be dominated by a non-basis (i.e., a solution with size smaller than $k$), implying that using the proof of Theorem~\ref{theo-matroid-sum} directly is insufficient. In this case, the non-basis with the largest value of $f+\lambda \cdot div$ in the population can be improved by adding a single item greedily. If we use $\hat{\bm{x}}$ to denote the solution with the largest value of $f+\lambda \cdot div$ in the population, the situation of $\hat{\bm{x}}$ will thus change between $|\hat{\bm{x}}|=k$ (i.e., a basis) and $|\hat{\bm{x}}|<k$ (i.e., a non-basis). Fortunately, in these two situations, a single local search and greedy operation can lead to a sufficient improvement on the objective $f(\bm{x})+\lambda \cdot div(\bm{x})$, respectively, making the total required number of iterations polynomially upper bounded for an approximation ratio of $1/2-\epsilon/(4n)$. Note that for a cardinality constraint, the corresponding rank $r$ (i.e., the size of a basis) is just $k$.

\begin{myproof}{Theorem~\ref{theo-cardinality-sum-v}}
We consider the solution (denoted as $\hat{\bm{x}}$) with the largest $f_1$ value in the population, where $f_1=f+\lambda \cdot div$. Note that the size of $\hat{\bm{x}}$ must be the largest in the population, because the solutions in the population are incomparable. Let $\bm{x}^*$ denote an optimal solution, i.e., $f(\bm{x}^*)+\lambda \cdot div(\bm{x}^*)=\mathrm{OPT}$. We divide the optimization process into two phases: (1) starts from the initial solution $\bm{0}$ and finishes when $f_1(\hat{\bm{x}}) \geq \max\{(1-1/e)\cdot f(\bm{x}^*), (2/(k(k-1)))\cdot \mathrm{OPT}\}$; (2) starts after phase~(1) and finishes when achieving the desired approximation ratio, i.e., $f_1(\hat{\bm{x}}) \geq (1/2-\epsilon/(4n))\cdot \mathrm{OPT}$. Note that $f_1(\hat{\bm{x}})$ will never be decreased because $\hat{\bm{x}}$ cannot be weakly dominated by a solution with smaller $f_1$ value.

We first consider phase~(1). Following the proof of Lemma~\ref{lemma:cardinality-sum} and deleting the factor $1/2$ appearing before $f$, we can directly get that for any $\bm{x} \in \{0,1\}^n$ with $|\bm{x}| < k$, there exists one item $v \notin \bm{x}$ such that
\begin{align}\label{eq-mid-37}
& f(\bm{x} \cup \{v\})+\lambda \cdot div(\bm{x} \cup \{v\})-f(\bm{x})-\lambda \cdot div(\bm{x}) \geq \frac{f(\bm{x}\cup \bm{x}^*)-f(\bm{x})}{k} +\frac{\lambda |\bm{x}|}{k(k-1)}\cdot div(\bm{x}^*).
\end{align}
Note that we have kept the term $f(\bm{x}\cup \bm{x}^*)$ in the right-hand side of Eq.~(\refeq{eq-mid-37}), rather than using its lower bound $f(\bm{x}^*)$ as in the proof Lemma~\ref{lemma:cardinality-sum}. Because $div$ is non-negative, $f$ is monotone and $f_1=f+\lambda \cdot div$, Eq.~(\refeq{eq-mid-37}) implies
\begin{align}\label{eq-mid-38}
& f_1(\bm{x} \cup \{v\})-f_1(\bm{x}) \geq \frac{f(\bm{x}^*)-f_1(\bm{x})}{k}.
\end{align}
By induction with Eq.~(\refeq{eq-mid-38}), we can prove that after the GSEMO optimizing Eq.~(\ref{eq-bi-problem-1-v}) (i.e., maximizing $f_1(\bm{x})$ and minimizing $|\bm{x}|$ simultaneously) for at most $enk(k+1)$ expected number of iterations, it holds that
\begin{align}\label{eq-mid-39}
f_1(\hat{\bm{x}}) \geq \left(1-\left(1-\frac{1}{k}\right)^k\right)\cdot f(\bm{x}^*)\geq  \left(1-\frac{1}{e}\right)\cdot f(\bm{x}^*).
\end{align}
The proof idea is actually to follow the behavior of the greedy algorithm maximizing $f_1$, and for the full proof we refer to that of Theorem~2 in~\cite{friedrich2015maximizing}. We pessimistically assume that $f_1(\hat{\bm{x}})$ is currently less than $(2/(k(k-1)))\cdot \mathrm{OPT}$; otherwise, phase~(1) has been finished. According to the analysis of phase~(2) in the proof of Theorem~\ref{theo-matroid-sum}, we know that to make $f_1(\hat{\bm{x}}) \geq (2/(k(k-1)))\cdot \mathrm{OPT}$, it is sufficient to flip at most two specific 0-bits and 1-bits of $\hat{\bm{x}}$, occurring with probability at least $(1/(k+1))\cdot (1/n)^4(1-1/n)^{n-4}\geq 1/(e(k+1)n^4)$ in each iteration. Note that the population size here is upper bounded by $k+1$. Thus, the expected number of iterations of phase~(1) is at most $enk(k+1)+e(k+1)n^4=O(kn^4)$.

After phase~(1), it holds that $f_1(\hat{\bm{x}}) \geq \max\{(1-1/e)\cdot f(\bm{x}^*), (2/(k(k-1)))\cdot \mathrm{OPT}\}$. We then consider phase~(2), where there are two situations for $\hat{\bm{x}}$: $|\hat{\bm{x}}|=k$ and $|\hat{\bm{x}}|<k$. When $|\hat{\bm{x}}|=k$, i.e., $\hat{\bm{x}}$ is a basis, as $f_1(\hat{\bm{x}}) \geq (2/(k(k-1)))\cdot \mathrm{OPT}$, the analysis of phase~(3) in the proof of Theorem~\ref{theo-matroid-sum} has shown that it needs at most $O(k^2n^3(\log k)/\epsilon)$ expected number of iterations to generate a $(1+\epsilon/(kn))$-approximate local optimum, i.e., achieve an approximation ratio of $1/2-\epsilon/(4n)$. When $|\hat{\bm{x}}|<k$, we know from Eq.~(\refeq{eq-mid-37}) that there exists one item $v \notin \hat{\bm{x}}$ such that
\begin{align}\label{eq-mid-40}
& f_1(\hat{\bm{x}} \cup \{v\})-f_1(\hat{\bm{x}}) \geq \frac{\lambda}{k(k-1)}\cdot div(\bm{x}^*).
\end{align}
That is, by selecting $\hat{\bm{x}}$ in line~3 of Algorithm~\ref{algo:GSEMO} and flipping only one specific 0-bit (corresponding to the above item $v$) in line~4, occurring with probability $(1/|P|)\cdot (1/n)(1-1/n)^{n-1} \geq 1/(en(k+1))$ in one iteration, $f_1(\hat{\bm{x}})$ will increase by at least $(\lambda/(k(k-1)))\cdot div(\bm{x}^*)$. Such an event is called ``a successful event". Because $f_1(\hat{\bm{x}}) \geq (1-1/e)\cdot f(\bm{x}^*)$ after phase~(1), it is sufficient to further increase $f_1(\hat{\bm{x}})$ by $(\lambda/2) \cdot div(\bm{x}^*)$ for achieving an approximation ratio of $1/2-\epsilon/(4n)$, i.e., $f_1(\hat{\bm{x}}) \geq (1/2-\epsilon/(4n))\cdot (f(\bm{x}^*)+\lambda \cdot div(\bm{x}^*))$. This implies that the required number of successful events is at most
\begin{align}
\frac{(\lambda/2) \cdot div(\bm{x}^*)}{(\lambda/(k(k-1)))\cdot div(\bm{x}^*)}=\frac{k(k-1)}{2}.
\end{align}
Thus, the expected number of iterations is at most $en(k+1) \cdot (k(k-1)/2)=O(k^3n)$. Combining the two situations of $|\hat{\bm{x}}|=k$ and $|\hat{\bm{x}}|<k$, the expected number of iterations of phase~(2) is $O(k^2n^3(\log k)/\epsilon)+O(k^3n)=O(k^2n^3(\log k)/\epsilon)$.

By combining the expected number of iterations in the above two phases, the GSEMO using Eq.~(\ref{eq-bi-problem-1-v}) requires at most $O(kn^4)+O(k^2n^3(\log k)/\epsilon)=O(kn^3(n+(k\log k)/\epsilon))$ iterations in expectation to achieve an approximation ratio of $1/2-\epsilon/(4n)$.
\end{myproof}

The alert reader may have noted that the derived upper bound $O(kn^3(n+(k\log k)/\epsilon))$ on the expected number of iterations in Theorem~\ref{theo-cardinality-sum-v} is the same as that in Theorem~\ref{theo-matroid-sum}, whose proof only follows the behavior of local search. We can find from the proof of Theorem~\ref{theo-cardinality-sum-v} that this is because the required number of iterations to follow the greedy behavior is dominated by that to follow local search.

\section{Analysis under Dynamic Environments}

In real-world scenarios, the quality function $f$ or the distance $d$ may change over time, which results in the dynamic change of the objective function $f+\lambda \cdot div$. Thus, in this section, we consider the dynamic version of the result diversification problem with a matroid constraint in Definition~\ref{def-prob-matroid}. Note that after each change, we can view the problem as a static problem with a new objective function, and run an algorithm from scratch, which, however, may lead to a significantly different solution. Thus, we usually want to maintain the solution quality by modifying the current solution without completely recomputing it. As in~\cite{borodin2017max,neumann2020analysis}, our main focus is the ability of an algorithm adapting to the change of the objective. That is, starting from a solution with a good approximation ratio for the old objective, we concern the running time of an algorithm until regaining a solution with the same approximation ratio for the new objective. Note that after dynamic change, we assume that the new quality is still monotone submodular and the new distance is still a metric.

For the result diversification problem with a matroid constraint, we have introduced in Section~\ref{sec-problem} that the local search algorithm in Algorithm~\ref{alg:local-search} can achieve a $(1/2)$-approximation ratio~\cite{borodin2017max}. Thus, a natural question is whether it can maintain a $(1/2)$-approximation ratio efficiently after the objective $f+\lambda \cdot div$ changes. Borodin \textit{et al.}~\cite{borodin2017max} have showed that for the specific case where $f$ is a modular function and the constraint is a cardinality constraint, if the magnitude of dynamic change is restricted, the local search algorithm can maintain a $(1/3)$-approximation ratio by only a single greedy swap operation. But for the general problem under dynamic environments, it is not yet clear whether the local search algorithm can maintain a $(1/2)$-approximation ratio in polynomial running time. In fact, it is still an open question whether there exists an algorithm which can maintain the optimal approximation ratio of $1/2$ in polynomial running time~\cite{borodin2017max}.

In this section, we will analyze the performance of the GSEMO for the dynamic result diversification problem. EAs are inspired from natural phenomena which have been successfully processed in dynamic natural environments, and hence the algorithmic simulations are also likely to be able to adapt the dynamic changes. The good performance of EAs has been theoretically shown on diverse dynamic optimization problems, including some artificial problems~\cite{droste2002analysis,droste2003analysis,kotzing20151+,rohlfshagen2009dynamic,shi2019reoptimization}, shortest path~\cite{lissovoi2015runtime}, makespan scheduling~\cite{neumann2015runtime}, vertex cover~\cite{pourhassan2015maintaining,pourhassan2017improved,shi2018runtime}, graph coloring~\cite{bossek2019runtime}, chance-constrained knapsack~\cite{assimi2020evolutionary} and subset selection~\cite{bian2021dynamic,aaai2021dynamic,aaai2019dynamic}. We prove in Theorem~\ref{theo-dynamic} that once seeing a change, the GSEMO can regain an asymptotically optimal approximation ratio of $1/2-\epsilon/(4n)$ after running $O(rn^3(n+(r\log r)/\epsilon))$ iterations in expectation.

\begin{theorem}\label{theo-dynamic}
For the result diversification problem with a matroid constraint in Definition~\ref{def-prob-matroid}, where the diversity measure $div$ is the sum-diversity in Definition~\ref{def-sum-div}, let $\bm{x}_{\mathrm{old}}$ denote a basis of $V$ with a $(1/2-\epsilon/(4n))$-approximation ratio, where $\epsilon >0$. After the quality function $f$ or the distance $d$ changes, the GSEMO using Eq.~(\refeq{eq-bi-problem-4}) and starting from $\bm{x}_{\mathrm{old}}$, finds a basis $\bm{x}_{\mathrm{new}}$ of $\,V$ with a $(1/2-\epsilon/(4n))$-approximation ratio for the new objective function by running $O(rn^3(n+(r\log r)/\epsilon))$ expected number of iterations.
\end{theorem}
\begin{proof}
It can be proved by directly following the phases (2) and (3) in the proof of Theorem~\ref{theo-matroid-sum}.

To prove a $(1/2-\epsilon/(4n))$-approximation ratio of the GSEMO when starting from the all-0s solution, the optimization process has been divided into three phases in the proof of Theorem~\ref{theo-matroid-sum}. The goal of phase (1) is to find a basis of $V$. After that, phase (2) is to find a basis of $V$ with the objective value at least $(2/(r(r-1)))\cdot \mathrm{OPT}$. Finally, phase (3) is to find a basis with the desired approximation ratio of $1/2-\epsilon/(4n)$.

Now, the GSEMO starts from $\bm{x}_{\mathrm{old}}$, which is a basis of $V$ with a $(1/2-\epsilon/(4n))$-approximation ratio for the old objective. Though the objective value of $\bm{x}_{\mathrm{old}}$ may be arbitrarily bad for the new objective, it is still a basis of $V$ for the problem with the new objective, because the matroid constraint does not change. Thus, the goal of phase (1) in the proof of Theorem~\ref{theo-matroid-sum} has already been reached, and we only need to reuse the processes of phases (2) and (3) to derive the expected number of iterations of the GSEMO, required to make the objective value of the basis be at least $(2/(r(r-1)))\cdot \mathrm{OPT}$ and $(1/2-\epsilon/(4n))\cdot \mathrm{OPT}$, respectively. Based on the proof of Theorem~\ref{theo-matroid-sum}, we know that the expected number of iterations for phases (2) and (3) is at most $ern^4$ and $O(r^2n^3\log r/\epsilon)$, respectively, and thus, the expected number of iterations for the GSEMO regaining a $(1/2-\epsilon/(4n))$-approximation ratio for the new objective is at most $ern^4+O(r^2n^3\log r/\epsilon)=O(rn^3(n+(r\log r)/\epsilon))$.
\end{proof}

We can also find why the local search algorithm cannot use the phases (2) and (3) in the proof of Theorem~\ref{theo-matroid-sum} to maintain a $(1/2-\epsilon/(4n))$-approximation ratio in polynomial running time. The goal of phase (2), i.e., making the objective value of the basis of $V$ lower bounded, is to ensure that the number of successful steps required in phase (3) can be polynomially upper bounded, i.e., Eq.~(\refeq{eq-mid-34}). Based on the analysis of phase (2), we know that it may require deleting two existing items in the current basis and adding two new items, which, however, cannot be accomplished by the local search algorithm. The local search algorithm in Algorithm~\ref{alg:local-search} can only perform a single swap operation (i.e., delete an old item and add a new item) in each iteration. As shown in the proof of Theorem~\ref{theo-matroid-sum}, such a behavior can be accomplished by the bit-wise mutation operator of the GSEMO, disclosing the robustness of the mutation operator of EAs against dynamic changes.

Based on the above finding, we can also modify the local search algorithm to allow at most two swaps, instead of only one swap, in each iteration. This will obviously make the local search algorithm able to follow the phases (2) and (3) in the proof of Theorem~\ref{theo-matroid-sum}, and then maintain a $(1/2-\epsilon/(4n))$-approximation ratio. We know from the proof of Theorem~\ref{theo-matroid-sum} that phase (2) requires only one local search step, while phase (3) requires $O(rn(\log r)/\epsilon)$ steps. As a local search step of allowing at most two swaps is performed in time $O(r^2(n-r)^2)$, the total running time is $O(r^3n(n-r)^2(\log r)/\epsilon)$. Note that it is incomparable with the expected time, i.e., $O(rn^3(n+(r\log r)/\epsilon))$, of the GSEMO. For example, when both $r$ and $\epsilon$ are constants, the running time upper bound of local search will be smaller; when $r=n/2$, that of the GSEMO will be smaller.

\section{Experiments}

In this section, we will examine the practical performance of the GSEMO on the applications, i.e., web-based search~\cite{borodin2017max}, multi-label feature selection~\cite{ghadiri2019distributed} and document summarization~\cite{dasgupta2013summarization}, of the result diversification problem with a cardinality constraint in Definition~\ref{def-prob-cardinality}. Note that the diversity $div$ is set as the commonly used sum-diversity in Definition~\ref{def-sum-div}, and the bi-objective reformulation in Eq.~(\refeq{eq-bi-problem-1}) is used for the GSEMO. We compare the GSEMO with the previous best algorithm, i.e., the greedy algorithm~\cite{borodin2017max}. As a cardinality constraint is a specific matroid constraint, local search~\cite{borodin2017max} can also be used for comparison. Here we even use an improved version of local search, which starts from the output solution of the greedy algorithm, and iteratively improves it by swapping two items until reaching a local maximum. The number of iterations of the GSEMO is set to $enk^3/2$ as suggested by Theorem~\ref{theo-cardinality-sum}. Note that besides the objective value $f+\lambda \cdot div$, there can be other evaluation measures (e.g., subset accuracy for multi-label feature selection and ROUGE score for document summarization) for the solution generated by each algorithm. But we will compare the objective value $f+\lambda \cdot div$ of the solutions only, which corresponds to the optimization performance of each algorithm and is the main focus of this paper. The codes and data sets are provided in \url{https://github.com/paper-submission-rafa/diversification-code}.

\subsection{Web-based Search}

Let $V=\{v_1,v_2,\ldots,v_n\}$ denote a set of documents. The goal is to find a small non-redundant subset of $V$ which is relevant to a query as much as possible. As in~\cite{borodin2017max}, we assume that the relevance of one document to a query can be treated independently from the relevance of other documents. That is, the quality $f(X)=\sum_{v \in X} f(v)$ is a modular function. The tradeoff parameter $\lambda$ between quality and diversity is set to $1.0$. We use a synthetic data set and a real-world data set \emph{letor} as in~\cite{borodin2017max}. For the synthetic one, the number $n$ of documents is set to 500, the relevance $f(v)$ of each document is randomly sampled from $[0,1]$, and the distance $d(v_i,v_j)$ between two documents is randomly sampled from $[1,2]$, which must be a metric. The data set \emph{letor} is widely used for learning to rank~\cite{qin2010letor}, where the relevance $f(v)$ of each document to a query belongs to $\{0,1,2,3,4\}$, and the distance $d(v_i,v_j)$ is the cosine similarity between the feature vectors of $v_i$ and $v_j$.

We generate 50 synthetic data sets by sampling $f(v)$ and $d(v_i,v_j)$ randomly. For \emph{letor}, we randomly sample 50 queries, and then randomly select 370 documents (i.e., $n=370$) for each query, thus also generating 50 data sets. The budget $k$ is set as $\{15,20,\ldots,50\}$. The average results of the algorithms over the 50 data sets are reported in Table~\ref{tab-web}. As expected, the objective value achieved by each algorithm increases with $k$ due to the monotonicity of the objective. Local search can improve the solution generated by the greedy algorithm except for $k=15$ on \emph{letor}, while the GSEMO always achieves the best average objective value. Moreover, the GSEMO is always significantly better than the runner-up local search by the Wilcoxon signed-rank test~\cite{wilcoxon1945individual} with confidence level $0.05$, except for $k\in\{15,20,25,40\}$ on \emph{letor}. For the Wilcoxon test, we use a two-sided test by default. Note that the standard deviation of each algorithm on \emph{letor} is much larger than that on the synthetic data set, which may be because the difference among the 50 generated data sets of \emph{letor} is larger. We can also see that the standard deviation of the GSEMO is the smallest on the same data set, implying its relatively good stability.

\begin{table*}[ht!]\caption{The objective $f+\lambda \cdot div$ value (mean$\pm$std) of the algorithms on the synthetic and \emph{letor} data sets for $\lambda=1.0$ and $k \in \{15,20,\ldots,50\}$. For each $k$, the largest objective value is bolded, and `$\bullet$' denotes that the GSEMO is significantly better than the corresponding algorithm by the Wilcoxon signed-rank test with confidence level $0.05$.}\label{tab-web}
\scriptsize
\begin{center}
\setlength{\tabcolsep}{0.75mm}{
\begin{tabular}{c|llllllll}
\hline\noalign{\vspace{0.2em}}
\multicolumn{9}{c}{synthetic data set ($n=500$)}\\
\hline
$k$  & \multicolumn{1}{c}{$15$} & \multicolumn{1}{c}{$20$} & \multicolumn{1}{c}{$25$} & \multicolumn{1}{c}{$30$} &\multicolumn{1}{c}{$35$} & \multicolumn{1}{c}{$40$} & \multicolumn{1}{c}{$45$} & \multicolumn{1}{c}{$50$}\\
\hline
Greedy & 193.9$\pm$1.40$\bullet$ & 338.1$\pm$1.86$\bullet$ & 521.2$\pm$2.37$\bullet$ & 742.6$\pm$2.93$\bullet$ & 1002.3$\pm$3.59$\bullet$ & 1299.9$\pm$4.35$\bullet$ & 1635.7$\pm$5.07$\bullet$ & 2009.5$\pm$5.85$\bullet$ \\
Local Search  & 194.7$\pm$1.25$\bullet$ & 339.4$\pm$1.59$\bullet$ & 523.0$\pm$2.12$\bullet$ & 744.7$\pm$3.08$\bullet$ & 1005.6$\pm$3.19$\bullet$ & 1303.4$\pm$3.95$\bullet$ & 1640.7$\pm$5.01$\bullet$ & 2014.4$\pm$5.63$\bullet$ \\
GSEMO  & \bf{195.5$\pm$0.86} & \bf{340.5$\pm$1.46} & \bf{524.8$\pm$1.90} & \bf{747.0$\pm$2.30} & \bf{1008.1$\pm$3.10} & \bf{1306.8$\pm$2.78} & \bf{1644.7$\pm$3.23} & \bf{2020.2$\pm$4.09} \\
\hline\hline\noalign{\vspace{0.2em}}
\multicolumn{9}{c}{\emph{letor} data set ($n=370$)}\\
\hline
Greedy & 131.9$\pm$9.69$\bullet$ & 221.9$\pm$14.33$\bullet$ & 335.4$\pm$20.88$\bullet$ & 472.2$\pm$29.60$\bullet$ & 633.0$\pm$39.59$\bullet$ & 818.2$\pm$50.35$\bullet$ & 1027.7$\pm$61.97$\bullet$ & 1261.6$\pm$73.71$\bullet$\\
Local Search & 131.9$\pm$9.71 & 223.9$\pm$12.11 & 338.3$\pm$17.55 & 478.8$\pm$20.33 $\bullet$ & 644.8$\pm$18.39$\bullet$ & 833.3$\pm$20.82 & 1046.2$\pm$23.33$\bullet$ & 1283.5$\pm$25.93$\bullet$\\
GSEMO & \bf{134.3$\pm$8.12} & \bf{226.3$\pm$10.50} & \bf{342.3$\pm$13.15} & \bf{482.2$\pm$15.79} & \bf{646.3$\pm$18.30} & \bf{835.1$\pm$20.37} & \bf{1048.4$\pm$22.70} & \bf{1286.2$\pm$24.88}\\
\hline
\end{tabular}}
\end{center}
\end{table*}

The results by fixing the budget $k=20$ and varying the tradeoff parameter $\lambda \in \{0.1,0.2,\ldots,1.0\}$ are shown in Table~\ref{tab-web-lambda}. As expected, the objective value $f+\lambda \cdot div$ achieved by each algorithm increases with $\lambda$. We can observe that local search can always improve the solution generated by the greedy algorithm, and the GSEMO can still always achieve the best average objective value. By the Wilcoxon signed-rank test with confidence level $0.05$, the GSEMO is significantly better than the runner-up local search in all cases of the synthetic data set, and four cases (i.e., $\lambda\in\{0.5,0.6,0.8,0.9\}$) of the \emph{letor} data set.

\begin{table*}[ht!]\centering\caption{The objective $f+\lambda \cdot div$ value (mean$\pm$std) of the algorithms on the synthetic and \emph{letor} data sets for $k=20$ and $\lambda \in \{0.1,0.2,\ldots,1.0\}$. For each $\lambda$, the largest objective value is bolded, and `$\bullet$' denotes that the GSEMO is significantly better than the corresponding algorithm by the Wilcoxon signed-rank test with confidence level $0.05$.}\label{tab-web-lambda}\vspace{1em}
\scriptsize
\newsavebox{\tablebox}
\begin{lrbox}{\tablebox}
\setlength{\tabcolsep}{0.3mm}{
\begin{tabular}{c|llllllllll}
\hline\noalign{\vspace{0.2em}}
\multicolumn{11}{c}{synthetic data set ($n=500$)}\\
\hline
$\lambda$  & \multicolumn{1}{c}{$0.1$} & \multicolumn{1}{c}{$0.2$} & \multicolumn{1}{c}{$0.3$} & \multicolumn{1}{c}{$0.4$} &\multicolumn{1}{c}{$0.5$} & \multicolumn{1}{c}{$0.6$} & \multicolumn{1}{c}{$0.7$} & \multicolumn{1}{c}{$0.8$} & \multicolumn{1}{c}{$0.9$} & \multicolumn{1}{c}{$1.0$}\\
\hline
Greedy & 49.8$\pm$0.35$\bullet$ & 80.9$\pm$0.66$\bullet$ & 112.4$\pm$0.77$\bullet$ & 144.5$\pm$0.92$\bullet$ & 176.7$\pm$1.19$\bullet$ & 208.9$\pm$1.28$\bullet$ & 241.2$\pm$1.44$\bullet$ & 273.4$\pm$1.47$\bullet$ & 305.7$\pm$1.59$\bullet$ & 338.1$\pm$1.86$\bullet$\\
Local Search & 50.0$\pm$0.29$\bullet$ & 81.4$\pm$0.47$\bullet$ & 113.0$\pm$0.73$\bullet$ & 145.2$\pm$0.76$\bullet$ & 177.5$\pm$0.95$\bullet$ & 209.9$\pm$1.25$\bullet$ & 242.2$\pm$1.33$\bullet$ & 274.3$\pm$1.40$\bullet$ & 307.0$\pm$1.51$\bullet$ & 339.4$\pm$1.59$\bullet$\\
GSEMO  & \bf{50.2$\pm$0.21} & \bf{81.8$\pm$0.38} & \bf{113.9$\pm$0.54} & \bf{145.9$\pm$0.67} & \bf{178.2$\pm$0.76} & \bf{210.6$\pm$0.92} & \bf{243.0$\pm$1.07} & 2\bf{75.4$\pm$0.99} & \bf{308.3$\pm$1.35} & \bf{340.5$\pm$1.46}\\
\hline\hline\noalign{\vspace{0.2em}}
\multicolumn{11}{c}{\emph{letor} data set ($n=370$)}\\
\hline
Greedy & 59.1$\pm$11.43$\bullet$ & 76.8$\pm$11.22$\bullet$ & 94.7$\pm$11.21$\bullet$ & 112.7$\pm$11.30$\bullet$ & 130.4$\pm$12.24$\bullet$ & 148.1$\pm$13.00$\bullet$ & 166.6$\pm$13.24$\bullet$ & 185.1$\pm$13.51$\bullet$ & 203.4$\pm$13.90$\bullet$ & 221.9$\pm$14.33$\bullet$\\
Local Search & \bf{59.3$\pm$11.43}  & \bf{77.4$\pm$11.14} & \bf{95.7$\pm$11.00} & \bf{114.2$\pm$10.92}  & 132.4$\pm$10.80 $\bullet$ & 150.3$\pm$11.40 $\bullet$ & 169.0$\pm$11.44 & 186.9$\pm$11.85 $\bullet$ & 205.3$\pm$11.89$\bullet$ & 223.9$\pm$12.11\\
GSEMO & \bf{59.3$\pm$11.43} & \bf{77.4$\pm$11.14} & \bf{95.7$\pm$11.01} & \bf{114.2$\pm$10.92} & \bf{132.8$\pm$10.84} & \bf{151.4$\pm$10.78} & \bf{170.1$\pm$10.70} & \bf{188.8$\pm$10.62} & \bf{207.5$\pm$10.58} & \bf{226.3$\pm$10.50}\\
\hline
\end{tabular}}
\end{lrbox}
\scalebox{0.9}{\usebox{\tablebox}}
\end{table*}

\subsection{Multi-label Feature Selection}

Let $V$ denote a set of features, and $L$ denote a set of labels. The goal is to find a small non-redundant subset of $V$ which can predict labels in $L$ accurately. The predicted quality of a subset $X$ of features is measured by $f(X)=\sum_{l\in L} \mathrm{top}^p_{v\in X} \{\mathit{MI}(v,l)\}$, where $\mathrm{top}^p_{v\in X} \{\mathit{MI}(v,l)\}$ is the sum of the $p$ largest values in $\{\mathit{MI}(v,l)\mid v \in X\}$, $\mathit{MI}(v,l)=I(v,l)/\sqrt{H(v)H(l)}$, $I(\cdot,\cdot)$ is the mutual information and $H(\cdot)$ is the entropy. It satisfies the monotone submodular property~\cite{ghadiri2019distributed}. The distance $d(v_i,v_j)$ between features is measured by $1-I(v_i,v_j)/H(v_i,v_j)$, which is a metric~\cite{vinh2010information}. $p$ is set to $10$, and the tradeoff parameter $\lambda$ is set to 0.5.

We use two multi-label data sets \emph{enron} (1001 \#feat, 53 \#labels) and \emph{medical} (1449 \#feat, 45 \#labels) from \url{http://mulan.sourceforge.net/datasets-mlc.html}. The results are shown in Table~\ref{tab-feature}. As the data set is fixed, the greedy algorithm and local search have only one output objective value. But for the GSEMO, which is a randomized algorithm, we repeat its run ten times independently, and report the average objective value and the standard deviation. Note that the standard deviation of the GSEMO is 0 sometimes, which is because the same good solutions are found in the ten runs. We can see that local search brings little improvement on the solution generated by the greedy algorithm, while the GSEMO can achieve much larger objective values, and is always significantly better than the greedy algorithm and local search.

\begin{table*}[ht!]\caption{The objective $f+\lambda \cdot div$ value (mean$\pm$std) of the algorithms on the \emph{enron} and \emph{medical} data sets for $\lambda=0.5$ and $k \in \{15,20,\ldots,50\}$. For each $k$, the largest objective value is bolded, and `$\bullet$' denotes that the GSEMO is significantly better than the corresponding algorithm by the Wilcoxon rank-sum test with confidence level $0.05$.}\label{tab-feature}
\scriptsize
\begin{center}
\setlength{\tabcolsep}{1mm}{
\begin{tabular}{c|llllllll}
\hline\noalign{\vspace{0.2em}}
\multicolumn{9}{c}{\emph{enron} data set ($n=1001$)}\\
\hline
$k$  & \multicolumn{1}{c}{$15$} & \multicolumn{1}{c}{$20$} & \multicolumn{1}{c}{$25$} & \multicolumn{1}{c}{$30$} &\multicolumn{1}{c}{$35$} & \multicolumn{1}{c}{$40$} & \multicolumn{1}{c}{$45$} & \multicolumn{1}{c}{$50$}\\
\hline
Greedy &58.9$\bullet$ & 102.3$\bullet$ & 158.1$\bullet$ & 226.3$\bullet$ & 307.7$\bullet$ & 400.9$\bullet$ & 505.7$\bullet$ & 622.7$\bullet$\\
Local Search & 59.0$\bullet$ & 102.4$\bullet$ & 158.2$\bullet$ & 226.4$\bullet$ & 307.8$\bullet$ & 401.0$\bullet$ & 505.9$\bullet$ & 622.9$\bullet$\\
GSEMO & \bf{64.7$\pm$0.00} & \bf{108.0$\pm$0.01} & \bf{163.2$\pm$0.02} & \bf{230.7$\pm$0.00} & \bf{310.5$\pm$0.02} & \bf{402.6$\pm$0.02} & \bf{507.1$\pm$0.01} & \bf{624.0$\pm$0.01}\\
\hline\hline\noalign{\vspace{0.2em}}
\multicolumn{9}{c}{\emph{medical} data set ($n=1449$)}\\
\hline
Greedy & 79.4$\bullet$ & 122.4$\bullet$ & 179.1$\bullet$ & 246.5$\bullet$ & 326.4$\bullet$ & 424.1$\bullet$ & 528.8$\bullet$ & 645.9$\bullet$\\
Local Search & 79.5$\bullet$ & 122.4$\bullet$ & 179.2$\bullet$ & 246.6$\bullet$ & 326.5$\bullet$ & 424.2$\bullet$ & 528.9$\bullet$ & 646.0$\bullet$\\
GSEMO &\bf{95.6$\pm$0.00} & \bf{141.9$\pm$0.04} & \bf{199.8$\pm$0.04} & \bf{269.5$\pm$0.04} & \bf{351.2$\pm$0.05} & \bf{444.8$\pm$0.00} & \bf{550.5$\pm$0.01} & \bf{668.2$\pm$0.01}\\
\hline
\end{tabular}}
\end{center}
\end{table*}

Table~\ref{tab-feature-lambda} shows the results when the budget $k=20$ and the tradeoff parameter $\lambda$ is varied from $0.1$ to $1$ with a gap of $0.1$. As we have observed in Table~\ref{tab-feature}, local search brings little improvement on the solution generated by the greedy algorithm, while the GSEMO can achieve much larger objective values. Furthermore, the GSEMO is always significantly better than the greedy algorithm and local search by the Wilcoxon rank-sum test~\cite{wilcoxon1945individual} with confidence level $0.05$.

\begin{table*}[ht!]\caption{The objective $f+\lambda \cdot div$ value (mean$\pm$std) of the algorithms on the \emph{enron} and \emph{medical} data sets for $k=20$ and $\lambda \in \{0.1,0.2,\ldots,1.0\}$. For each $\lambda$, the largest objective value is bolded, and `$\bullet$' denotes that the GSEMO is significantly better than the corresponding algorithm by the Wilcoxon rank-sum test with confidence level $0.05$.}\label{tab-feature-lambda}
\scriptsize
\begin{center}
\setlength{\tabcolsep}{0.5mm}{
\begin{tabular}{c|llllllllll}
\hline\noalign{\vspace{0.2em}}
\multicolumn{11}{c}{\emph{enron} data set ($n=1001$)}\\
\hline
$\lambda$  & \multicolumn{1}{c}{$0.1$} & \multicolumn{1}{c}{$0.2$} & \multicolumn{1}{c}{$0.3$} & \multicolumn{1}{c}{$0.4$} &\multicolumn{1}{c}{$0.5$} & \multicolumn{1}{c}{$0.6$} & \multicolumn{1}{c}{$0.7$} & \multicolumn{1}{c}{$0.8$} & \multicolumn{1}{c}{$0.9$} & \multicolumn{1}{c}{$1.0$}\\
\hline
Greedy & 26.6$\bullet$ & 45.5$\bullet$ & 64.4$\bullet$ & 83.4$\bullet$ & 102.3$\bullet$ & 121.3$\bullet$ & 140.2$\bullet$ & 159.1$\bullet$ & 178.1$\bullet$ & 197.0$\bullet$\\
Local Search & 26.6$\bullet$& 45.5$\bullet$& 64.5$\bullet$& 83.4$\bullet$& 102.4$\bullet$& 121.3$\bullet$& 140.3$\bullet$& 159.2$\bullet$& 178.2$\bullet$& 197.1$\bullet$\\
GSEMO  & \bf{35.3$\pm$0.00} & \bf{53.1$\pm$0.00} & \bf{71.2$\pm$0.01} & \bf{89.6$\pm$0.01} & \bf{108.0$\pm$0.01} & \bf{126.5$\pm$0.01} & \bf{145.0$\pm$0.02} & \bf{163.7$\pm$0.02} & \bf{182.4$\pm$0.02} & \bf{201.1$\pm$0.02}\\
\hline\hline\noalign{\vspace{0.2em}}
\multicolumn{11}{c}{\emph{medical} data set ($n=1449$)}\\
\hline
Greedy &46.7$\bullet$ & 65.6$\bullet$ & 84.5$\bullet$ & 103.4$\bullet$ & 122.4$\bullet$ & 141.3$\bullet$ & 160.2$\bullet$ & 179.1$\bullet$ & 198.0$\bullet$ & 216.9$\bullet$\\
Local Search & 46.7$\bullet$& 65.6$\bullet$& 84.6$\bullet$& 103.5$\bullet$& 122.4$\bullet$& 141.4$\bullet$& 160.3$\bullet$& 179.2$\bullet$& 198.1$\bullet$& 217.1$\bullet$\\
GSEMO & \bf{68.3$\pm$0.00} & \bf{86.6$\pm$0.02} & \bf{104.9$\pm$0.01} & \bf{123.4$\pm$0.05} & \bf{141.9$\pm$0.04} & \bf{160.4$\pm$0.03} & \bf{179.0$\pm$0.05} & \bf{197.5$\pm$0.05} & \bf{216.1$\pm$0.06} & \bf{234.7$\pm$0.00}
\\
\hline
\end{tabular}}
\end{center}
\end{table*}

\subsection{Document Summarization}

Let $\mathcal{C}$ denote a collection of documents, and $V=\{v_1,v_2,\ldots,v_n\}$ denote the set of sentences contained by $\mathcal{C}$. The goal is to find a small non-redundant subset of $V$ which can make a good summary. As in~\cite{dasgupta2013summarization}, each sentence is represented by a set of dependency relations present in the sentence. A dependency relation is denoted by $(rel:a,b)$, where $rel$ is a relation type (e.g., nsubj) and $a$, $b$ are the two arguments present in the dependency relation. For example, the sentence ``We like sports" is represented by the two dependency relations (nsubj: like, We) and (dobj: like, sports). Based on this structured representation, a graph can be constructed by treating each sentence (i.e., a set of dependency relations) as a node and using the semantic similarity between two sentences as the edge weight. For two sentences $v_i$ and $v_j$, the semantic similarity $s(v_i,v_j)$ is calculated by $\sum_{(rel:a,b)\in v_i, (rel:a',b')\in v_j} \mathit{WN}(a,a')\times \mathit{WN}(b,b')$, where $\mathit{WN}(\cdot,\cdot)$ denotes the WordNet similarity score~\cite{patwardhan2006using} between two words. The quality function $f$ and the diversity can then be defined based on this graph.

The summary quality of a subset $X$ of sentences can be measured by a linear combination of three terms: popularity, cluster contribution and cover contribution~\cite{dasgupta2013summarization}. Popularity requires the selected subset $X$ of sentences to include popular views expressed across multiple documents. For each sentence (i.e., node on the graph) $v$, let $w(v)$ denote the number of documents such that at least one relation $rel$ in $v$ appears in a sentence of the document. Then, the popularity of $X$ is measured by $\sum_{v\in X} w(v)$. Cluster contribution requires the selected subset $X$ of sentences not to include too many sentences from the same document, and is measured by $\sum_{C \in \mathcal{C}} |X \cap C|^{1/2}$. For a sentence (i.e., node) $v$ and a subset $S$ of sentences, let $cov(v,S)=\sum_{u \in S} s(v,u)$, i.e., the sum of similarity between $v$ and the sentences in $S$. Cover contribution of the selected subset $X$ of sentences is defined as $\sum_{v \in X} \min\{cov(v,X),0.25\cdot cov(v,V)\}$. Finally, the summary quality $f$ of a subset $X$ of sentences is defined as $f(X)=\sum_{v\in X} w(v)+\alpha \cdot \sum_{C \in \mathcal{C}} |X \cap C|^{1/2} +\beta \cdot \sum_{v \in X} \min\{cov(v,X),0.25\cdot cov(v,V)\}$, which is monotone submodular, because each item (i.e., popularity, cluster contribution or cover contribution) is monotone submodular~\cite{dasgupta2013summarization}. In the experiments, we set $\alpha=0.2$ and $\beta=0.8$.

As the semantic similarity $s(v_i,v_j)$ between two sentences has been calculated, their distance can be naturally measured by $d'(v_i,v_j)=1-s(v_i,v_j)$, which is, however, not a metric. In~\cite{dasgupta2013summarization}, the edge weight $s(v_i,v_j)$ between $v_i$ and $v_j$ on the graph is replaced by $d'(v_i,v_j)$, and the weight of the shortest path from $v_i$ to $v_j$ in the resulting graph is used as the distance $d(v_i,v_j)$, which is now a metric.

We use two real-world corpora for experiments. One is the \emph{DUC 2004} corpus that comprises 50 clusters, which can be downloaded from \url{https://duc.nist.gov/duc2004/tasks.html}. Each cluster corresponds to one summarization task, which contains 10 documents. For the resulting 50 tasks, the number $n$ of sentences is varied from 141 to 494, and the average number is 235.7. For each task, the goal is to select a subset of sentences to make a good summary of the contained 10 documents. The other corpus is \emph{New York Times Comments} (briefly denoted as \emph{NYT Comments}), which can be downloaded from \url{https://www.kaggle.com/aashita/nyt-comments}. We extract 50 articles, each of which is associated with anywhere from 100-200 comments. A comment contains several sentences, and can be viewed as a document. Now each article corresponds to one summarization task, and the goal is to select a subset of sentences to make a good summary of the associated comments. For these 50 tasks, the number $n$ of sentences is varied from 218 to 771, and is 410.7 on average.

Table~\ref{tab-documment} shows the results by fixing the tradeoff parameter $\lambda$ between quality and diversity as $1.0$ and varying the budget $k$ among $\{15,20,\ldots,50\}$. Table~\ref{tab-documment-lambda} shows the results by fixing $k=20$ and varying $\lambda$ among $\{0.1,0.2,\ldots,1.0\}$. We can observe that local search is better than the greedy algorithm, and the GSEMO always achieves the best average objective value. By the Wilcoxon signed-rank test with confidence level $0.05$, the GSEMO is significantly better than local search in most cases, except for $k\in\{30,40,45,50\}$ on \emph{DUC 2004} and $k\in\{20,40,50\}$ on \emph{NYT Comments} in Table~\ref{tab-documment}, and $\lambda=1.0$ on \emph{NYT Comments} in Table~\ref{tab-documment-lambda}. We also note that for \emph{NYT Comments} in Table~\ref{tab-documment-lambda}, the average objective value obtained by the greedy algorithm even decreases when $\lambda$ increases (e.g., from $0.4$ to $0.5$), disclosing the limited performance of the greedy algorithm in practice. By simply selecting the same subset of sentences, increasing $\lambda$ obviously will lead to a larger objective value. However, a larger $\lambda$ may actually lead to a very different search behavior of the greedy algorithm, which may even generate a subset with a smaller objective value as we have observed.

\begin{table*}[h!]\caption{The objective $f+\lambda \cdot div$ value (mean$\pm$std) of the algorithms on the \emph{DUC2004} and \emph{NYT Comments} data sets for $\lambda=1.0$ and $k \in \{15,20,\ldots,50\}$. For each $k$, the largest objective value is bolded, and `$\bullet$' denotes that the GSEMO is significantly better than the corresponding algorithm by the Wilcoxon signed-rank test with confidence level $0.05$.}\label{tab-documment}
\scriptsize
\begin{center}
\setlength{\tabcolsep}{0.75mm}{
\begin{tabular}{c|llllllll}
\hline\noalign{\vspace{0.2em}}
\multicolumn{9}{c}{\emph{DUC 2004} ($n=235.7$ on average)}\\
\hline
$k$  & \multicolumn{1}{c}{$15$} & \multicolumn{1}{c}{$20$} & \multicolumn{1}{c}{$25$} & \multicolumn{1}{c}{$30$} &\multicolumn{1}{c}{$35$} & \multicolumn{1}{c}{$40$} & \multicolumn{1}{c}{$45$} & \multicolumn{1}{c}{$50$}\\
\hline
Greedy &117.0$\pm$9.79$\bullet$ & 200.4$\pm$16.76$\bullet$ & 302.7$\pm$26.81$\bullet$ & 419.3$\pm$37.42$\bullet$ & 546.1$\pm$54.05$\bullet$ & 682.2$\pm$74.92$\bullet$ & 826.8$\pm$102.72$\bullet$ & 978.6$\pm$133.58$\bullet$ \\
Local Search & 140.8$\pm$7.77$\bullet$ & 236.0$\pm$15.12$\bullet$ & 345.9$\pm$27.93$\bullet$ & 467.4$\pm$45.53 & 596.3$\pm$67.55 $\bullet$ & 734.0$\pm$95.35 & 878.2$\pm$126.01 & 1030.8$\pm$162.46 \\
GSEMO &\bf{146.7$\pm$8.97} & \bf{243.3$\pm$18.59} & \bf{353.1$\pm$32.75} & \bf{469.8$\pm$48.94} & \bf{598.6$\pm$69.25} & \bf{737.3$\pm$100.72} & \bf{880.7$\pm$130.46} & \bf{1031.2$\pm$163.16}\\
\hline\hline\noalign{\vspace{0.2em}}
\multicolumn{9}{c}{\emph{NYT Comments} ($n=410.7$ on average)}\\
\hline
$k$  & \multicolumn{1}{c}{$15$} & \multicolumn{1}{c}{$20$} & \multicolumn{1}{c}{$25$} & \multicolumn{1}{c}{$30$} &\multicolumn{1}{c}{$35$} & \multicolumn{1}{c}{$40$} & \multicolumn{1}{c}{$45$} & \multicolumn{1}{c}{$50$}\\
\hline
Greedy & 114.4$\pm$6.09$\bullet$ & 198.9$\pm$11.32$\bullet$ & 303.4$\pm$18.22$\bullet$ & 426.5$\pm$27.38$\bullet$ & 567.0$\pm$39.12$\bullet$ & 722.2$\pm$51.43$\bullet$ & 892.2$\pm$66.61$\bullet$ & 1075.4$\pm$82.26$\bullet$ \\
Local Search & 136.3$\pm$6.78$\bullet$ & 226.9$\pm$13.01 & 342.1$\pm$19.42$\bullet$ & 476.8$\pm$26.03$\bullet$ & 628.0$\pm$36.32$\bullet$ & 793.9$\pm$50.07 & 975.5$\pm$67.28 $\bullet$ & 1170.1$\pm$90.09 \\
GSEMO &\bf{138.6$\pm$6.55} & \bf{231.8$\pm$10.77} & \bf{346.2$\pm$17.75} & \bf{480.4$\pm$25.26} & \bf{632.3$\pm$35.82} & \bf{798.7$\pm$49.87} & \bf{978.5$\pm$67.66} & \bf{1171.7$\pm$89.97}\\
\hline
\end{tabular}}
\end{center}
\end{table*}

\begin{table*}[h!]\centering\caption{The objective $f+\lambda \cdot div$ value (mean$\pm$std) of the algorithms on the \emph{DUC2004} and \emph{NYT Comments} data sets for $k=20$ and $\lambda \in \{0.1,0.2,\ldots,1.0\}$. For each $\lambda$, the largest objective value is bolded, and `$\bullet$' denotes that the GSEMO is significantly better than the corresponding algorithm by the Wilcoxon signed-rank test with confidence level $0.05$.}\label{tab-documment-lambda}\vspace{1em}
\scriptsize
\begin{lrbox}{\tablebox}
\setlength{\tabcolsep}{0.6mm}{
\begin{tabular}{c|llllllllll}
\hline\noalign{\vspace{0.2em}}
\multicolumn{11}{c}{\emph{DUC 2004} ($n=235.7$ on average)}\\
\hline
$\lambda$  & \multicolumn{1}{c}{$0.1$} & \multicolumn{1}{c}{$0.2$} & \multicolumn{1}{c}{$0.3$} & \multicolumn{1}{c}{$0.4$} &\multicolumn{1}{c}{$0.5$} & \multicolumn{1}{c}{$0.6$} & \multicolumn{1}{c}{$0.7$} & \multicolumn{1}{c}{$0.8$} & \multicolumn{1}{c}{$0.9$} & \multicolumn{1}{c}{$1.0$}\\
\hline
Greedy & 175.4$\pm$25.08$\bullet$ & 177.8$\pm$24.28$\bullet$ & 180.4$\pm$23.64$\bullet$ & 183.6$\pm$23.09$\bullet$ & 185.6$\pm$24.58$\bullet$ & 187.7$\pm$24.49$\bullet$ & 189.1$\pm$21.95$\bullet$ & 191.0$\pm$19.40$\bullet$ & 194.7$\pm$18.41$\bullet$ & 200.4$\pm$16.76$\bullet$ \\
Local Search & 208.9$\pm$24.29$\bullet$ & 212.0$\pm$23.88$\bullet$ & 212.6$\pm$21.61$\bullet$ & 216.2$\pm$20.88$\bullet$ & 220.6$\pm$21.93$\bullet$ & 223.0$\pm$21.05$\bullet$ & 225.0$\pm$19.19$\bullet$ & 226.6$\pm$17.44$\bullet$ & 230.9$\pm$16.20$\bullet$ & 236.0$\pm$15.12$\bullet$ \\
GSEMO  & \bf{217.0$\pm$27.54} & \bf{219.9$\pm$25.95} & \bf{222.1$\pm$25.79} & \bf{226.0$\pm$25.34} & \bf{227.6$\pm$23.80} & \bf{230.2$\pm$23.70} & \bf{232.1$\pm$22.76} & \bf{238.0$\pm$23.06} & \bf{239.3$\pm$20.05} & \bf{243.3$\pm$18.59}\\
\hline\hline\noalign{\vspace{0.2em}}
\multicolumn{11}{c}{\emph{NYT Comments} ($n=410.7$ on average)}\\
\hline
$\lambda$  & \multicolumn{1}{c}{$0.1$} & \multicolumn{1}{c}{$0.2$} & \multicolumn{1}{c}{$0.3$} & \multicolumn{1}{c}{$0.4$} &\multicolumn{1}{c}{$0.5$} & \multicolumn{1}{c}{$0.6$} & \multicolumn{1}{c}{$0.7$} & \multicolumn{1}{c}{$0.8$} & \multicolumn{1}{c}{$0.9$} & \multicolumn{1}{c}{$1.0$}\\
\hline
Greedy & 192.3$\pm$18.09$\bullet$ & 194.1$\pm$17.43$\bullet$ & 195.8$\pm$17.21$\bullet$ & 196.5$\pm$18.55$\bullet$ & 189.0$\pm$25.97$\bullet$ & 178.2$\pm$17.68$\bullet$ & 181.1$\pm$15.65$\bullet$ & 186.0$\pm$15.07$\bullet$ & 189.8$\pm$13.13$\bullet$ & 198.9$\pm$11.32$\bullet$ \\
Local Search & 210.2$\pm$14.76$\bullet$ & 211.7$\pm$14.49$\bullet$ & 213.2$\pm$15.91$\bullet$ & 214.3$\pm$16.75$\bullet$ & 216.6$\pm$16.17$\bullet$ & 219.4$\pm$14.37$\bullet$ & 220.6$\pm$13.97$\bullet$ & 223.1$\pm$13.22$\bullet$ & 223.5$\pm$13.68$\bullet$ & 226.9$\pm$13.01\\
GSEMO & \bf{213.8$\pm$14.97} & \bf{215.2$\pm$14.43}  & \bf{216.9$\pm$14.83}  & \bf{218.9$\pm$13.68} & \bf{220.0$\pm$14.43}  & \bf{222.0$\pm$14.20}  & \bf{224.9$\pm$13.39}  & \bf{226.6$\pm$13.60}  & \bf{227.5$\pm$12.44}  & \bf{231.8$\pm$10.77} \\
\hline
\end{tabular}}
\end{lrbox}
\scalebox{0.85}{\usebox{\tablebox}}
\end{table*}

\subsection{Running Time}\label{sec-runtime}

In the previous subsections, we have shown the superior optimization performance of the GSEMO over the greedy algorithm and local search. Next, we consider the running time, measured in the number of objective function evaluations. The greedy algorithm takes nearly $kn$ time. Local search starts from the output of the greedy algorithm, and repeatedly performs the best local swap operation (each costs $k(n-k)$ evaluations) until convergence. For the GSEMO, the number of iterations has been set to $enk^3/2$ derived in theoretical analysis, and each iteration costs one evaluation. We want to examine how efficient the GSEMO can be in practice. When $k=20$, we plot the curve of the average objective value over the running time for these algorithms on the tested six data sets. In Figure~\ref{fig-time}, one unit on the $x$-axis corresponds to $kn$ objective evaluations. The GSEMO uses more time to achieve a better performance in Figure~\ref{fig-time}(a--b), while can be faster and better in Figure~\ref{fig-time}(c--f). Compared with the theoretical running time $enk^3/2\approx 543 kn$, the GSEMO is much more efficient. This is expected, because we have used a theoretical worst-case upper bound on the time for the GSEMO to achieve a good approximation. Note that the curves of the greedy algorithm and local search in Figure~\ref{fig-time}(c--d) are almost overlapped, which is because the improvement of local search over the greedy algorithm is very little in these two cases.

\begin{figure*}[ht!]\centering
\begin{minipage}[c]{0.32\linewidth}\centering
        \includegraphics[width=1\linewidth]{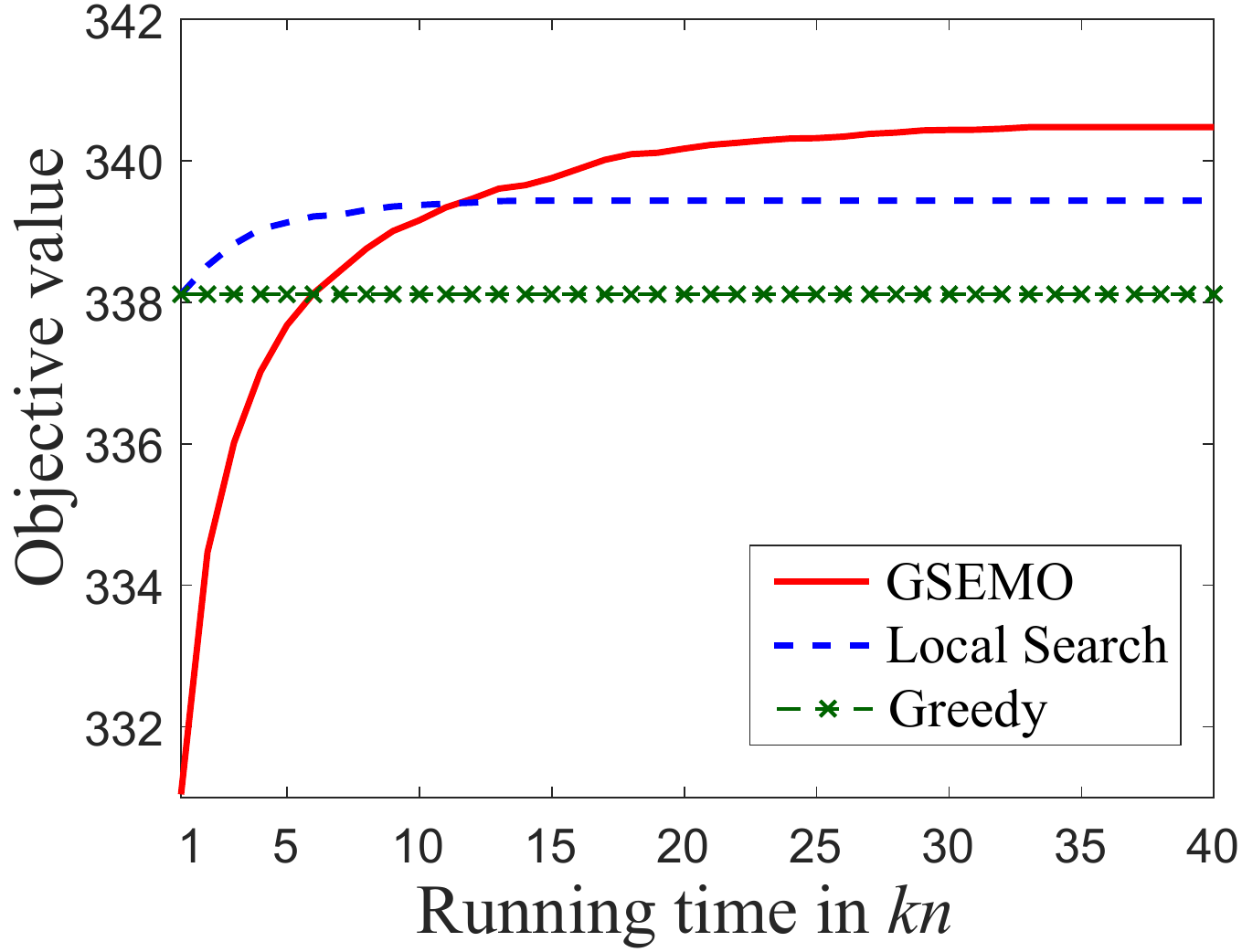}
\end{minipage}
\begin{minipage}[c]{0.32\linewidth}\centering
        \includegraphics[width=1\linewidth]{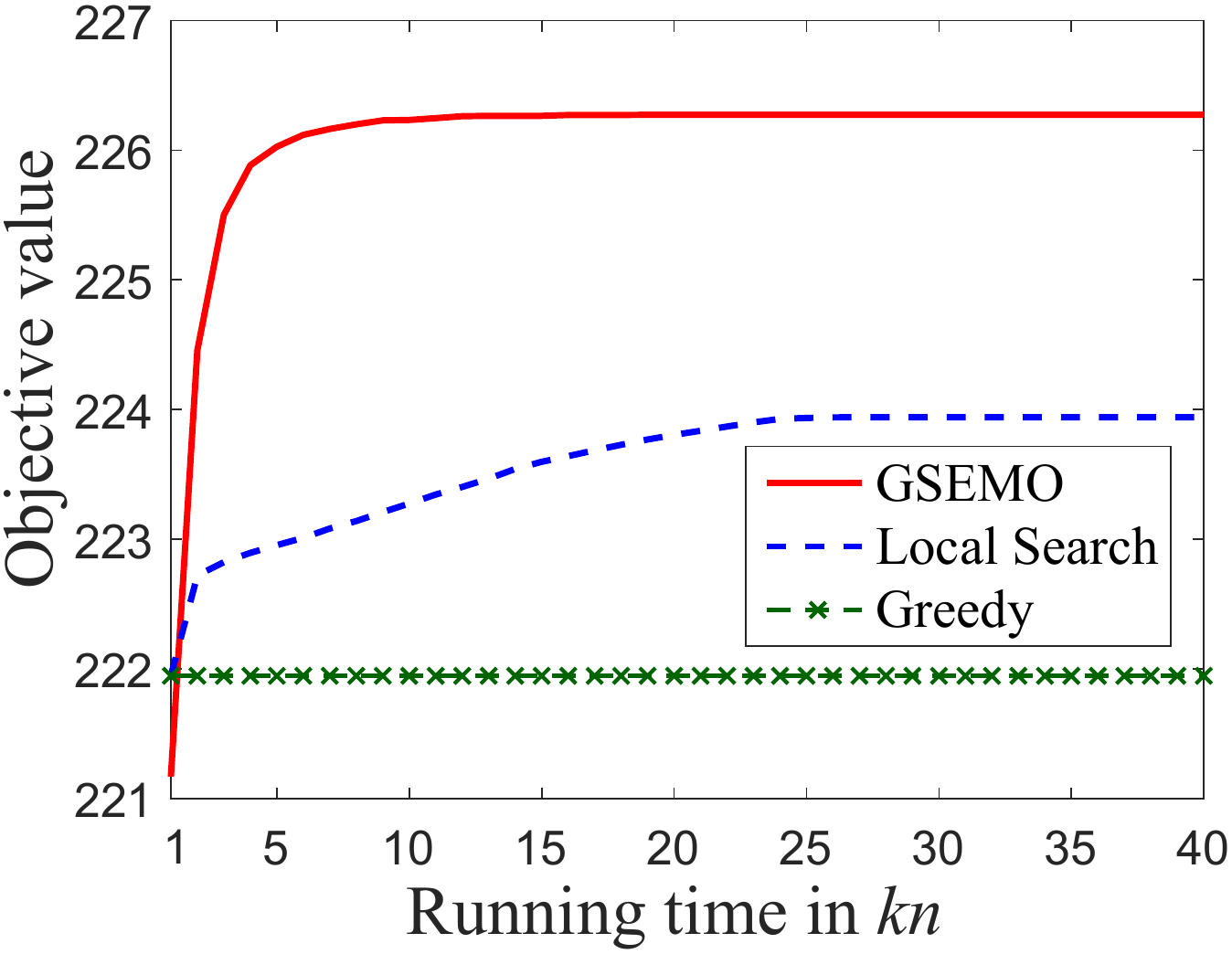}
\end{minipage}
\begin{minipage}[c]{0.32\linewidth}\centering
        \includegraphics[width=1\linewidth]{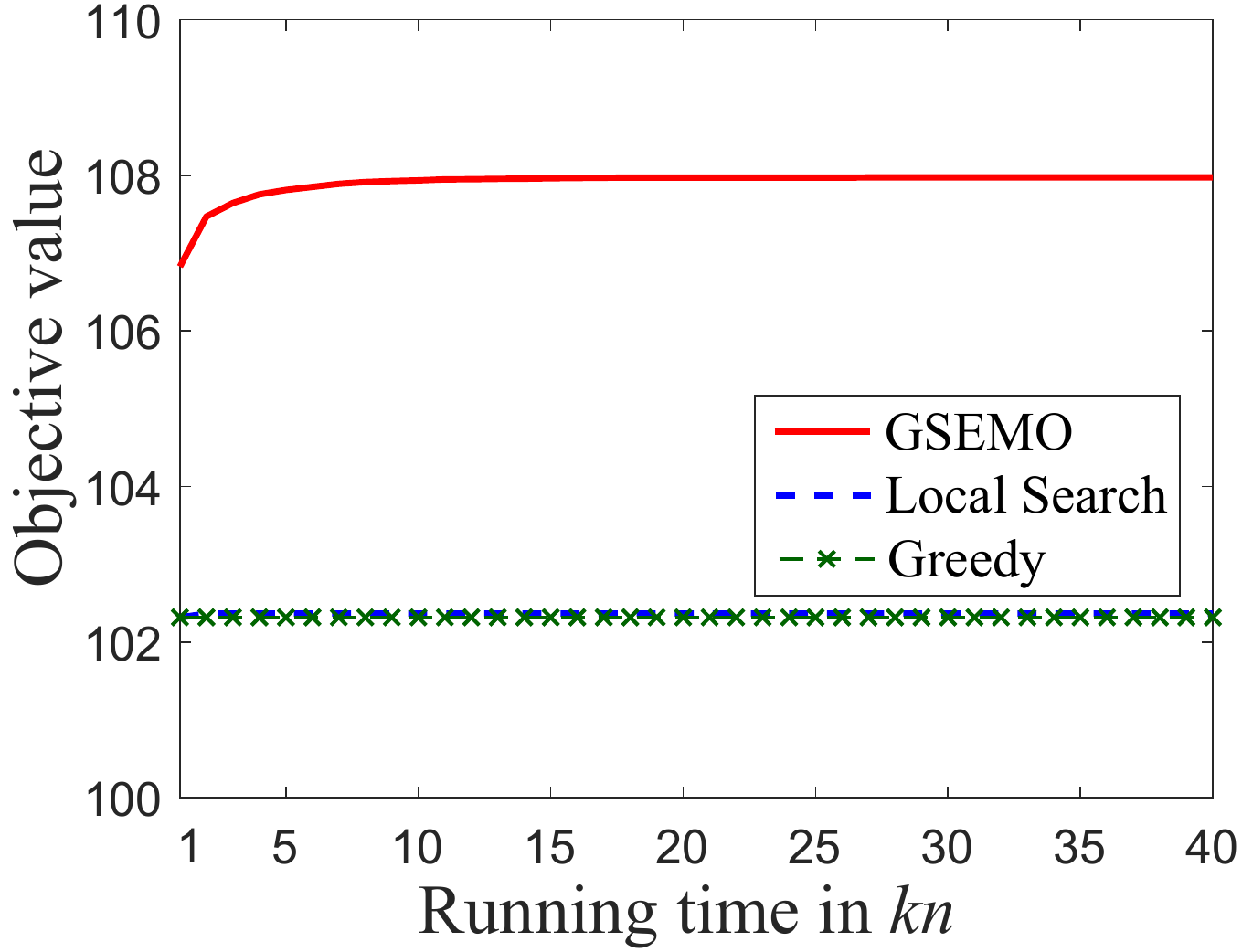}
\end{minipage}\\\vspace{0.5em}
\begin{minipage}[c]{0.32\linewidth}\centering
    \small(a) synthetic
\end{minipage}
\begin{minipage}[c]{0.32\linewidth}\centering
    \small(b) \emph{letor}
\end{minipage}
\begin{minipage}[c]{0.32\linewidth}\centering
    \small(c) \emph{enron}
\end{minipage}
\\\vspace{0.8em}
\begin{minipage}[c]{0.32\linewidth}\centering
        \includegraphics[width=1\linewidth]{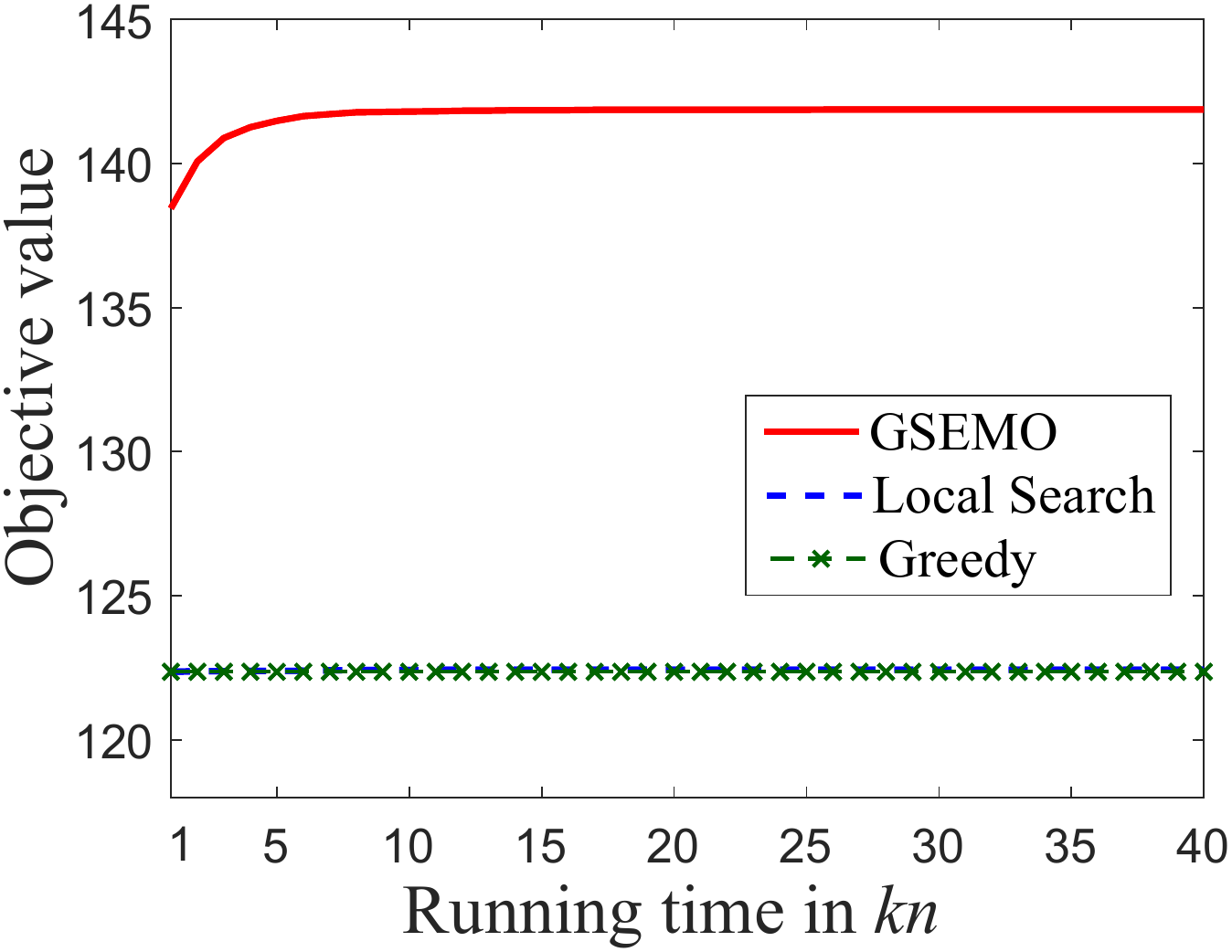}
\end{minipage}
\begin{minipage}[c]{0.32\linewidth}\centering
        \includegraphics[width=1\linewidth]{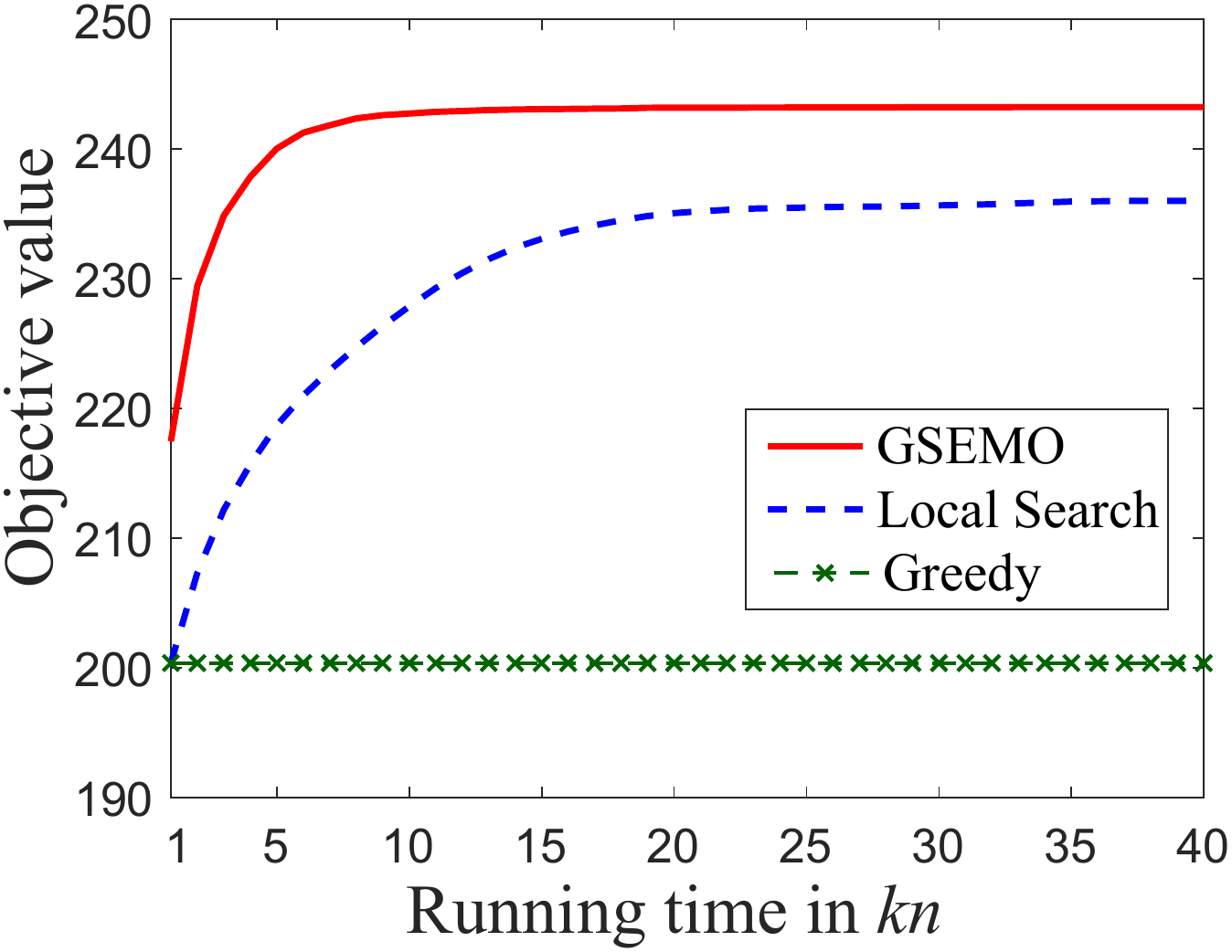}
\end{minipage}
\begin{minipage}[c]{0.32\linewidth}\centering
        \includegraphics[width=1\linewidth]{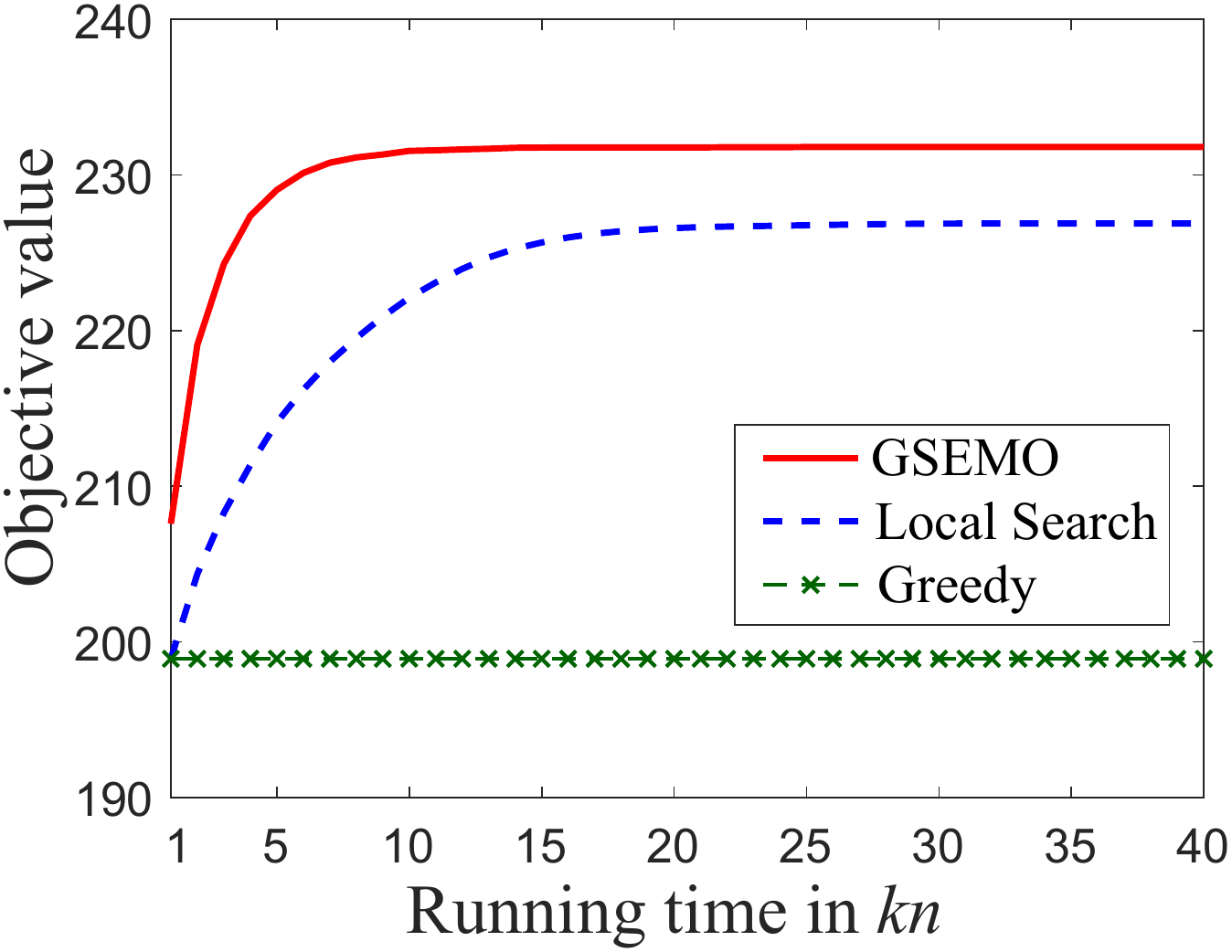}
\end{minipage}\\\vspace{0.5em}
\begin{minipage}[c]{0.32\linewidth}\centering
    \small(d) \emph{medical}
\end{minipage}
\begin{minipage}[c]{0.32\linewidth}\centering
    \small(e) \emph{DUC 2004}
\end{minipage}
\begin{minipage}[c]{0.32\linewidth}\centering
    \small(f) \emph{NYT Comments}
\end{minipage}\\
\caption{The average objective value vs. running time (i.e., number of objective evaluations), where $k=20$.}\label{fig-time}
\end{figure*}

\subsection{Dynamic Environments}

Finally, we examine the performance of the GSEMO under dynamic environments, by using the synthetic data set on the application of web-based search with $\lambda=1.0$ and $k=20$. Each dynamic change on the objective perturbs the relevance of an item or the distance between two items uniformly at random (i.e., each with probability $1/2$) and repeats this process $m$ times independently. If making a perturbation on the relevance, an item $v$ is randomly chosen, and its relevance $f(v)$ is reset from $[0,1]$ randomly. If making a perturbation on the distance, two items $v_i,v_j$ are randomly chosen, and their distance $d(v_i,v_j)$ is reset from $[1,2]$ randomly. $m$ is set to 50.

Note that the greedy algorithm cannot be applied under dynamic environments. To make fair comparison, both local search and the GSEMO start from the solution generated by the greedy algorithm, and run with the same number $t$ of objective function evaluations once seeing a dynamic change on the objective. $t$ is set to $10kn$. The average results over 50 synthetic data sets are shown in Figure~\ref{fig-dynamic}(a), where the objective changes 50 times. It can be clearly observed that the GSEMO always performs better than local search. By setting $t=5kn$ (i.e., a shorter running time allowed after each dynamic change) or $m=300$ (i.e., a larger magnitude of dynamic change), the GSEMO is still better as shown in Figure~\ref{fig-dynamic}(b--c).

\begin{figure*}[ht!]\centering
\begin{minipage}[c]{0.32\linewidth}\centering
        \includegraphics[width=1\linewidth]{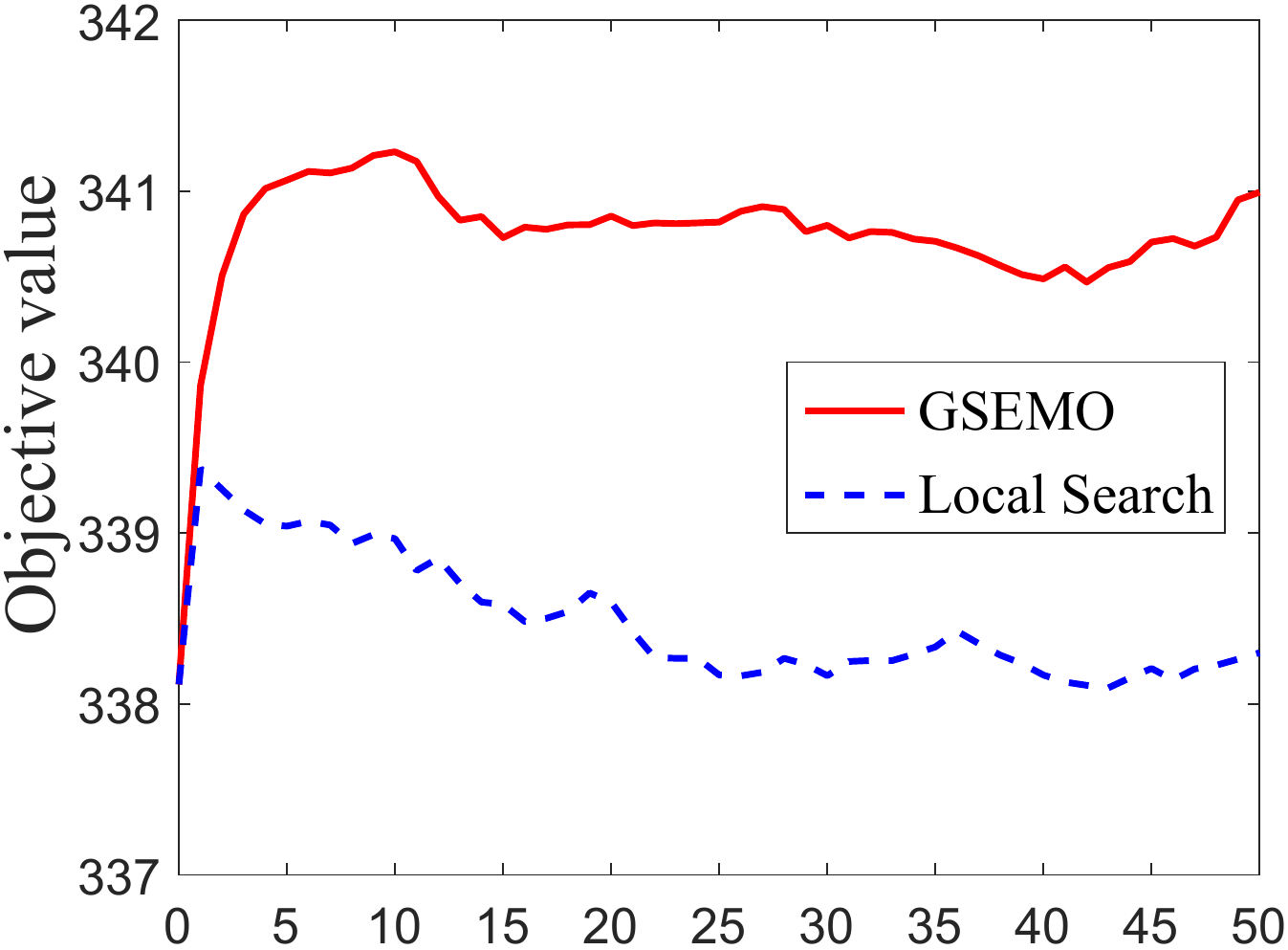}
\end{minipage}
\begin{minipage}[c]{0.32\linewidth}\centering
        \includegraphics[width=1\linewidth]{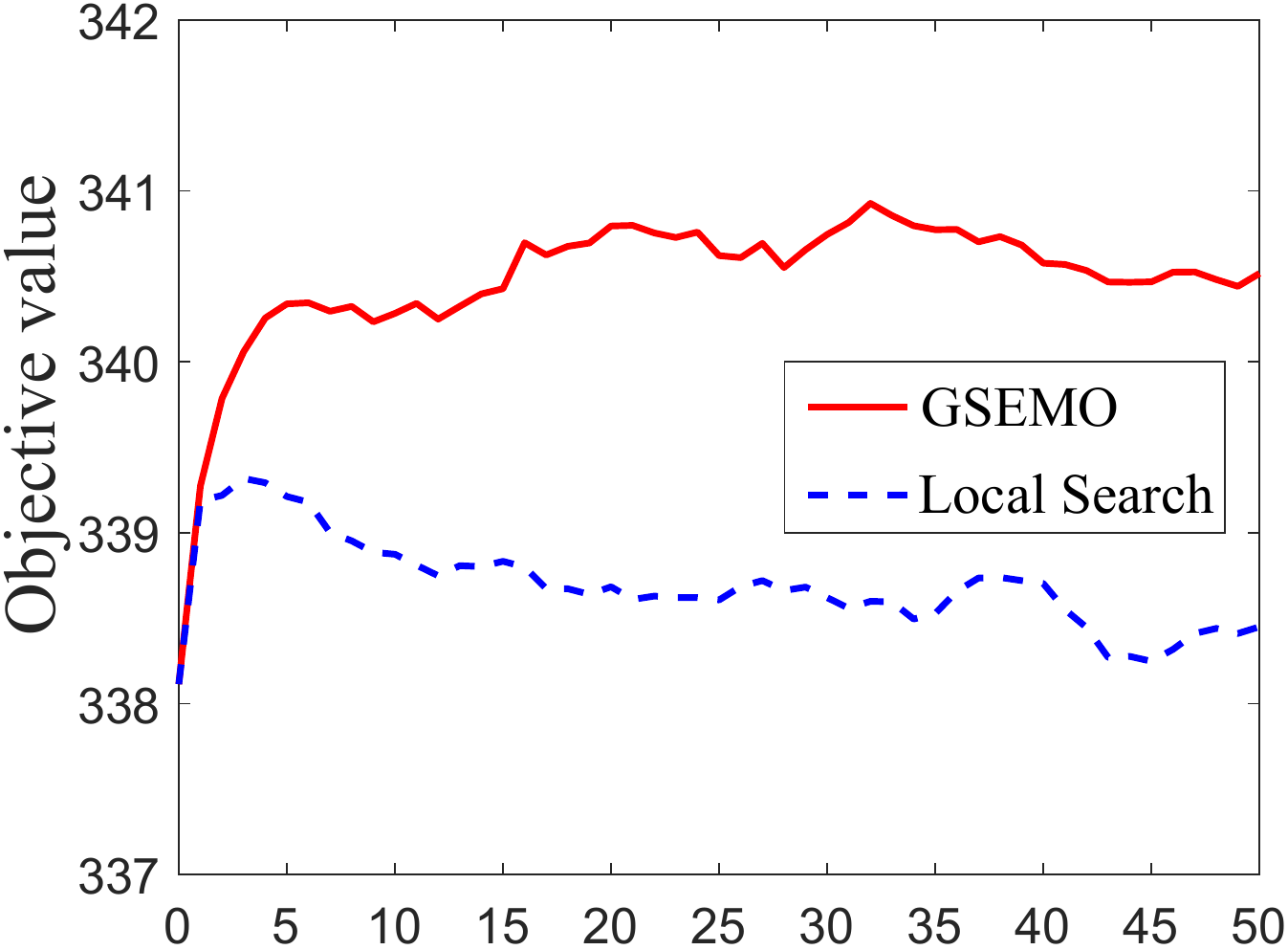}
\end{minipage}
\begin{minipage}[c]{0.32\linewidth}\centering
        \includegraphics[width=1\linewidth]{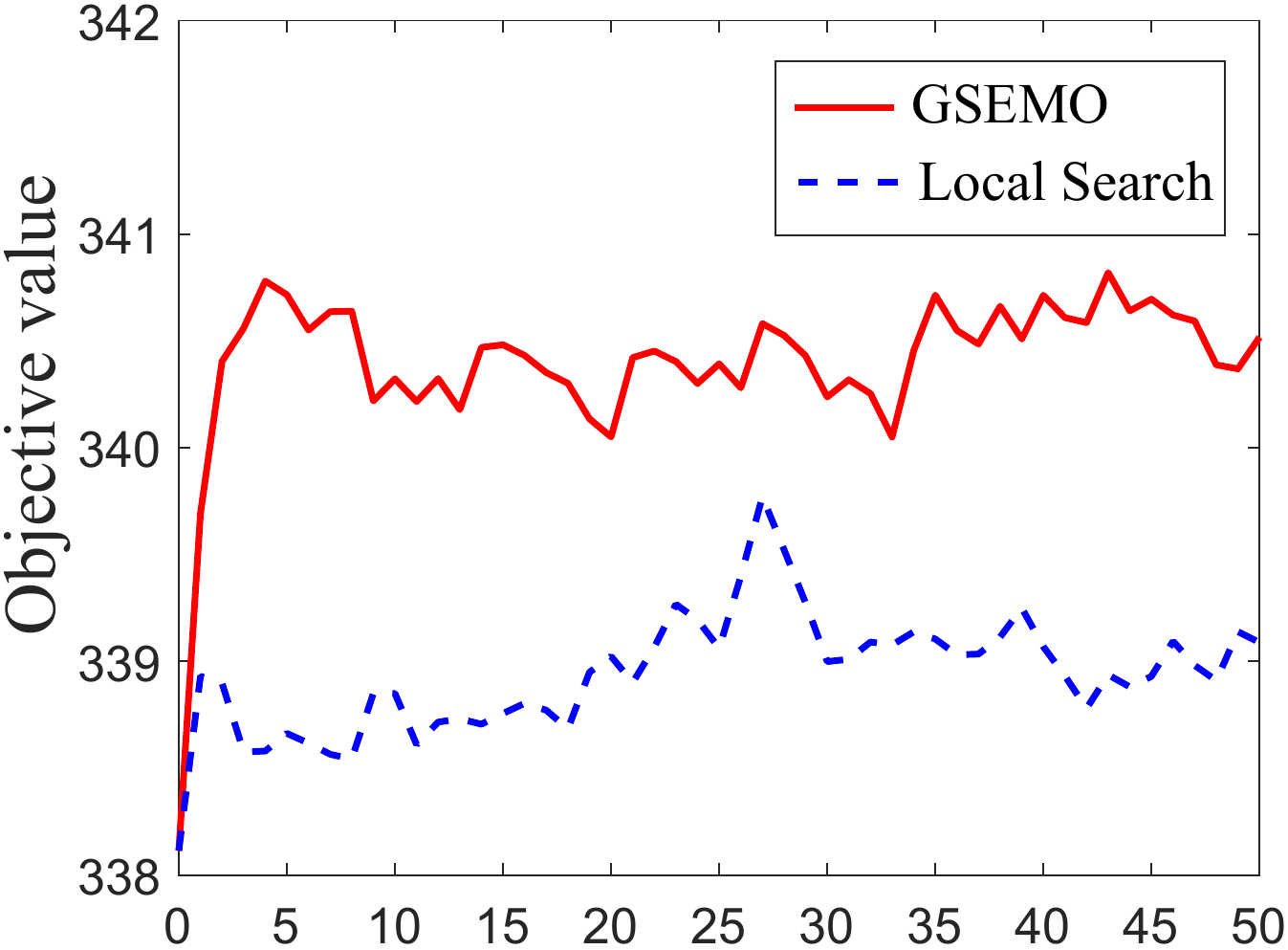}
\end{minipage}\\\vspace{0.5em}
\begin{minipage}[c]{0.32\linewidth}\centering
    \small(a) $m=50$, $t=10kn$
\end{minipage}
\begin{minipage}[c]{0.32\linewidth}\centering
    \small(b) $m=50$, $t=5kn$
\end{minipage}
\begin{minipage}[c]{0.32\linewidth}\centering
    \small(c) $m=300$, $t=10kn$
\end{minipage}
\caption{The average objective values achieved by the GSEMO and local search on the 50 synthetic data sets of web-based search under dynamic environments, where the objective changes 50 times.}\label{fig-dynamic}
\end{figure*}

\section{Conclusion}

This paper applies MOEAs to solve the result diversification problem with wide applications in diverse areas. In particular, the original result diversification problem is transformed as a bi-objective maximization problem, which is then solved by a simple MOEA, i.e., the GSEMO. We prove that the GSEMO can achieve the optimal polynomial-time approximation ratio of $1/2$ for the problem with a cardinality constraint. While for the more general matroid constraint, it can achieve an asymptotically optimal polynomial-time approximation ratio of $1/2-\epsilon/(4n)$, and can maintain this ratio in polynomial running time when the objective changes dynamically. Note that this is the first time that a $(1/2-\epsilon/(4n))$-approximation ratio has been proved under dynamic environments. Experimental results on the applications of web-based search, multi-label feature selection and document summarization clearly show the superiority of the GSEMO over the previous best algorithms, i.e., the greedy algorithm and local search, under both static and dynamic environments.

The diversity of a subset of items has been the sum-diversity, measured by the sum of distances between items. Though we have proved that the GSEMO can still achieve good approximation ratios for the min-diversity (measured by the minimum distance) and mst-diversity (measured by the weight of the minimum spanning tree), it would be still interesting to examine the approximation performance of the GSEMO under other diversity measures~\cite{chandra2001approximation}. Our analyses require that the distance function is a metric or satisfies the parameterized triangle inequality. It is expected to study the performance of the GSEMO under other types of distances, e.g., distances of negative type~\cite{cevallos2019improved}.

For the result diversification problem with a cardinality constraint and the sum-diversity, the GSEMO has achieved the optimal polynomial-time approximation ratio of $1/2$ by maximizing a carefully designed objective function $(1+|\bm{x}|/k)f(\bm{x})/2+\lambda \cdot div(\bm{x})$ and minimizing the subset size $|\bm{x}|$ simultaneously. We have also shown that if maximizing the original objective function $f(\bm{x})+\lambda \cdot div(\bm{x})$, the GSEMO achieves an approximation ratio of $1/2-\epsilon/(4n)$. Thus, an interesting future work is to examine whether an exact approximation ratio of $1/2$ can be achieved using the original objective function. Note that the greedy algorithm can also achieve an approximation ratio of $1/2$ by maximizing $f(\bm{x})/2+\lambda \cdot div(\bm{x})$, but its approximation performance by maximizing the original objective function $f(\bm{x})+\lambda \cdot div(\bm{x})$ has not yet been known, which is worth to be studied. If we can construct an example where the approximation ratio of the greedy algorithm using the original objective function is worse than $1/2-\epsilon/(4n)$, it implies that without using the knowledge of the problem, the GSEMO is better than the greedy algorithm, reflecting the general-purpose property of EAs to some extent.

Under dynamic environments, local search fails to maintain an asymptotically optimal polynomial-time approximation ratio of $1/2-\epsilon/(4n)$, while the GSEMO can. This is just a first step towards solving dynamic result diversification problems. In the future, an important work is to develop faster dynamic algorithms, which, however, may require some limitations on the dynamic change.

The current MOEA employed to solve the transformed bi-objective optimization problem is the GSEMO, which, though very simple, has already shown good approximation performance. In fact, it can be replaced by more complicated MOEAs (e.g., NSGA-II~\cite{deb2002fast}), which may further bring performance improvement, but the theoretical analysis will be a challenging task. The original objective function of result diversification is a linear combination of the quality function and the diversity, which is not separated by the bi-objective transformation. Thus, it would be interesting to treat the quality and diversity as two objectives, and optimize them simultaneously~\cite{ahmed2016discovering}. In the experiments, it has been shown that the GSEMO is much more efficient in practice than in theory, and even can be both faster and better than local search and the greedy algorithm sometimes. But to tackle huge data sets, it may be needed to develop distributed version of the GSEMO.

\section{Acknowledgments}

The authors want to thank the associate editor and anonymous reviewers for their helpful comments and suggestions. This work was supported by the National Science Foundation of China (62022039, 61921006), the project of HUAWEI-LAMDA Joint Laboratory of Artificial Intelligence, and the Collaborative Innovation Center of Novel Software Technology and Industrialization.

\bibliography{aij-diversification-moea}
\bibliographystyle{abbrvnat}

\end{document}